\documentclass{article}

\usepackage{times}
\usepackage{graphicx}
\usepackage{sidecap}
\usepackage[small]{caption}
\usepackage{subcaption}
\usepackage{amsmath}
\allowdisplaybreaks
\usepackage{amsthm}
\newtheorem{assumption}{Assumption}
\usepackage{thmtools}
\usepackage{thm-restate}

\usepackage{multicol}

\usepackage{setspace}
\AtBeginDocument{%
  \addtolength\abovedisplayskip{-0.25\baselineskip}%
  \addtolength\belowdisplayskip{-0.25\baselineskip}%
  \addtolength\abovedisplayshortskip{-0.25\baselineskip}%
  \addtolength\belowdisplayshortskip{-0.25\baselineskip}%
}

\usepackage{pifont}
\usepackage{xspace}

\DeclareMathOperator{\argmax}{argmax}

\usepackage[round]{natbib}

\usepackage{appendix}

\newcommand{\rightcomment}[1]{\(\triangleright\) {\small \it #1}}
\newcommand{\eqcomment}[1]{\addtocounter{equation}{1}\tag*{\rightcomment{#1}\quad(\theequation)}}
\usepackage{suffix}
\WithSuffix\newcommand\eqcomment*[1]{\tag*{\rightcomment{#1}}}

\usepackage[hidelinks]{hyperref}

\usepackage[final]{neurips_2020}

\usepackage{algorithm}
\usepackage[noend]{algpseudocode}

\renewcommand\algorithmicdo{:}
\renewcommand\algorithmicthen{:}
\algnewcommand{\IfThen}[2]{\State \algorithmicif\ #1\ \algorithmicthen\ #2}
\algnewcommand{\IfThenElse}[3]{\State \algorithmicif\ #1\ \algorithmicthen\ #2\ \algorithmicelse\ #3}
\algrenewcommand{\algorithmiccomment}[1]{\hfill \rightcomment{#1}}
\algnewcommand{\LineComment}[1]{\State \rightcomment{#1}}
\algnewcommand{\LinesComment}[1]{\State \rightcomment{\parbox[t]{\linewidth-\leftmargin-\widthof{\(\triangleright\) }}{#1}}\smallskip}
\algnewcommand\algorithmicinput{{\bfseries Input:}}
\algnewcommand\INPUT{\item[\algorithmicinput]}
\algnewcommand\algorithmicoutput{{\bfseries Output:}}
\algnewcommand\OUTPUT{\item[\algorithmicoutput]}
\makeatletter
\newcounter{algorithmicH}
\let\oldalgorithmic\algorithmic
\renewcommand{\algorithmic}{%
  \stepcounter{algorithmicH}
  \oldalgorithmic}
\renewcommand{\theHALG@line}{ALG@line.\thealgorithmicH.\arabic{ALG@line}}
\makeatother
\makeatletter
\newcommand{\algmargin}{\the\ALG@thistlm}
\makeatother
\algnewcommand{\Statepar}[1]{\State\parbox[t]{\dimexpr\linewidth-\algmargin}{\strut #1\strut}}

\newcommand{\pluseq}{\mathrel{+\!\!=}}

\usepackage{xpatch}
\xapptocmd\normalsize{
 \abovedisplayskip=11pt plus 3pt minus 9pt
 \abovedisplayshortskip=0pt plus 3pt
 \belowdisplayskip=11pt plus 3pt minus 9pt
 \belowdisplayshortskip=6.5pt plus 3.5pt minus 3pt
}{}{}

\usepackage{amsmath}
\usepackage{amssymb}
\usepackage{mathrsfs}
\usepackage{mathtools, cuted}
\usepackage{tabularx, booktabs}
\newcolumntype{C}{>{\centering\arraybackslash}X}
\newcolumntype{R}{>{\raggedleft\arraybackslash}X}
\newcolumntype{S}{>{\raggedleft\arraybackslash\hsize=.5\hsize}X}
\usepackage{latexsym}
\usepackage{url}
\usepackage{xspace}
\usepackage{bm,array}
\usepackage{amsfonts}
\usepackage{enumitem}
\usepackage{cases}
\usepackage{mathtools}
\usepackage{empheq}
\usepackage{bm}
\usepackage{esvect}
\usepackage[noabbrev,capitalize]{cleveref}
\crefname{equation}{equation}{equations}
\crefname{section}{section}{sections}
\crefname{footnote}{footnote}{footnotes}   
\crefname{line}{line}{lines}   
\crefname{assumption}{assumption}{assumptions}

\usepackage{bbm}

\usepackage{endnotes}
\let\realfootnote=\footnote  
\let\footnote=\endnote    %
\crefname{footnote}{endnote}{endnote} %
\renewcommand{\thefootnote}{\fnsymbol{footnote}} %

\newcommand{\nofootnotes}{\realfootnote{This paper uses endnotes instead of footnotes.  They are found at the start of the supplementary material.\looseness=-1}}

\let\frac=\tfrac  %

\usepackage[compact]{titlesec}  %

\newcommand{\defn}[1]{\textbf{#1}}
\newcommand{\defeq}{\mathrel{\stackrel{\textnormal{\tiny def}}{=}}}
\newcommand{\xpct}{\mathbb{E}}
\newcommand{\E}[2][]{\xpct_{{#1}}\left[#2\right]}
\newcommand{\var}{\mathbb{V}}
\newcommand{\Real}{\mathbb{R}}

\newcommand{\Uniform}{\mathrm{Unif}}
\newcommand{\Exp}{\mathrm{Exp}}
\newcommand{\Normal}{\mathrm{Normal}}

\renewcommand{\th}{\textsuperscript{th}\xspace}

\newcommand{\inv}[1]{#1^{\scriptscriptstyle-\!1}}

\newcommand{\vecb}[1]{{\boldsymbol{\mathbf{#1}}}}

\newcommand{\param}{\theta}
\newcommand{\Param}{\Theta}
\newcommand{\paramq}{\phi}
\newcommand{\history}{\mathcal{H}}

\newcommand{\es}[2]{#1_{#2}} %
\newcommand{\esh}[2]{\es{#1}{[0,#2)}} %
\newcommand{\eshclosed}[2]{\es{#1}{[0,#2]}} %
\newcommand{\esm}[3]{\es{#1}{#2}^{#3}} %
\newcommand{\esmh}[3]{\esm{#1}{[0,#2)}{#3}} %

\newcommand{\model}{p}
\newcommand{\data}{p^*}
\newcommand{\noise}{q}

\newcommand{\optparens}[1]{\if\relax\detokenize{#1}\relax\else(#1)\fi}   %
\newcommand{\inten}[2]{\lambda_{{#1}}\optparens{#2}} %
\newcommand{\intenbound}{\overline{\lambda}} %
\newcommand{\intend}[2]{\lambda^*_{{#1}}\optparens{#2}} %
\newcommand{\intenm}[3]{\lambda^{#3}_{{#1}}\optparens{#2}} %
\newcommand{\intenq}[2]{\intenm{#1}{#2}{\mathrm{q}}} %
\newcommand{\intenstar}[2]{\intenm{#1}{#2}{*}} %

\newcommand{\tree}[2]{%
	\ifthenelse{\equal{#1}{a}}{$\sqrt{#2}$}{%
		\ifthenelse{\equal{#1}{b}}{"Hi."}{}}}

\newcommand{\set}[1]{{\mathcal{\uppercase{#1}}}}
\newcommand{\xunderbrace}[2]{{\underbrace{#1}_{#2}}}
\newcommand{\objind}{L}
\newcommand{\obj}{J}
\newcommand{\mle}{\obj_{\mathrm{LL}}}
\newcommand{\nce}{\obj_{\mathrm{NC}}}

\newcommand{\draw}[2]{#1_{#2}}

\newcommand{\nothing}{\varnothing}
\newcommand{\dt}{dt}
\newcommand{\ds}{ds}

\newcommand{\tbeg}{t_{\text{beg}}}

\newcommand{\tend}{t_{\text{end}}}
\newcommand{\probsym}{\mathbb{P}}
\newcommand{\prob}[1]{\probsym\left[#1\right]}

\usepackage[usenames,dvipsnames,svgnames,table]{xcolor}
\usepackage[disable]{todonotes}

\newcommand{\cutforspace}[1]{}

\usepackage{tikz}
\usetikzlibrary{shapes.geometric}

\newcommand{\redsolid}{\begin{tikzpicture} \draw [Red, ultra thick] (0,1) -- (0.5,1); \end{tikzpicture}\xspace}
\newcommand{\reddash}{\begin{tikzpicture} \draw [Red, dashed, ultra thick] (0,1) -- (0.5,1); \end{tikzpicture}\xspace}
\newcommand{\bluesolid}{\begin{tikzpicture} \draw [Blue, ultra thick] (0,1) -- (0.5,1); \end{tikzpicture}\xspace}

\usetikzlibrary{arrows,decorations.markings}
\usetikzlibrary{arrows}

\usepackage{verbatim}

\title{Noise-Contrastive Estimation for Multivariate Point Processes}

\author{
	Hongyuan Mei \ \ \ \ \ \  Tom Wan \ \ \ \ \ \ Jason Eisner \\
	Department of Computer Science, Johns Hopkins University \\
	3400 N. Charles Street, Baltimore, MD 21218 U.S.A \\
	\texttt{\{hmei,tom,jason\}@cs.jhu.edu} \\
}

\begin{document}
	
\maketitle

\vspace{-10pt}
\begin{abstract}\label{sec:abstract}

\vspace{-3pt}
The log-likelihood of a generative model often involves both positive and negative terms.  For a temporal multivariate point process, the negative term sums over all the possible event types at each time and also integrates  over all the possible times.  As a result, maximum likelihood estimation is expensive.  We show how to instead apply a version of noise-contrastive estimation%
---a general parameter estimation method with a less expensive stochastic objective.  Our specific instantiation of this general idea works out in an interestingly non-trivial way and has provable guarantees for its optimality, consistency and efficiency. On several synthetic and real-world datasets, our method shows benefits: for the model to achieve the same level of log-likelihood on held-out data, our method needs considerably fewer function evaluations and less wall-clock time.

\end{abstract}

\vspace{-8pt}
\section{Introduction}\label{sec:intro}

Maximum likelihood estimation (MLE) is a popular training method for generative models.  However, to obtain the likelihood of a generative model given the observed data, one must compute the probability of each observed sample, which often includes an expensive normalizing constant. %
For example, in a language model, each word is typically drawn from a softmax distribution over a large vocabulary, whose normalizing constant requires a summation over the vocabulary.

This paper aims to alleviate a similar computational cost for
multivariate point processes.  These generative models are natural tools to analyze streams of discrete events in continuous time.
Their likelihood is improved not only by raising the probability of the observed events, but by lowering the probabilities of the events that were observed \emph{not} to occur. 
There are infinitely many times at which no event of any type occurred; to predict these \emph{non}-occurrences, the likelihood must integrate the infinitesimal event probability for each event type over the entire observed time interval.  Therefore, the likelihood is expensive to compute, particularly when there are many possible event types.

As an alternative to MLE, we propose to train the model by learning to discriminate the observed events from events sampled from a noise process. 
Our method is a version of \defn{noise-contrastive estimation} (NCE), which was originally developed for unnormalized (energy-based) distributions and then extended to conditional softmax distributions such as language models.
To our best knowledge, we are the first to extend the method and its theoretical guarantees (for optimality, consistency and efficiency) to the context of multivariate point processes. 
We will also discuss similar efforts in related areas in \cref{sec:related}. 

On several datasets, our method shows compelling results. By evaluating
fewer event intensities, training takes much less wall-clock time while still achieving competitive log-likelihood.\looseness=-1

\section[Preliminaries]{Preliminaries}\label{sec:notation}\label{sec:prilim}

\subsection{Event Streams and Multivariate Point Processes}\label{sec:es}\label{sec:mpp}
Given a fixed time interval $[0, T)$, we may observe an \defn{event stream} $\es{x}{[0,T)}$:  at each continuous time $t$, the observation $\es{x}{t}$ is one of the discrete types $\{\nothing, 1, \ldots, K\}$ where $\nothing$ means \emph{no event}. 
An non-$\nothing$ observation is called an \defn{event}. 
A generative model of an event stream is called a \defn{multivariate point process}.\nofootnotes

We wish to fit an \defn{autoregressive} probability model to observed event streams.  In a discrete-time autoregressive model, events would be generated from left to right, where  $x_t$ is drawn from a distribution that depends on $x_0, \ldots, x_{t-1}$.  The continuous-time version still generates events from left to right,\footnote{A special event $\es{x}{0}$ is sometimes given at time 0 to mark the beginning of the sequence; the model then generates the rest of the sequence conditioned on $\es{x}{0}$.} but at any specific time $t$ we have $p(x_t=\nothing)=1$, with only an infinitesimal probability of any event.  (For a computationally practical sampling method, see \cref{sec:sample_q}.)  The model is a stochastic process defined by functions $\lambda_k$ that determine a finite \defn{intensity} $\inten{k}{t \mid \esh{x}{t}} \geq 0$ for each event type $k\neq\nothing$ at each time $t > 0$.  This intensity depends on the \defn{history} of events $\esh{x}{t}$ that were drawn at times $< t$.  It quantifies the \defn{instantaneous rate} at time $t$ of events of type $k$.  
That is, $\inten{k}{t \mid \esh{x}{t}}$ is the limit as $dt \rightarrow^+ 0$ of $\frac{1}{dt}$ times the expected number of events of type $k$ on the interval $[t,t+dt)$, where the expectation is conditioned on the history.

As the event probabilities are infinitesimal, the times of the events are almost surely distinct.  To ensure that we have a point process, the intensity functions must be chosen such that the total number of events on any bounded interval is almost surely finite.
Models of this form include inhomogeneous Poisson processes \citep{daley-07-poisson},
in which the intensity functions ignore the history, as well as (non-explosive) Hawkes processes \citep{hawkes-71} and their modern neural versions \citep{du-16-recurrent,mei-17-neuralhawkes}.

Most models use intensity functions that are continuous between events.  Our analysis requires only
\begin{assumption}[Continuity]\label{asmp:inten_cont}	
	For any event stream $\es{x}{[0,T)}$ and event type $k \in \{1, \ldots, K\}$, $\inten{k}{t \mid \es{x}{[0,t)}}$ is Riemann integrable, i.e., bounded and continuous almost everywhere w.r.t. time $t$. %
\end{assumption}

\subsection{Maximum Likelihood Estimation: Usefulness and Difficulties}\label{sec:mle}

In practice, we parameterize the intensity functions by $\param$.  We write $\model_{\param}$ for the resulting probability density over event streams.   
When learning $\param$ from data, we make the conventional assumption that the true point process $\data$
actually falls into the chosen model family:
\begin{assumption}[Existence]\label{asmp:exist}	
	There exists at least one parameter vector $\param^*$ such that $\model_{\param^*} = \data$.
\end{assumption}
Then as proved in \cref{app:mle_proof_details}, such a $\param^*$ can be found as an argmax of
\begin{align}\label{eqn:mle}
	\mle(\param)
	\defeq \E[ \es{x}{[0,T)} \sim \data ]{ \log \model_{\param}(\es{x}{[0,T)} ) }
\end{align}
Given \cref{asmp:inten_cont}, the $\param$ values that maximize $\mle(\param)$ are exactly the
set $\Param^*$ of values for which $\model_{\param} = \data$: any $\param$ for which $\model_{\param} \neq \data$ would end up with a strictly smaller $\mle(\param)$ by increasing the cross entropy $-\data\log\model_{\param}$ over some interval $(t, t')$ for a set of histories with non-zero measure. 

If we modify \cref{eqn:mle} to take the expectation under the empirical distribution of event streams $\es{x}{[0,T)}$ in the training dataset, then $\mle(\param)$ is proportional to the log-likelihood of $\param$. 
For any $\es{x}{[0,T)}$ that satisfies the condition in \cref{asmp:inten_cont},  the log-density used in \cref{eqn:mle} can be expressed in terms of $\inten{k}{t \mid \esh{x}{t}}$: 
\begin{align}\label{eqn:loglik}
\log \model_{\param}(\es{x}{[0,T)} )
= \sum_{t : \es{x}{t} \neq \nothing } \log \inten{\es{x}{t}}{t \mid \esh{x}{t} }  - \int_{t=0}^{T} \sum_{k=1}^{K} \inten{k}{t \mid \esh{x}{t} } \dt
\end{align}
Notice that the second term lacks a log.
It is \emph{expensive} to compute in the following cases: 
\begin{itemize}[nosep,leftmargin=*]
	\item The total number of event types $K$ is large, making $\sum_{k=1}^{K}$ slow. %
	\item The integral $\int_{t=0}^{T}$ is slow to estimate well, e.g., via a Monte Carlo estimate $\frac{T}{J} \sum_{j=1}^{J} \sum_{k=1}^{K} \inten{k}{t_j}$ where each $t_j$ is randomly sampled from the uniform distribution over $[0, T)$.
	\item The chosen model architecture makes it hard to parallelize the $\inten{k}{t_j}$ computation over $j$ and $k$.
\end{itemize}

\subsection{Noise-Contrastive Estimation in Discrete Time}\label{sec:nce_discrete}

For autoregressive models of \emph{discrete-time} sequences, a similar computational inefficiency can be tackled by applying the principle of noise-contrastive estimation \citep{gutmann-10-nce}, as follows.
For each history $\es{x}{0:t}\defeq\es{x}{0}\es{x}{1}\ldots\es{x}{t-1}$ in training data, NCE trains the model $\model_{\param}$ to discriminate the actually observed datum $\es{x}{t}$ from some noise samples whose distribution $\noise$ is known.  
The intuition is: optimal performance is obtained \emph{if and only if} $\model_{\param}$ matches the true distribution $\data$.  

More precisely, given a bag $\{\esm{x}{t}{0}, \esm{x}{t}{1}, \ldots, \esm{x}{t}{M}\}$, where exactly one element of the bag was drawn from $\data$ and the rest drawn i.i.d.\@ from $\noise$, consider the log-posterior probability (via Bayes' Theorem\footnote{The product $\data(\esm{x}{t}{m} \mid \es{x}{0:t}) \prod_{m'\neq m} \noise(\esm{x}{t}{m'} \mid \es{x}{0:t})$ is the likelihood of $\esm{x}{t}{m}$ being the one drawn from $\data$. The prior is uniform since any $m$ in the unordered bag was \emph{a priori} equally probable.}) that $\esm{x}{t}{0}$ was the one drawn from $\data$:

\vspace{-2.5\baselineskip}
\begin{align}\label{eqn:nce_ma}
\hspace{2cm}\log \frac{\phantom{\sum_{m=0}^{M}}\data(\esm{x}{t}{0} \mid \es{x}{0:t}) \prod_{m=1}^{M} \noise(\esm{x}{t}{m} \mid \es{x}{0:t}) }{\sum_{m=0}^{M} \data(\esm{x}{t}{m} \mid \es{x}{0:t}) \prod_{m'\neq m} \noise(\esm{x}{t}{m'} \mid \es{x}{0:t}) }
\end{align}
The ``ranking'' variant of NCE \citep{jozefowicz-16-lm} substitutes $\model_\param$ for $\data$ in this expression, and seeks $\param$ (e.g., by stochastic gradient ascent) to maximize the expectation of the resulting quantity when $\esm{x}{t}{0}$ is a random observation in training data,\footnote{\label{fn:expectedseq}In practice, it is more convenient to maximize the expected \emph{sum} over $t$ in a sequence drawn uniformly from the set of sequences in the training dataset.  This scales the objective up by the average sequence length, preserving the property that longer sequences have more weight.}
 $\es{x}{0:t}$ is its history, and  $\esm{x}{t}{1}, \ldots, \esm{x}{t}{M}$ are drawn i.i.d.\@ from $q(\cdot \mid \es{x}{0:t})$.

This objective is really just conditional maximum log-likelihood on a supervised dataset of $(M+1)$-way classification problems. 
Each problem presents an unordered set of $M+1$ samples---one drawn from $\data$ and the others drawn i.i.d.\@ from $\noise$.  The task is to guess \emph{which} sample was drawn from $\data$.
Conditional MLE trains $\param$ to maximize (in expectation) the log-probability that the model assigns to the correct answer.
In the infinite-data limit, it will find $\param$ (if possible) such that these log-probabilities \emph{match} the true ones given by \eqref{eqn:nce_ma}.  For that, it is \emph{sufficient} for $\param$ to be such that $\model_{\param} = \data$.  Given \cref{asmp:exist}, \citet{ma-18-nce} show that $\model_{\param} = \data$ is also \emph{necessary}, i.e., the NCE task is sufficient to find the true parameters. 
Although the NCE objective does not learn to predict the full observed sample $\es{x}{t}$ as MLE does, but only to distinguish it from the $M$ noise samples, their theorem implies that in expectation over all possible sets of $M$ noise samples, it actually retains all the information (provided that $M > 0$ and 
  $\noise$ has support everywhere that $\data$ does).\looseness=-1

This NCE objective is computationally cheaper than MLE when the distribution $\model_\param(\cdot \mid \es{x}{0:t})$ is a softmax distribution over $\{1, \ldots, K\}$ with large $K$.  The reason is that the expensive normalizing constants in the numerator and denominator of \cref{eqn:nce_ma} need not be computed.  They cancel out because all the probabilities are conditioned on the same (actually observed) history.

\section{Applying Noise-Contrastive Estimation in Continuous Time}\label{sec:objective}\label{sec:practical}\label{sec:nce_obj}

The expensive $\int\sum$ term in \cref{eqn:loglik} is rather similar to a normalizing constant,\footnote{Our model does not need any normalization: $\model(\es{x}{t}=\nothing)+\sum_{k=1}^K \model(\es{x}{t}=k) = 1 + \text{(infinitesimal quantities)} = 1$.} as it sums over non-occurring events.
We might try to avoid computing it\footnote{
  While this paper's speedup over the MLE objective \eqref{eqn:loglik} comes from avoiding the integral, an alternative would be to estimate the integral more efficiently.  One might try randomized adaptive quadrature \citep{baran2008optimally} modified for our discontinuous intensity functions and GPU hardware; or importance sampling of $(t,k)$ pairs where the proposal distribution is roughly proportional to $\inten{k}{t}$---much like the noise distribution we will develop for NCE.}
by discretizing the time interval $[0, T)$ into finitely many intervals of width $\Delta$ and applying NCE.  In this case, we would be distinguishing the true sequence of events on an interval $[i\Delta, (i+1)\Delta )$ from corresponding noise sequences on the same interval, given the same (actually observed) history $\esh{x}{i\Delta}$.  Unfortunately, the distribution $\model_\param(\cdot \mid \esh{x}{i\Delta})$ in the objective still involves an $\int\sum$ term where the integral is over $[i\Delta, (i+1)\Delta )$ and the inner sum is over $k$.  The solution is to shrink the intervals to  \emph{infinitesimal width} $dt$.
Then our log-posterior over each of them becomes
\begin{align}\label{eqn:nce_interval}
	\log \dfrac{\phantom{\sum_{m=0}^{M}} {\model}_{\param}(\esm{x}{[t, t+\dt)}{0} \mid \esmh{x}{t}{0}) \prod_{m=1}^{M} \noise(\esm{x}{[t, t+\dt)}{0} \mid \esmh{x}{t}{0})  }{\sum_{m=0}^{M} { \model}_{\param}(\esm{x}{[t, t+\dt)}{m} \mid \esmh{x}{t}{0}) \prod_{m'\neq m} \noise(\esm{x}{[t, t+\dt)}{m'} \mid \esmh{x}{t}{0})  }
\end{align}

We will define the noise distribution $\noise$ in terms of finite intensity  functions $\intenq{k}{}$,  like the ones $\inten{k}{}$ that define $\model_\param$.  As a result, at a \emph{given} time $t$, there is only an infinitesimal probability that \emph{any} of $\{\esm{x}{t}{0}, \esm{x}{t}{1}, \ldots, \esm{x}{t}{M}\}$ is an event.  Nonetheless, at \emph{each} time $t \in [0,T)$, we will consider generating a noise event (for each $m > 0$) conditioned on the actually observed history $\esh{x}{t}$.  Among these uncountably many times $t$, we may have some for which $\esm{x}{t}{0} \neq \nothing$ (the observed events), or where  $\esm{x}{t}{m} \neq \nothing$ for some $1 \leq m \leq M$ (the noise events).

Almost surely, the set of times $t$ with a real or noise event remains finite.  Our NCE objective is the expected sum of \cref{eqn:nce_interval} over all such times $t$ in an event stream, when the stream is drawn uniformly from the set of streams in the training dataset---as in \cref{fn:expectedseq}---and the noise events are then drawn as above.

Our objective ignores all other times $t$, as they provide no information about $\param$.  After all, when  $\esm{x}{t}{0}=\cdots=\esm{x}{t}{M}=\nothing$, the probability that $\esm{x}{t}{0}$ is the one drawn from the true model must be $1/(M+1)$ by symmetry, regardless of $\param$.  At these times, the ratio in \cref{eqn:nce_interval} does reduce to $1/(M+1)$, since all probabilities are 1.\looseness=-1

At the times $t$ that we do consider, how do we compute \cref{eqn:nce_interval}?  Almost surely, exactly one of $\esm{x}{t}{0},\ldots,\esm{x}{t}{M}$ is an event $k$ for some $k \neq \nothing$.  As a result, exactly one factor in each product is infinitesimal ($dt$ times the $\inten{k}{}$ or $\intenq{k}{}$ intensity), and the other factors are 1.  Thus, the $dt$ factors cancel out between numerator and denominator, and \cref{eqn:nce_interval} simplifies to
\begin{align}\label{eqn:nce_inten_term}
\log \frac{ \inten{k}{t \mid \esmh{x}{t}{0}} }{ \inten{k}{t \mid \esmh{x}{t}{0}} + M \intenq{k}{t \mid \esmh{x}{t}{0}} } \text{ if } \esm{x}{t}{0} = k \text{ and } \log \frac{ \intenq{k}{t \mid \esmh{x}{t}{0}} }{ \inten{k}{t \mid \esmh{x}{t}{0}} + M \intenq{k}{t \mid \esmh{x}{t}{0}} } \text{ if } \esm{x}{t}{0} = \nothing
\end{align}
When a gradient-based optimization method adjusts $\param$ to increase \cref{eqn:nce_inten_term}, the intuition is as follows. 
If $\esm{x}{t}{0} = k$, the model intensity $\inten{k}{t}$ is \emph{increased} to explain why an event of type $k$ occurred at this particular time $t$.  If $\esm{x}{t}{0} = \nothing$, the model intensity $\inten{k}{t}$ is \emph{decreased} to explain why an event of type $k$ did \emph{not} actually occur at time $t$ (it was merely a noise event $\esm{x}{t}{m} = k$, for some $m \neq 0$).  These cases achieve the same qualitative effects as following the gradients of the first and second terms, respectively, in the log-likelihood \eqref{eqn:loglik}.

Our full objective is an expectation of the sum of finitely many such log-ratios:\footnote{We remark that $\nce(\param)$ is the expected log-\emph{probability} of a discrete choice, whereas $\mle(\param)$ was the expected log-\emph{density} of an observation that includes continuous times.  A density must be integrated to yield a probability.}
\begin{align}\label{eqn:nce_inten}%
\nce(\param) \defeq
\E[\esm{x}{[0,T)}{0}\sim \data, \esm{x}{[0,T)}{1:M} \sim \noise]{\sum_{t : \esm{x}{t}{0} \neq\nothing } \log \frac{ \inten{\esm{x}{t}{0}}{t \mid \esmh{x}{t}{0}} }{ \underline{\lambda}_{\esm{x}{t}{0}}(t \mid \esmh{x}{t}{0}) } + \sum_{m=1}^M \sum_{t : \esm{x}{t}{m}\neq\nothing } \log \frac{ \intenq{\esm{x}{t}{m}}{t \mid \esmh{x}{t}{0}} }{ \underline{\lambda}_{\esm{x}{t}{m}}(t \mid \esmh{x}{t}{0}) } }
\end{align}
where $\underline{\lambda}_{k}(t \mid \esmh{x}{t}{0}) \defeq \inten{k}{t \mid \esmh{x}{t}{0}} + M  \intenq{k}{t \mid \esmh{x}{t}{0}}$. %
The expectation is estimated by sampling: we draw an observed stream $\esm{x}{[0,T)}{0}$ from the training dataset, then draw noise events $\esm{x}{[0,T)}{1:M}$ from $\noise$ conditioned on the prefixes (histories) given by this observed stream, as explained in the next section. %
Given these samples, the bracketed term is easy to compute (and we then use backprop to get its gradient w.r.t.\@ $\param$, which is a stochastic gradient of the objective \eqref{eqn:nce_inten}).  It eliminates the $\int\sum$ of \cref{eqn:loglik} as desired, replacing it with a sum over the noise events.  For each real or noise event, we compute only two intensities---the true and noise intensities of that event type at that time.

\subsection{Efficient Sampling of Noise Events}\label{sec:sample_q}\label{sec:q}

The \defn{thinning algorithm} \citep{lewis-79-sim,liniger-09-hawkes} is a rejection sampling method for drawing an event stream over a given observation interval $[0,T)$ from a continuous-time autoregressive process.  Suppose we have already drawn the first $i-1$ times, namely $t_1,\ldots,t_{i-1}$.  For every future time $t \geq t_{i-1}$, let $\history(t)$ denote the context $\esh{x}{t}$ consisting only of the events at those times, and define $\inten{}{t \mid \history(t)} \defeq \sum_{k=1}^{K} \inten{k}{t \mid \history(t)}$.  If $\inten{}{t \mid \history(t)}$ were constant at $\intenbound$,  we could draw the next event time as $t_i \sim t_{i-1} + \Exp(\intenbound)$.  We would then set $\es{x}{t}=\nothing$ for all of the intermediate times $t \in (t_{i-1},t_i)$, and finally draw the type $\es{x}{t_i}$ of the event at time $t_i$, choosing $k$ with probability $\inten{k}{t_i \mid \history(t)}\,/\,\intenbound$.  But what if $\inten{}{t \mid \history(t)}$ is not constant?  The thinning algorithm still runs the foregoing method, taking $\intenbound$ to be any upper bound: $\intenbound \geq \inten{}{t \mid \history(t)}$ for all $t \geq t_{i-1}$.  In this case, there may be ``leftover'' probability mass not allocated to any $k$.  This mass is allocated to $\nothing$.  A draw of $\es{x}{t_i}=\nothing$ means there was no event at time $t_i$ after all (corresponding to a rejected proposal).  Either way, we now continue on to draw $t_{i+1}$ and $\es{x}{t_{i+1}}$, using a version of $\history(t)$ that has been updated to include the event or non-event $\es{x}{t_i}$.  The update to $\history(t)$ affects $\inten{}{t \mid \history(t)}$ and the choice of $\intenbound$.\looseness=-1

{\bfseries How to sample noise streams.} To draw a stream $\esmh{x}{t}{m}$ of noise events, we run the thinning algorithm, using the noise intensity functions $\intenq{k}{}$.  However, there is a modification: $\history(t)$ is now defined to be $\esmh{x}{t}{0}$---the history from the \emph{observed} event stream, rather than the previously sampled \emph{noise} events---and is updated accordingly.  This is because in \cref{eqn:nce_inten}, at each time $t$, all of $\{\esm{x}{t}{0}, \esm{x}{t}{1}, \ldots, \esm{x}{t}{M}\}$ are conditioned on $\esmh{x}{t}{0}$ (akin to the discrete-time case).\footnote{\label{fn:gan}This is not essential to the NCE approach, since in principle the $M+1$ elements of the bag could all be drawn from different distributions.  However, the homogeneity simplifies \crefrange{eqn:nce_inten_term}{eqn:nce_inten}, and not having to keep track of previous noise samples simplifies bookkeeping.  Furthermore, much as in a GAN, we expect the discrimination task to be most challenging and informative when the noise intensity $\intenq{k}{}$ at time $t$ is close to the true intensity $\intenstar{k}{t \mid \esmh{x}{t}{0}}$.  Therefore we give the function $\intenq{k}{}$ access to the true history $\esmh{x}{t}{0}$, and will train it to predict something like the true intensity.\looseness=-1}
The full pseudocode is given in \cref{alg:nce} in the supplementary material.

{\bfseries Coarse-to-fine sampling of event types.} Although our NCE method has eliminated the need to integrate over $t$, the thinning algorithm above still sums over $k$ in the definition of $\intenq{}{t \mid \history(t)}$.  For large $K$, this sum is expensive if we take the noise distribution on each training minibatch to be, for example, the $\model_{\param}$ with the current value of $\param$.  That is a \emph{statistically} efficient choice of noise distribution, but we can make a more \emph{computationally} efficient choice.  A simple scheme is to first generate each noise event with a coarse-grained type $c \in \{1,\ldots,C\}$, and then stochastically choose a refinement $k \in \{1,\ldots,K\}$:
\begin{align}\label{eqn:q}
\intenq{k}{t \mid \esmh{x}{t}{0}}
&\defeq \sum_{c =1}^{C} \noise( k \mid c )  \intenq{c}{t \mid \esmh{x}{t}{0}}  \text{ for } k = 1, 2, \ldots, K 
\end{align}
This noise model is parameterized by the functions $\intenq{c}{}$ and the probabilities $\noise(k\mid c)$.  The total intensity is now $\intenq{}{t \mid \history(t)} = \sum_{c=1}^{C} \intenq{c}{t}$, so we now need to examine only $C$ intensity functions, not $K$, to choose $\intenbound$ in the thinning algorithm.  If we \emph{partition} the $K$ types into $C$ coarse-grained clusters (e.g., using domain knowledge), then evaluating the noise probability \eqref{eqn:q} within the training objective \eqref{eqn:nce_inten} is also fast because there is only one non-zero summand $c$ in \cref{eqn:q}.  This simple scheme works well in our experiments. 
However, it could be elaborated by replacing $\noise( k \mid c )$ with $\noise( k \mid c, \esmh{x}{t}{0} )$, by partitioning the event vocabulary automatically, by allowing overlapping clusters, or by using multiple levels of refinement: all of these elaborations are used by the fast hierarchical language model of \citet{mnih-09-scalable}.

{\bfseries How to draw $M$ streams.} An efficient way to draw the union of $M$ i.i.d.\@ noise streams is to run the thinning algorithm once, with all intensities multiplied by $M$.   In other words, the expected number of noise events on any interval is multiplied by $M$.  This scheme does not tell us which specific noise stream $m$ generated a particular noise event, but the NCE objective \eqref{eqn:nce_inten} does not need to know that.  The scheme works only because every noise stream $m$ has the same intensities $\intenq{k}{t \mid \esmh{x}{t}{0}}$ (not $\intenq{k}{t \mid \esmh{x}{t}{m}}$) at time $t$: there is no dependence on the previous events from that stream.  Amusingly, NCE can now run even with non-integer $M$.\looseness=-1

{\bfseries Fractional objective.}
One view of the thinning algorithm is that it accepts the proposed time $t_i$ with probability $\mu = \inten{}{t_i}/\intenbound$, and in that case, labels it as $k$ with probability $\inten{k}{t_i}/\inten{}{t_i}$.  To get a greater diversity of noise samples, we can accept the time with probability 1, if we then scale its term in the objective \eqref{eqn:nce_inten} by $\mu$.  This does not change the expectation  \eqref{eqn:nce_inten} but may reduce the sampling variance in estimating it.  Note that increasing the upper bound $\intenbound$ now has an effect similar to increasing $M$: more noise samples.\footnote{This trick does carry computational cost: we need to train (via backpropagation) on proposals that might not have been accepted otherwise.  This cost is perhaps not worth it when $\mu(t)$ is too low: it might be better spent on increasing $M$ or running more training epochs for a fixed $M$.  As a compromise, if $\mu$ is small ($\leq 0.05$ in our current experiments), we revert to the original approach of accepting the time with probability $\mu$ and not scaling it.}

\subsection{Computational Cost Analysis}\label{sec:runtime}

 State-of-the-art intensity models use neural networks whose state summarizes the history and is updated after each event.  So to train on a single event stream $x$ with $I \geq 0$ events, both MLE and NCE must perform $I$ updates to the neural state.  Both MLE and NCE then evaluate the intensities $\inten{k}{t \mid \esh{x}{t}}$ of these $I$ events, and also the intensities of a number of events that did \emph{not} occur, which almost surely fall at other times.\footnote{In between the events, even if the neural state remains constant,  the intensity functions need not be constant.}\looseness=-1

Consider the \emph{number of intensities evaluated}.
For MLE, assume the Monte Carlo integration technique mentioned in \cref{sec:mle}.  MLE computes the intensity $\inten{}{}$ for $I$ observed events and for all $K$ possible events at each of $J$ sampled times.  We take $J = \rho I$ (with randomized rounding to an integer), where $\rho > 0$ is a hyperparameter \citep{mei-17-neuralhawkes}.
Hence, the expected total number of intensity evaluations is $I + \rho I K$. 

For NCE with the coarse-to-fine strategy, let $J$ be the total number of times \emph{proposed} by the thinning algorithm.  Observe that $\E{I} = \int_0^T \intend{}{t \mid \esh{x}{t}} dt$, and $\E{J} = M \cdot \int_0^T \intenbound(t \mid \esh{x}{t}) dt$.  Thus, $\E{J} \approx M\cdot \E{I}$ if (1) $\intenbound$ at any time is a tight upper bound on the noise event rate $\intenq{}{}$ at that time and (2) the average noise event rate well-approximates the average observed event rate (which should become true very early in training).  To label or reject each of the $J$ proposals, NCE evaluates $C$ noise intensities $\intenq{c}{}$; if the proposal is accepted with label $k$ (perhaps fractionally), it must also evaluate its model intensity $\inten{k}{}$.  The noise and model intensities $\intenq{c}{}$ and $\inten{k}{}$ must also be evaluated for the $I$ observed events.  Hence, the total number of intensity evaluations is at most $(C+1)J + 2I$, which $\approx (C+1)MI + 2I$ in expectation.

Dividing by $I$, 
we see that making $(M+1)(C+1) \leq \rho K$ suffices to make NCE's stochastic objective take less work per observed stream than MLE's stochastic objective.  $M=1$ and $C=1$ is a valid choice.  But NCE's objective is less informed for smaller $M$, so its stochastic gradient carries less information about $\param^*$.
In \cref{sec:exp}, we empirically investigate the effect of $M$ and $C$ on NCE and compare to MLE with different $\rho$.\looseness=-1

\subsection{Theoretical Guarantees: Optimality, Consistency and Efficiency}\label{sec:theory}\label{sec:why}\label{sec:sketch}\label{sec:consistency}\label{sec:efficiency}

The following theorem implies that stochastic gradient ascent on NCE converges to a correct $\param$ (if one exists):
\begin{restatable}[Optimality]{theorem}{optimality}\label{thm:optimality}
	Under \cref{asmp:inten_cont,asmp:exist},
	${\param} \in \argmax_{\param} \nce(\param)$ if and only if $\model_{\param} = \data$.
\end{restatable}
This theorem falls out naturally when we rearrange the NCE objective in \cref{eqn:nce_inten} as
\begin{align*}%
\int_{t=0}^{T} \sum_{\esmh{x}{t}{0}} \data(\esmh{x}{t}{0} ) \sum_{k=1}^{K} \underline{\lambda}^*_{k}(t  \mid \esmh{x}{t}{0}) \xunderbrace{\left( \frac{{\lambda}^*_{k}(t\mid \esmh{x}{t}{0})}{\underline{\lambda}^*_{k}(t \mid \esmh{x}{t}{0}) } \log \frac{{\lambda}_{k}(t\mid \esmh{x}{t}{0})}{\underline{\lambda}_{k}(t \mid \esmh{x}{t}{0}) } + M \frac{{\lambda}^q_{k}(t\mid \esmh{x}{t}{0})}{\underline{\lambda}^*_{k}(t \mid \esmh{x}{t}{0}) } \log \frac{{\lambda}^q_{k}(t\mid \esmh{x}{t}{0})}{\underline{\lambda}_{k}(t \mid \esmh{x}{t}{0}) } \right)}{\text{a negative cross entropy}} \dt
\end{align*}
where ${\lambda}^*_{k}$ is the intensity under $\data$ and $\underline{\lambda}^*_{k}$ is defined analogously to $\underline{\lambda}_{k}$: see full derivation in \cref{app:derivation}. 
Obviously, $\model_{\param} = \data$ is \emph{sufficient} to maximize the negative cross-entropy for any $k$ given any history and thus maximize $\nce(\param)$. 
It turns out to be also \emph{necessary} because any $\param$ for which $\model_{\param} \neq \data$ would, given \cref{asmp:inten_cont}, end up decreasing the negative cross-entropy for some $k$ over some interval $(t, t')$ given a set of histories with non-zero measure. 
A full proof can be found in \cref{app:optimal}: as we'll see there, although it resembles Theorem~3.2 of \citet{ma-18-nce}, the proof of our \cref{thm:optimality} requires new analysis to handle continuous time, since \citet{ma-18-nce} only worked on discrete-time sequential data. 

Moreover, our NCE method is strongly consistent for any $M \geq 1$ and approaches \emph{Fisher efficiency} when $M$ is large. 
These properties are the same as in \citet{ma-18-nce} and the proofs are also similar. 
Therefore, we leave the related theorems together with their assumptions and proofs to \cref{app:consistency,app:efficiency}. 

\section{Related Work}\label{sec:related}

The original ``binary classification'' NCE principle was proposed by \citet{gutmann-10-nce} to estimate parameters for joint models of the form $\model_{\param}(x) \propto \exp(\text{score}(x,\param))$.  \citet{gutmann-12-nce} applied it to natural image statistics.
It was then widely applied to natural language processing problems such as language modeling \citep{mnih-12-nce}, learning word representations \citep{mikolov-13-distributed} and machine translation \citep{vaswani-13-mt}. 
The ``ranking-based'' variant \citep{jozefowicz-16-lm}\footnote{\citet{jozefowicz-16-lm} considered it a competitor to NCE; \citet{ma-18-nce} argued for regarding it as a variant.}
is better suited for conditional distributions \citep{ma-18-nce}, including
those used in autoregressive models, and has shown strong
performance in large-scale language modeling with recurrent neural
networks.  

\citet{guo-18-nce} tried NCE on (univariate) point
processes but used the binary classification version.  They used
discrimination problems of the form: ``Is event $k$ at time $t'$ the true
next event following history $\eshclosed{x}{t}$, or was it generated from a noise distribution?''
Their classification-based NCE variant is \emph{not} well-suited to conditional distributions \citep{ma-18-nce}: this complicates their method since they needed to build a parametric model of the local normalizing constant, giving them weaker theoretical guarantees and worse performance (see \cref{sec:exp}).
In contrast, we choose the ranking-based variant: our key idea of how to apply this to continuous time is new (see \cref{sec:nce_obj}) and requires new analysis (see \cref{app:mle_proof_details,app:nce_proof_details}).

\section{Experiments}\label{sec:exp}

We evaluate our NCE method on several synthetic and real-world datasets, with comparison to MLE, \citet{guo-18-nce} (denoted as b-NCE), and least-squares estimation (LSE) \citep{eichler-et-al-2017}.
b-NCE has the same hyper-parameter $M$ as our NCE, namely the number of noise events.
LSE's objective involves an integral over times $[0,T)$, so it has the same hyper-parameter $\rho$ as MLE.

On each of the datasets, we will show the estimated log-likelihood  on the held-out data achieved by the models trained on the NCE, b-NCE, MLE and LSE objectives, as training consumes increasing amounts of computation---measured by the number of intensity evaluations and the elapsed wall-clock time (in seconds).\footnote{Our code is written in PyTorch \citep{paszke-17-pytorch} and will be released upon paper acceptance.  Our experiments were run on NVIDIA Tesla K80.}  We always set the minibatch size $B$ to exhaust the GPU capacity, so smaller $\rho$ or $M$ allows larger $B$.  Larger $B$ in turn increases the number of epochs per unit time (but decreases the possibly beneficial variance in the stochastic gradient updates).

\subsection{Synthetic Datasets}\label{sec:synthetic}

In this section, we work on two synthetic datasets with $K=10000$ event types. 
We choose the \defn{neural Hawkes process (NHP)} \citep{mei-17-neuralhawkes} to be our model $\model_{\param}$.\footnote{We use the public PyTorch implementation.  NHP is a thoughtfully
designed framework that has been demonstrated effective on temporal data, but our method can also be used for other models with parametric intensity functions.}
For the noise distribution $\noise$, we choose $C=1$ and also parametrize its intensity function as a neural Hawkes process. 

The first dataset has sequences drawn from the randomly initialized $\noise$ such that we can check how well our NCE method could perform with the ``ground-truth'' noise distribution $\noise = \data$; the sequences of the second dataset were drawn from a randomly initialized neural Hawkes process to evaluate both methods in the case that the model family $\model_{\param}$ is well-specified. 
We show (the zoomed-in views of the interesting parts of) multiple learning curves on each dataset in \cref{fig:nhp}: NCE is observed to consume substantially fewer intensity evaluations and less wall-clock time than MLE to achieve competitive log-likelihood, while b-NCE and LSE are slower and only converge to lower log-likelihood. 
Note that the wall-clock time may not be proportional to the number of intensities because computing intensities is not all of the work (e.g., there are LSTM states of both $\model_{\param}$ and $\noise$ to compute and store on GPU). 

We also observed that models that achieved comparable log-likelihood---no matter how they were trained---achieved comparable prediction accuracies (measured by root-mean-square-error for time and error rate for type). Therefore, our NCE still beats other methods at converging quickly to the highest prediction accuracy.

\begin{figure*}%
	\begin{center}
		\begin{subfigure}[b]{0.49\linewidth}
			\begin{center}
				\includegraphics[width=0.48\linewidth]{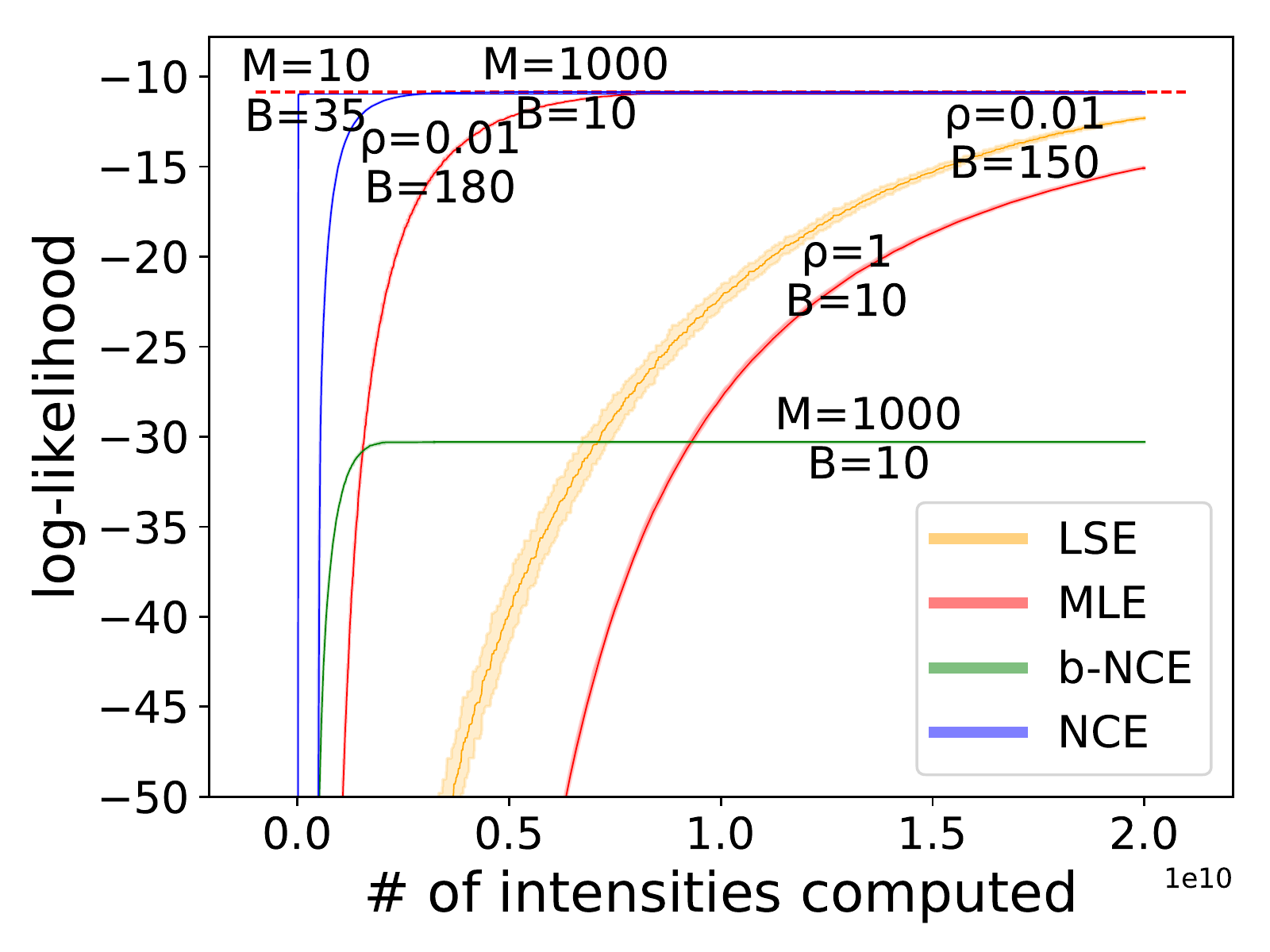}
				\includegraphics[width=0.48\linewidth]{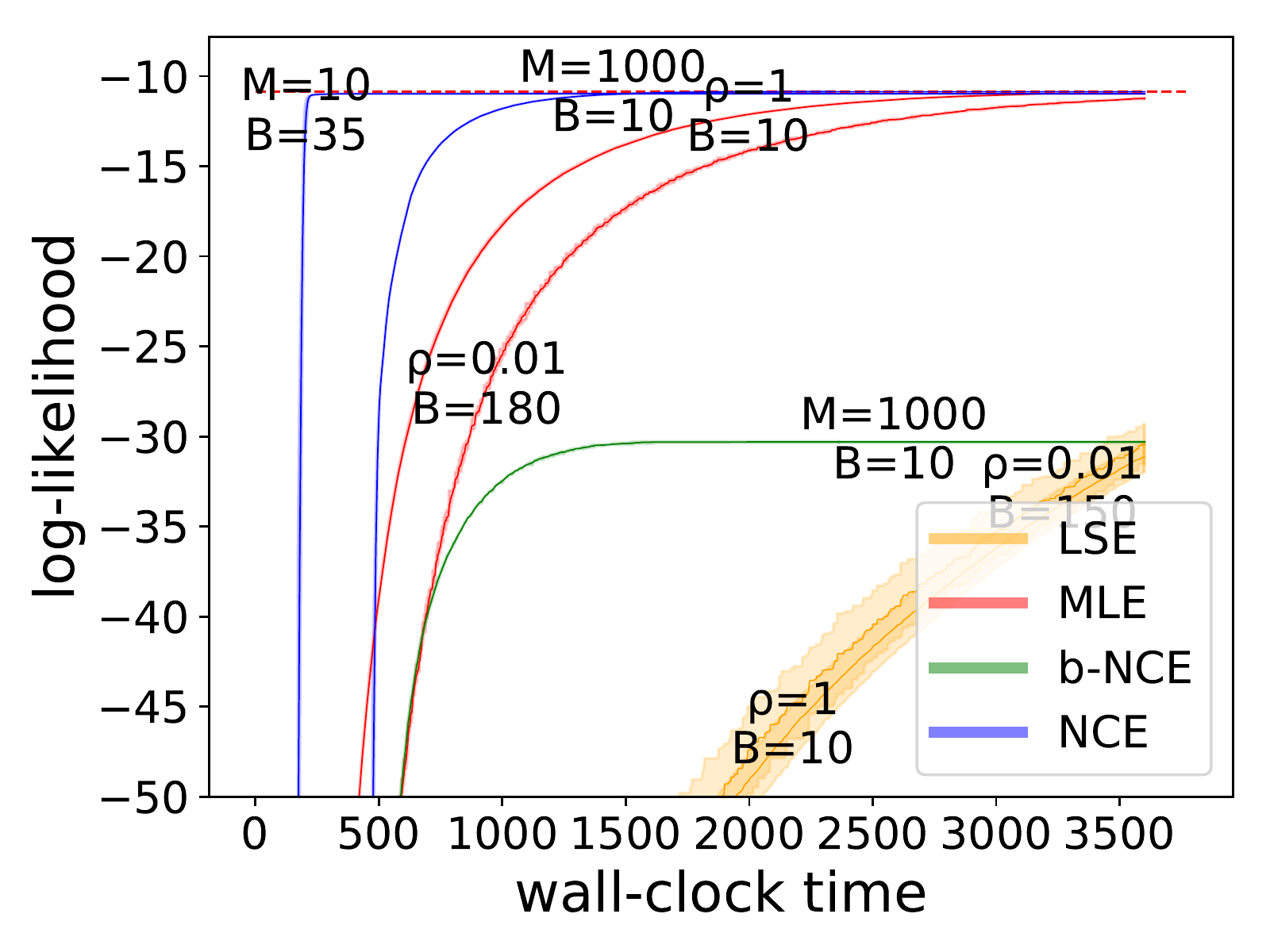}
				\vspace{-2pt}
				\caption{Synthetic-1: $\data = \noise$. }\label{fig:nhp_qeqp}
			\end{center}
		\end{subfigure}
		~
		\begin{subfigure}[b]{0.49\linewidth}
			\begin{center}
				\includegraphics[width=0.48\linewidth]{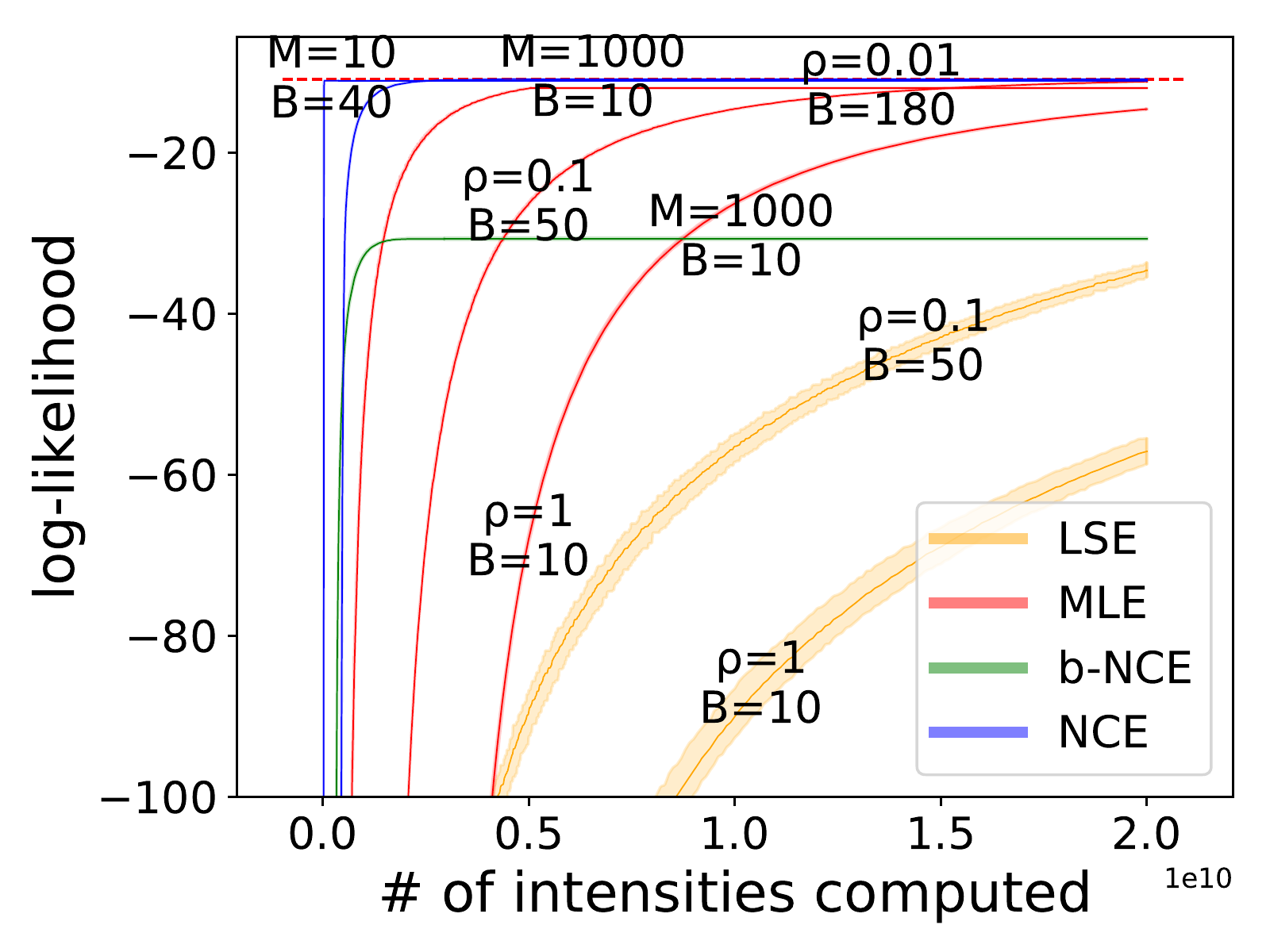}
				\includegraphics[width=0.48\linewidth]{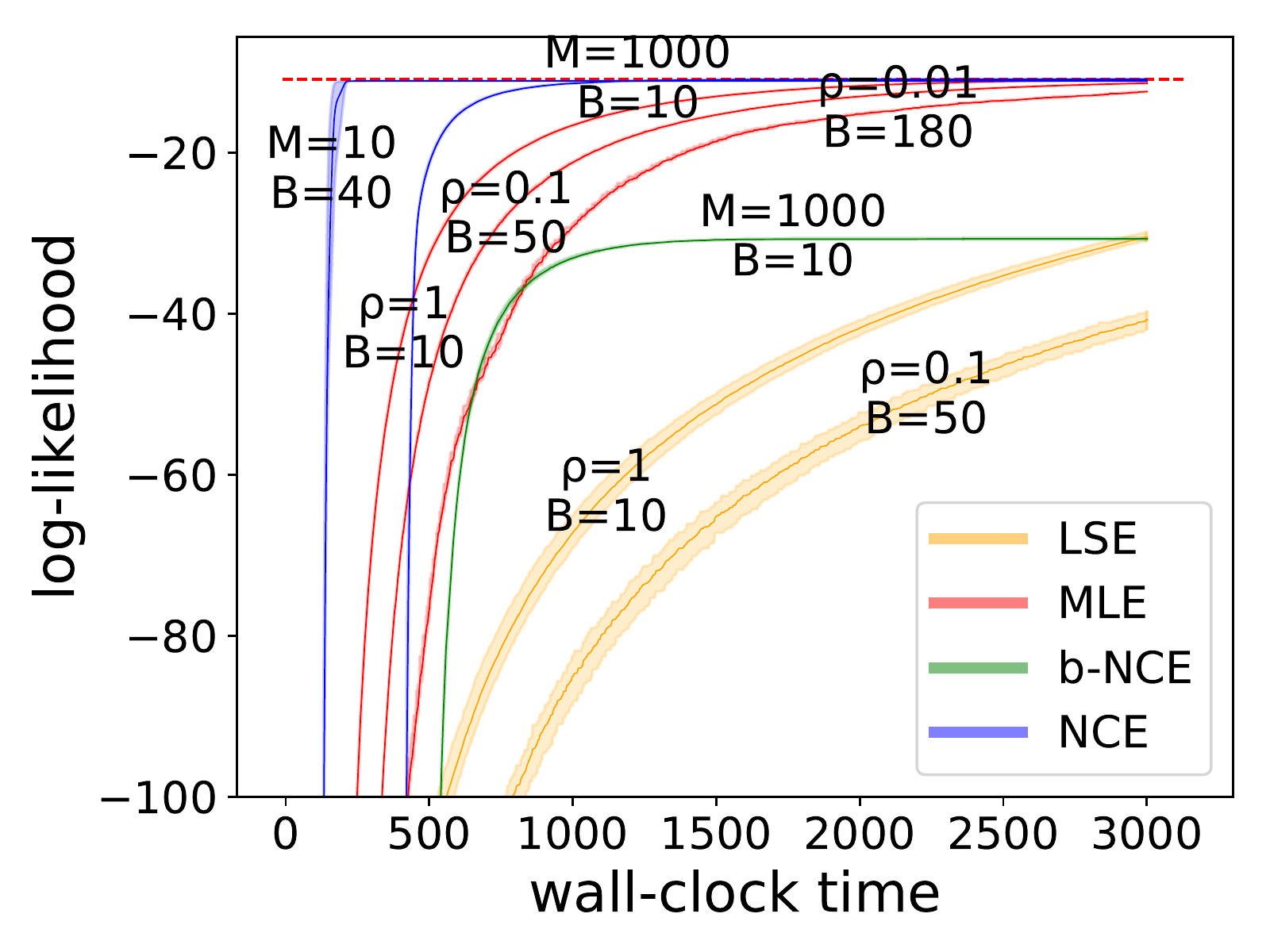}
				\vspace{-2pt}
				\caption{Synthetic-2: $\data$ and $\model_{\param}$ are of the same family.}\label{fig:nhp_qneqp_train}
			\end{center}
		\end{subfigure}
		\vspace{-12pt}
		\caption{Learning curves of MLE and NCE on synthetic datasets. The displayed $\rho$ and $M$ values are among the better ones that we found during hyperparameter search. The horizontal red line marks the highest held-out log-likelihood achieved by MLE. The shaded area of each curve shows the range of log-likelihood of three independent runs; most of them are too narrow to be easily noticed.}
		\label{fig:nhp}
	\end{center}
\end{figure*}

{\bfseries Ablation Study I: Always or Never Redraw Noise Samples.}
During training, for each observed data, we can choose to either redraw a new set of noise samples every time we train on it or keep reusing the old samples: we did the latter for \cref{fig:nhp}. 
In experiments doing the former, we observed better generation  for tiny $M$ (e.g., $M=1$) but substantial slow-down (because of sampling) with no improved generalization for large $M$ (e.g, $1000$). Such results suggest that we always reuse old samples as long as $M$ is reasonably large: it is then what we do for all other experiments throughout the paper. See \cref{app:p0p1} for more details of this ablation study, including learning curves of the ``always redraw'' strategy in \cref{fig:nhp_p0p1}. 

\subsection{Real-World Social Interaction Datasets with Large $K$}\label{sec:real_large}

We also evaluate the methods on several real-world social interaction datasets that have many event types: see \cref{app:data_details} for details (e.g, data statistics, pre-processing, data splits, etc). 
In this section, we show the learning curves on two particularly interesting datasets (explained below) in \cref{fig:social} and leave those on the other datasets (which look similar) to \cref{app:real_large}. 

{\bfseries EuroEmail} \citep{paranjape-17-motif}. This dataset contains time-stamped emails between anonymized members of a European research institute. We work on a subset of $100$ most active members   and then end up with $K=10000$ possible event types and $50000$ training event tokens. 

{\bfseries BitcoinOTC} \citep{kumar-16-edge}. This dataset contains time-stamped rating (positive/negative) records between anonymized users on the BitcoinOTC trading platform. We work on a subset of 100 most active users and then end up with $K=19800$ (self-rating not allowed) possible event types but only $1000$ training event tokens: this is an extremely data-sparse setting. 

On these datasets, our model $\model_{\param}$ is still a neural Hawkes process.  For the noise distribution $\noise$, we experiment with not only the coarse-to-fine neural process with $C=1$ but also a homogeneous Poisson process.  
As shown in \cref{fig:social}, our NCE tends to perform better with the neural $\noise$: this is because a neural model can better fit the data and thus provide better training signals, analogous to how a good generator can benefit the discriminator in the generative adversarial framework \citep{goodfellow-14-gan}. 
NCE with Poisson $\noise$ also shows benefits through the early and middle training stages, but it might suffer larger variance (e.g., \cref{fig:email_train_poisson}) and end up with slightly worse generalization (e.g., \cref{fig:bitcoin_train_poisson}). 
MLE with different $\rho$ values all eventually achieve the highest log-likelihood ($\approx-10$ on EuroEmail and $\approx-15$ on BitcoinOTC), but most of these runs are so slow that their peaks are out of the current views. 
The b-NCE runs with different $M$ values are slower, achieve worse generalization and suffer larger variance than our NCE; interestingly, b-NCE prefers Poisson $\noise$ to neural $\noise$ (better generalization on EuroEmail and smaller variance on BitcoinOTC).
In general, LSE is the slowest, and the highest log-likelihood it can achieve ($\approx-30$ on EuroEmail and $\approx-25$ on BitcoinOTC) is lower than that of MLE and our NCE. 

\begin{figure*}%
	\begin{center}
		\begin{subfigure}[t]{0.49\linewidth}
			\renewcommand\thesubfigure{\alph{subfigure}1}
			\begin{center}
				\includegraphics[width=0.49\linewidth]{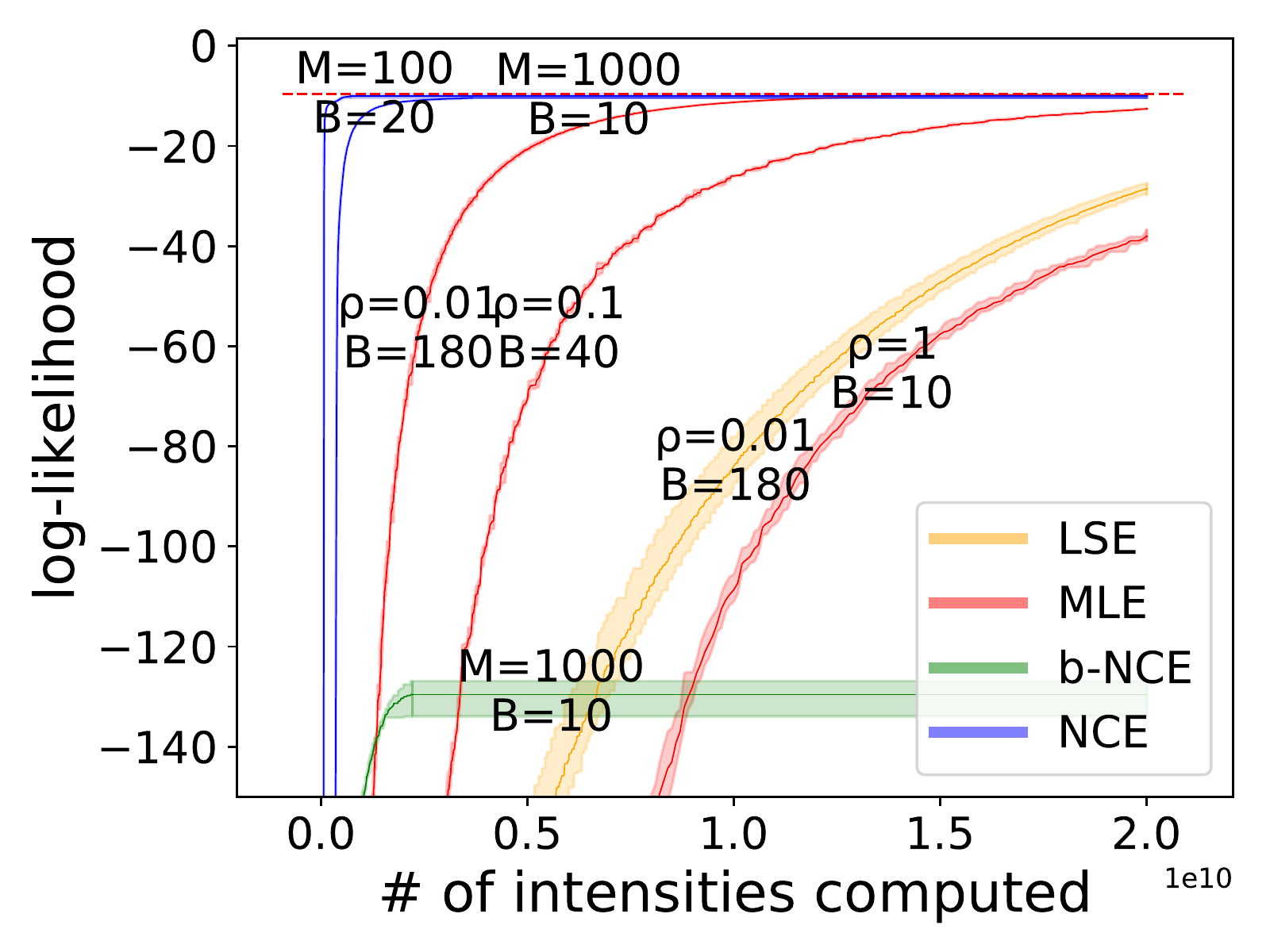}
				\includegraphics[width=0.49\linewidth]{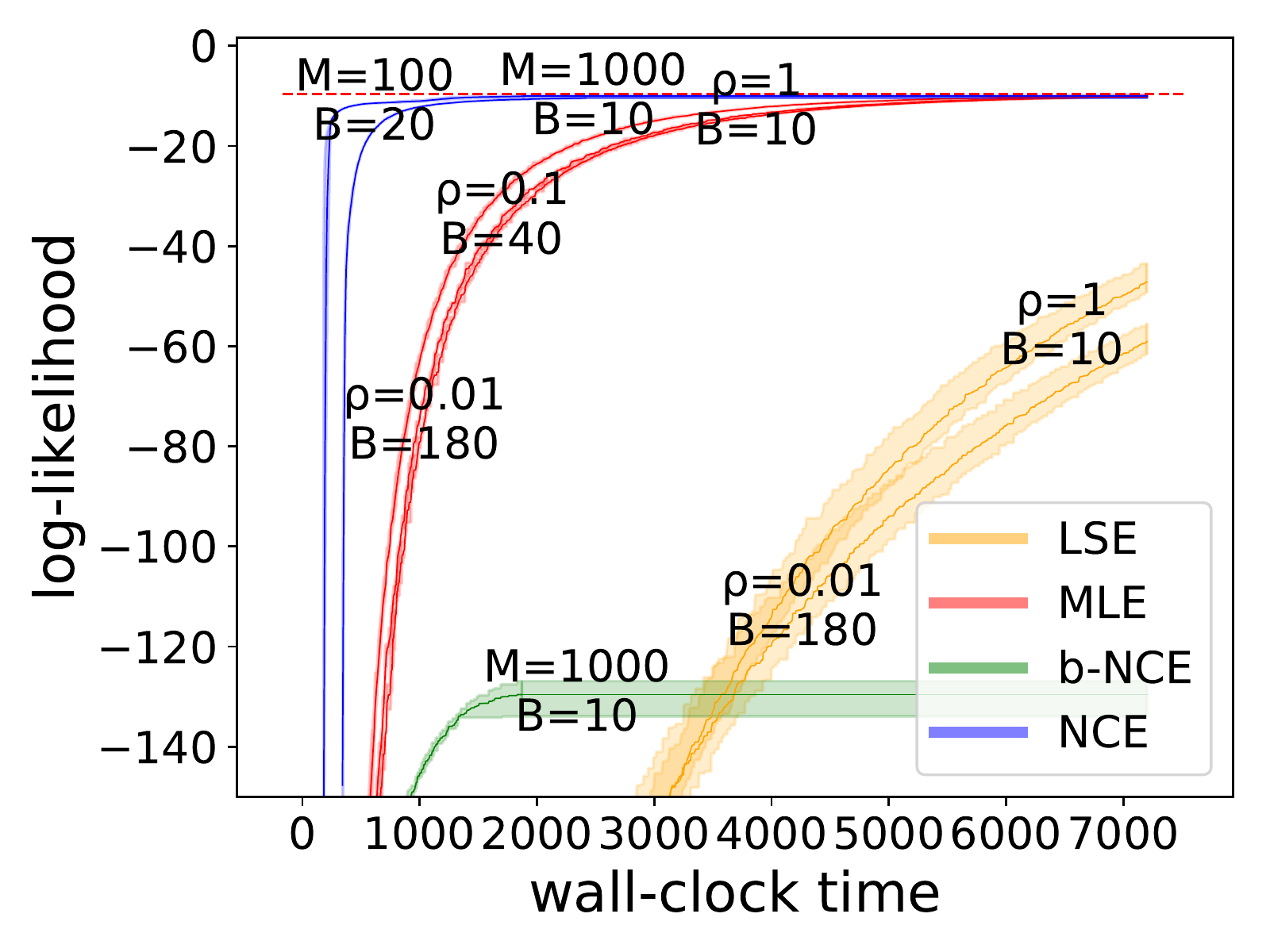}
				\vspace{-16pt}
				\caption{EuroEmail: neural $\noise$}\label{fig:email_train_neural}
			\end{center}
		\end{subfigure}
		~
		\begin{subfigure}[t]{0.49\linewidth}
			\addtocounter{subfigure}{-1}
			\renewcommand\thesubfigure{\alph{subfigure}2}
			\begin{center}
				\includegraphics[width=0.49\linewidth]{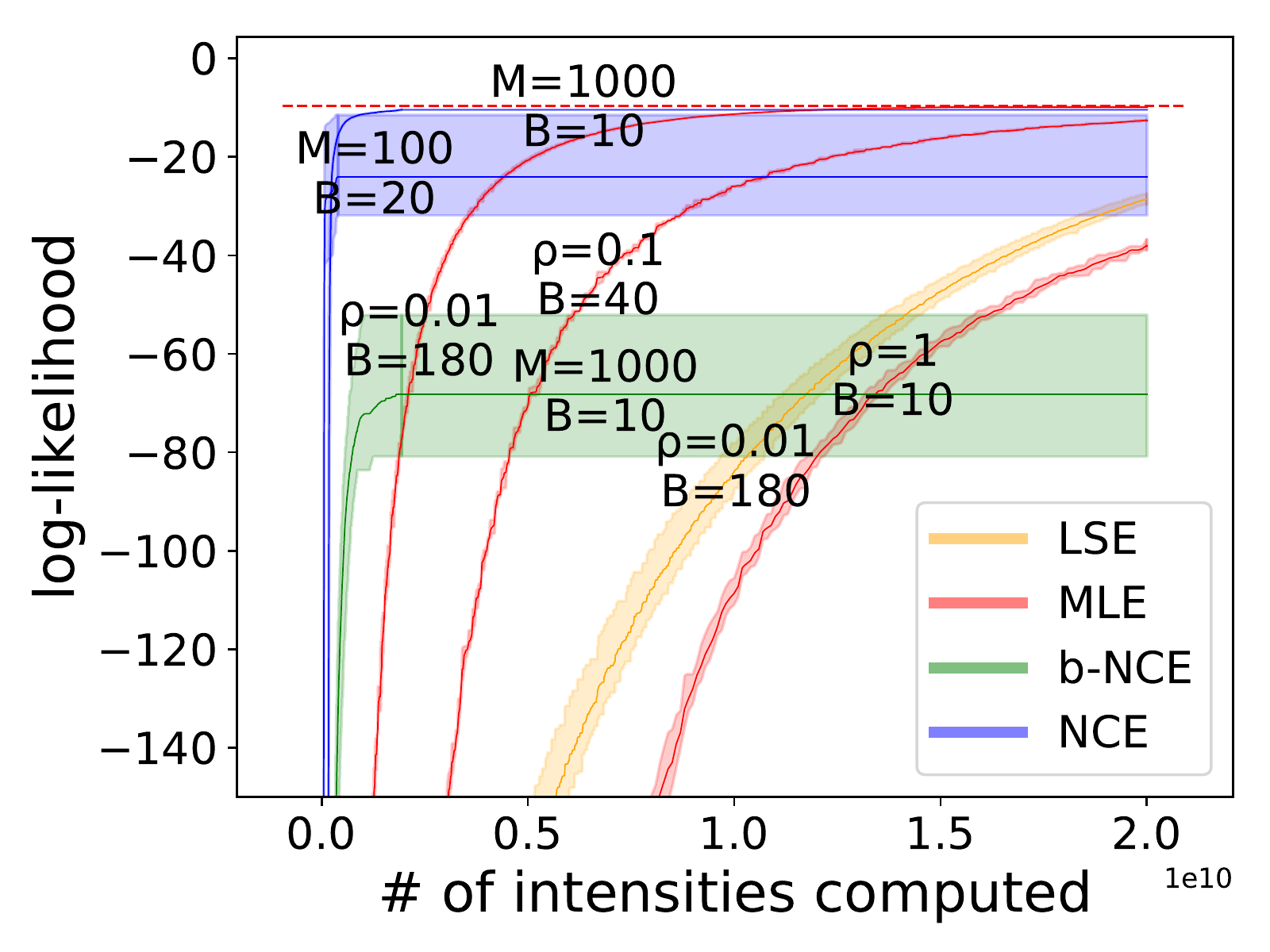}
				\includegraphics[width=0.49\linewidth]{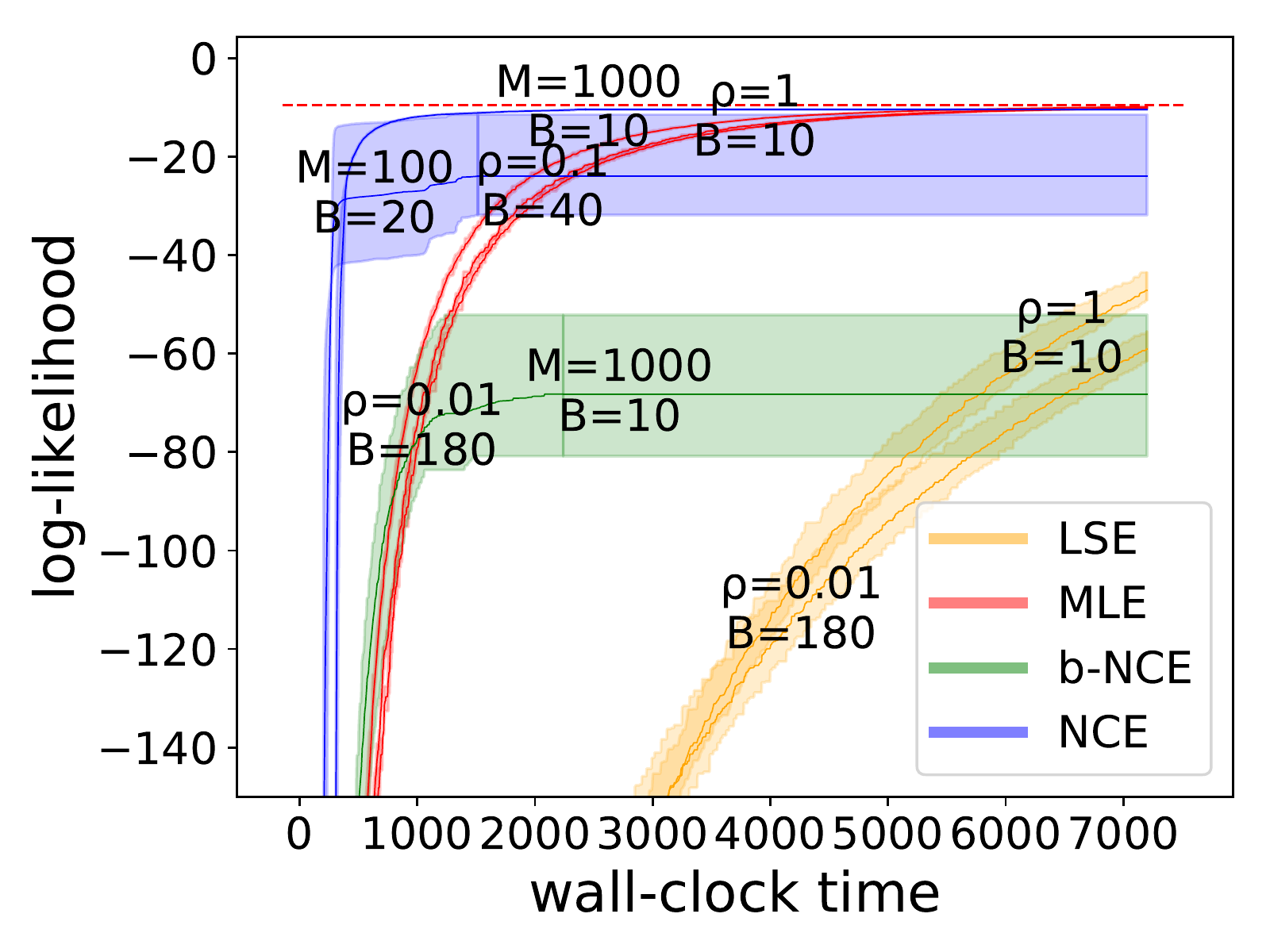}
				\vspace{-16pt}
				\caption{EuroEmail: Poisson $\noise$}\label{fig:email_train_poisson}
			\end{center}
		\end{subfigure}
		
		\begin{subfigure}[b]{0.49\linewidth}
			\renewcommand\thesubfigure{\alph{subfigure}1}
			\begin{center}
				\includegraphics[width=0.49\linewidth]{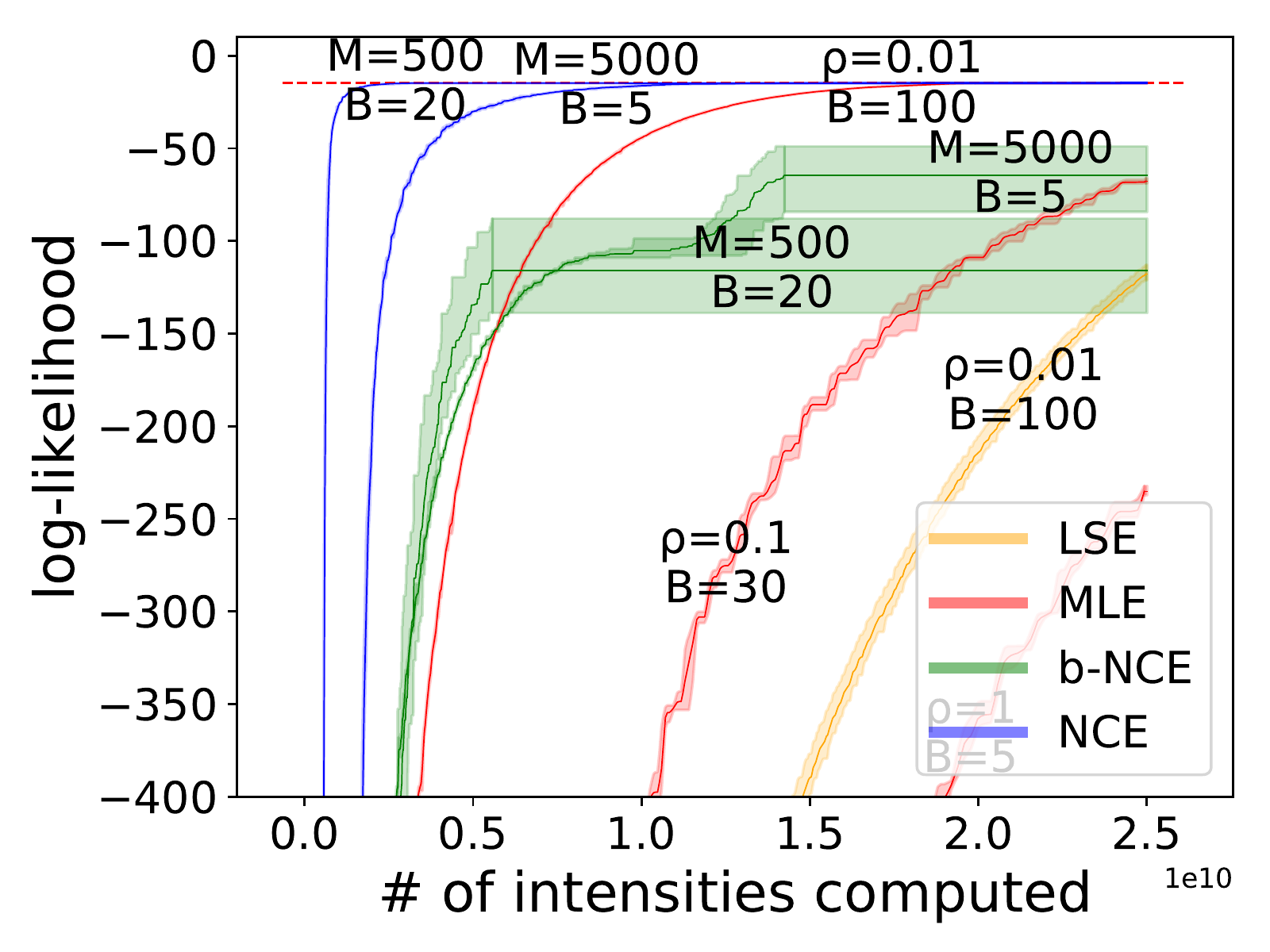}
				\includegraphics[width=0.49\linewidth]{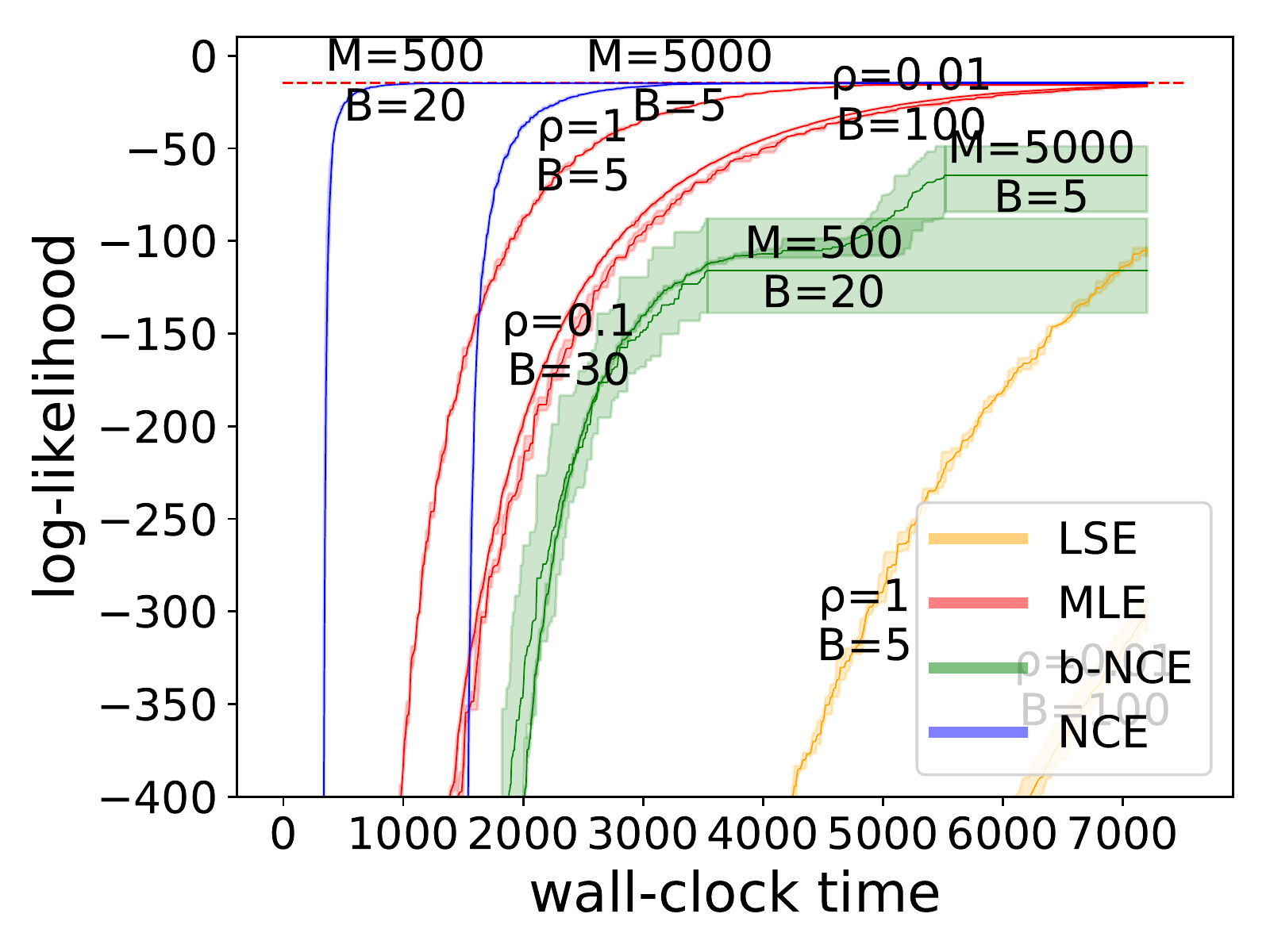}
				\vspace{-16pt}
				\caption{BitcoinOTC: neural $\noise$.}\label{fig:bitcoin_train_neural}
			\end{center}
		\end{subfigure}
		~
		\begin{subfigure}[b]{0.49\linewidth}
			\addtocounter{subfigure}{-1}
			\renewcommand\thesubfigure{\alph{subfigure}2}
			\begin{center}
				\includegraphics[width=0.49\linewidth]{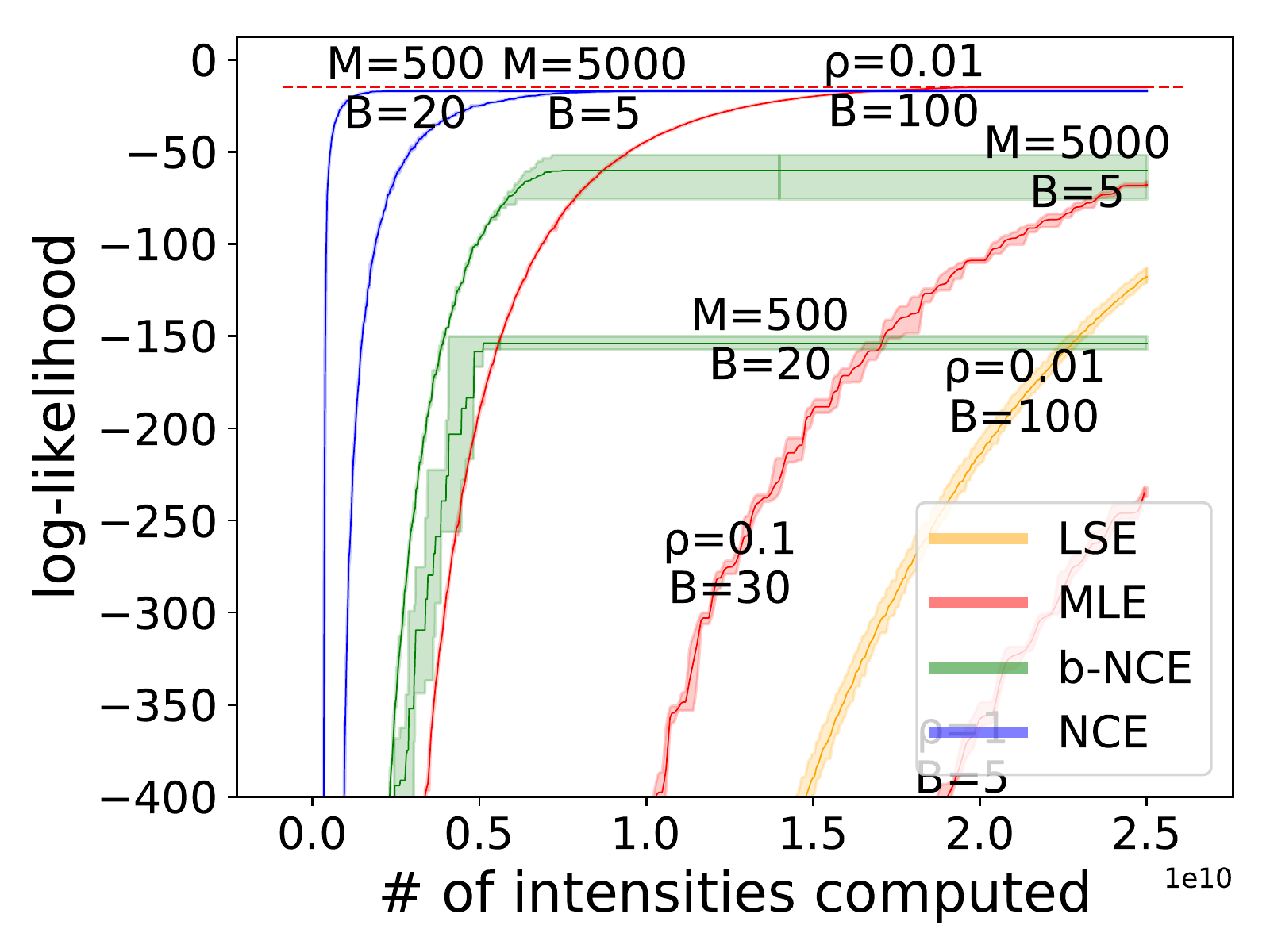}
				\includegraphics[width=0.49\linewidth]{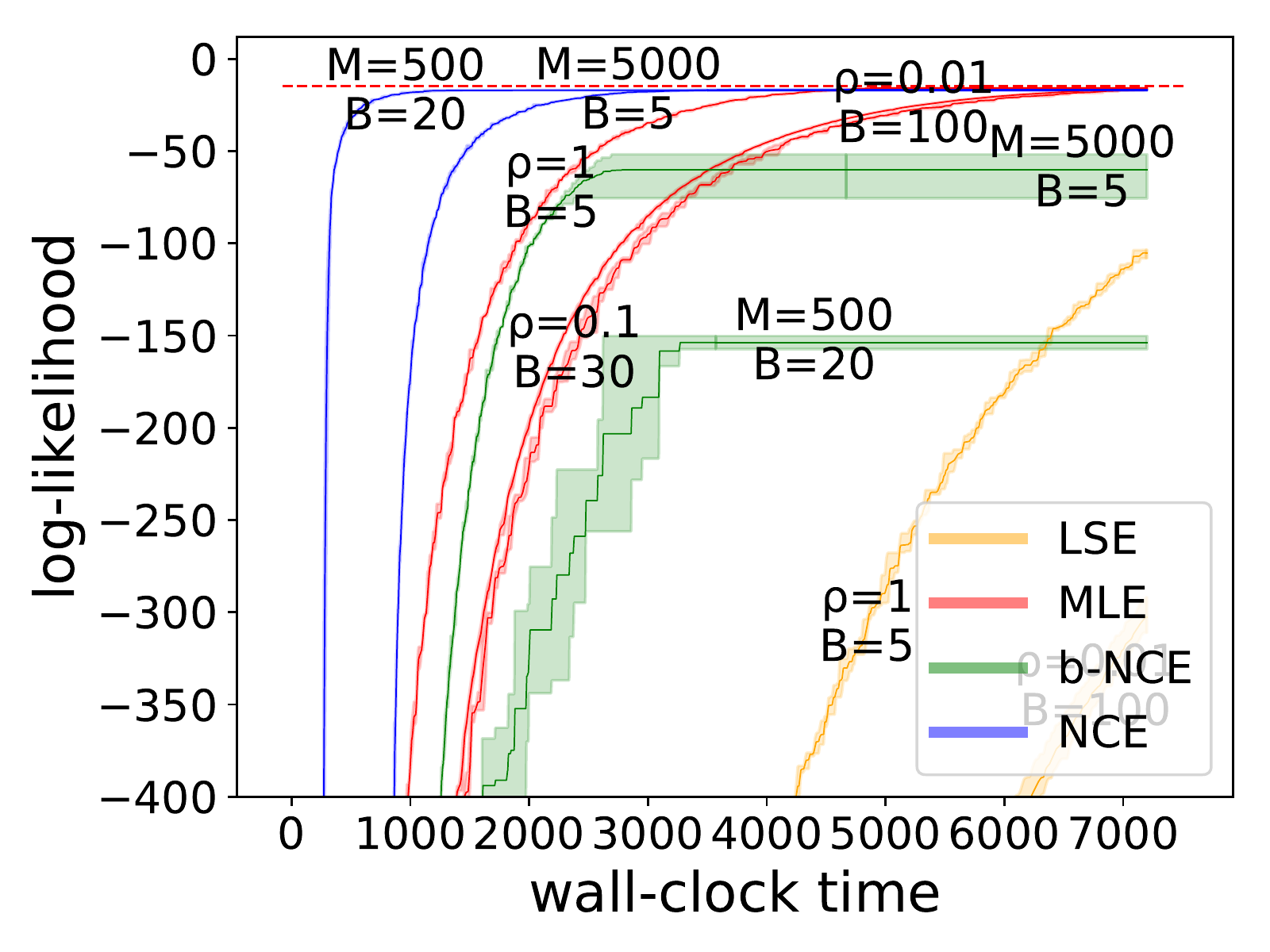}
				\vspace{-16pt}
				\caption{BitcoinOTC: Poisson $\noise$.}\label{fig:bitcoin_train_poisson}
			\end{center}
		\end{subfigure}
		\vspace{-16pt}
		\caption{Learning curves of MLE and NCE on the real-world social interaction datasets.}
		\label{fig:social}
	\end{center}
\end{figure*}

{\bfseries Ablation Study II: Trained vs. Untrained $\noise$.}
The noise distributions (except the ground-truth $\noise$ for Synthetic-1) that we have used so far were all pretrained on the same data as we train $\model_{\param}$. The training cost is cheap: e.g., on the datasets in this section, the actual wall-clock training time for the neural $\noise$ is less than 2\% of what is needed to train $\model_{\param}$, and training the Poisson $\noise$ costs even less.\footnote{We train $\noise$ by MLE: summing $C$ intensities is not expensive when $C$ is small. In \cref{app:train_q}, we document an alternative strategy that uses $\noise$ as the noise distribution to train itself by NCE.}\footnote{For the experiments in \cref{sec:real_dyna}, training the neural $\noise$ takes only $< 1/100$ of what needed to train $\model_{\param}$.}
We also experimented with untrained noise distributions and they were observed to perform worse (e.g., worse generalization, slower convergence and larger variance). 
See \cref{app:random_q} for more details, including learning curves (\cref{fig:social_random}).

\subsection{Real-World Dataset with Dynamic Facts}\label{sec:real_dyna}\label{sec:robocup}\label{sec:iptv}

In this section, we let $\model_{\param}$ be a \defn{neural Datalog through time (NDTT)} model \citep{mei-20-neuraldatalog}.  Such a model can be used in a domain in which new events dynamically update the set of event types and the structure of their intensity functions. We evaluate our method on training the domain-specific models presented by \citet{mei-20-neuraldatalog}, on the same datasets they used:

\begin{figure*}%
	\begin{center}
		\begin{subfigure}[b]{0.49\linewidth}
			\begin{center}
				\includegraphics[width=0.49\linewidth]{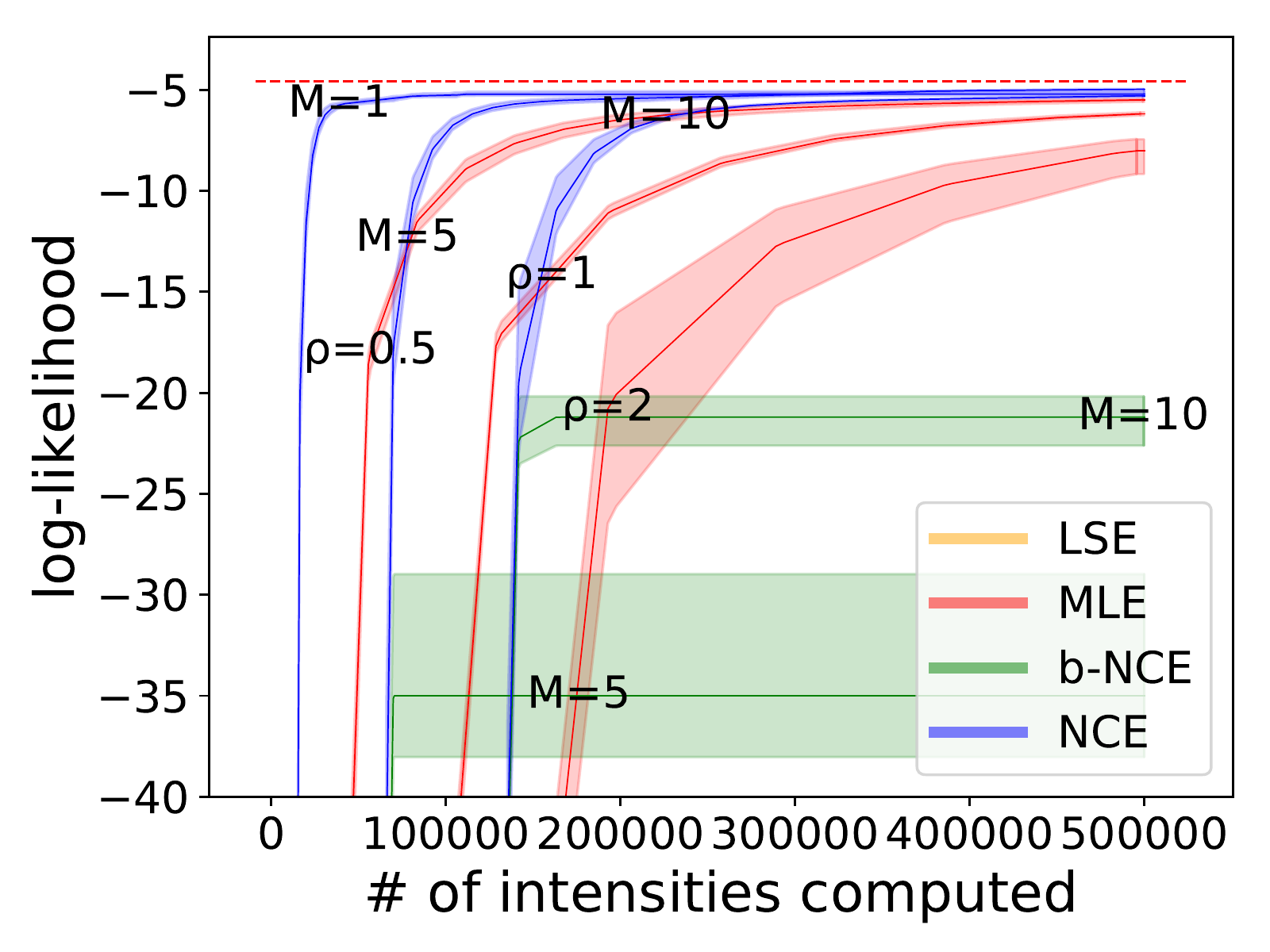}
				\includegraphics[width=0.49\linewidth]{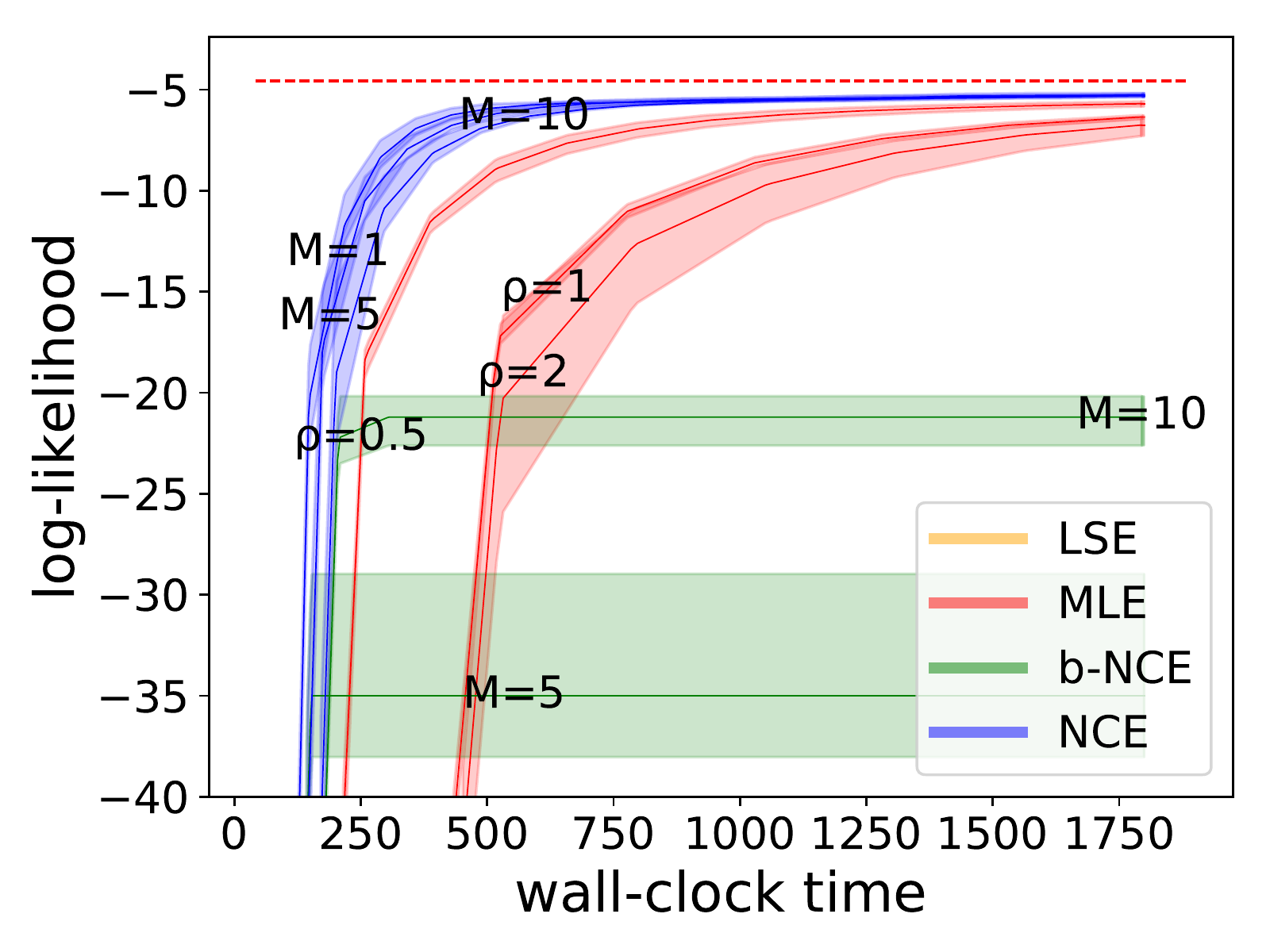}
				\vspace{-16pt}
				\caption{RoboCup: $C=5$}\label{fig:robocup_c5}
			\end{center}
		\end{subfigure}
		~
		\begin{subfigure}[b]{0.49\linewidth}
			\begin{center}
				\includegraphics[width=0.49\linewidth]{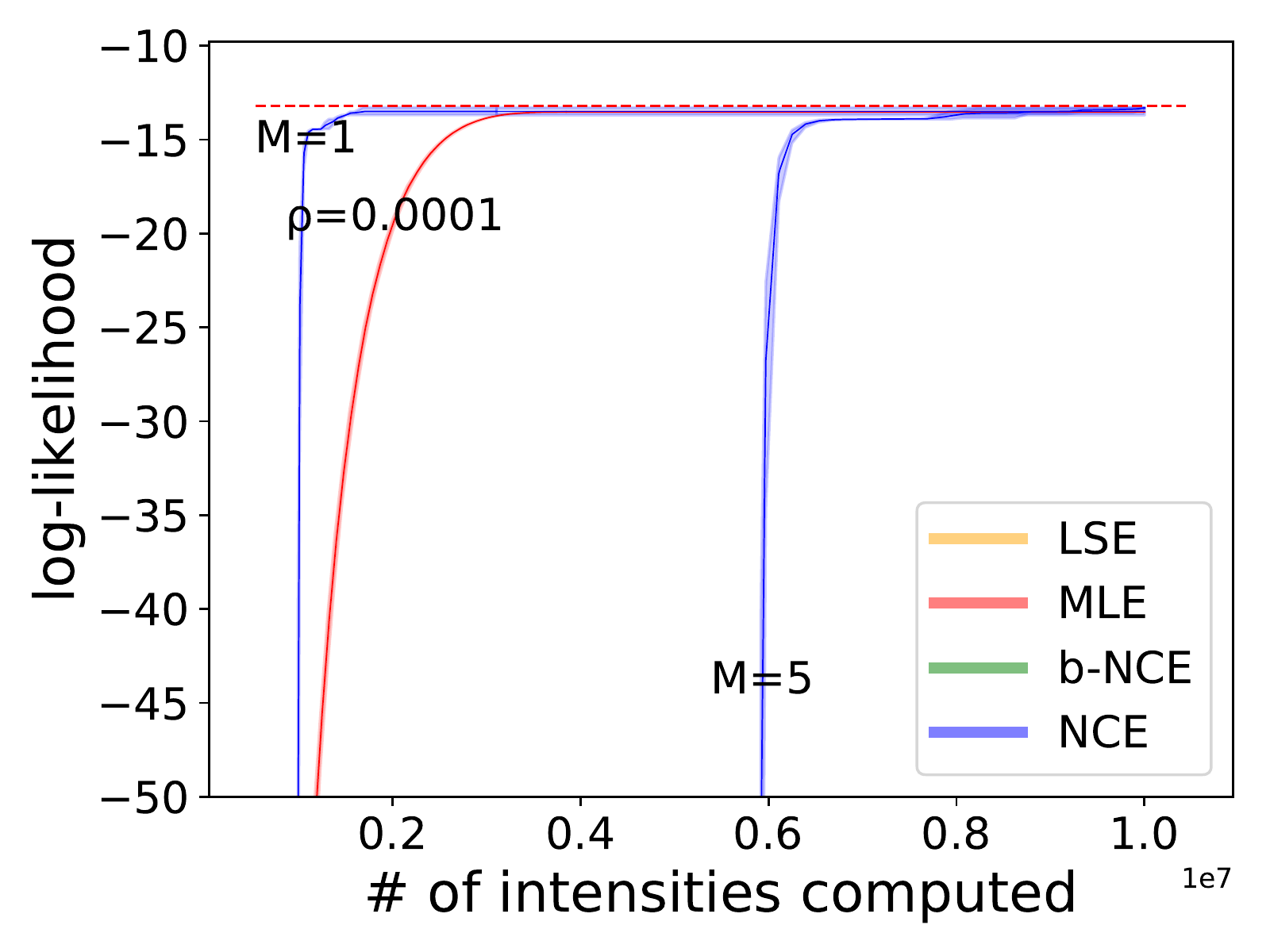}
				\includegraphics[width=0.49\linewidth]{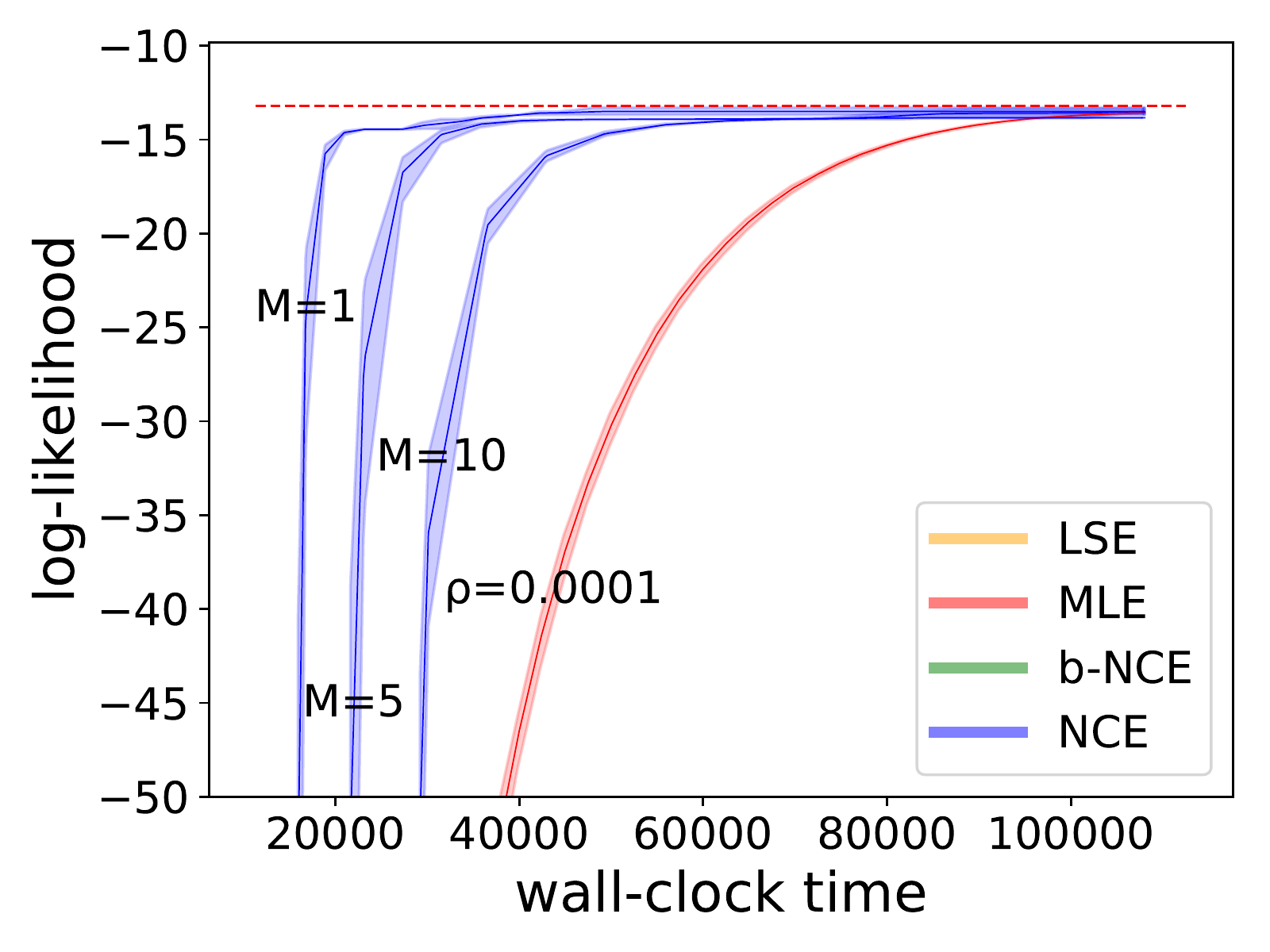}
				\vspace{-16pt}
				\caption{IPTV: $C=49$}\label{fig:iptv_c49}
			\end{center}
		\end{subfigure}
		\vspace{-16pt}
		\caption{Learning curves of MLE and NCE on RoboCup and IPTV datasets.}
		\label{fig:dyna}
	\end{center}
\end{figure*}

{\bfseries RoboCup} \citep{chen-08-robocup}. This dataset logs actions of robot players during RoboCup soccer games.
The set of possible event types dynamically changes over time (e.g., only ball possessor can kick or pass) as the ball is frequently transferred between players (by passing or stealing). 
There are $K=528$ event types over all time, but only about $20$ of them are possible at any given time. 

{\bfseries IPTV} \citep{xu-18-online}. This dataset contains time-stamped records of 1000 users watching 49 TV programs over 2012.  
The users are not able to watch a program until it is released, so the number of event types grows from $K=0$ to $K=49000$ as programs are released one after another. 

The learning curves are displayed in \cref{fig:dyna}. 
On RoboCup, NCE only progresses faster than MLE at the early to middle training stages: $M=5$ and $M=10$ eventually achieved the highest log-likelihood at the same time as MLE and $M=1$ ended up with worse generalization. 
On IPTV, NCE with $M=1$ turned out to learn as well as and much faster than MLE. 
The dynamic architecture makes it hard to parallelize the intensity computation;  MLE in particular performs poorly in wall-clock time, and we needed a remarkably small $\rho$ to let MLE finish within the shown time range.
On both datasets, b-NCE and LSE drastically underperform MLE and NCE: their learning curves increase so slowly and achieve such poor generalization that only b-NCE with $M=5$ and $M=10$ are visible on the graphs.

{\bfseries Ablation Study III:  Effect of $C$.}
In the above figures, we used the coarse-to-fine neural model as $\noise$.  On RoboCup, each action (kick, pass, etc.) has a coarse-grained intensity, so $C=5$.  On IPTV, we partition the event vocabulary by TV program, so $C=49$.
We also experimented with $C=1$: this reduces the number of intensities computed during sampling on both datasets, but has (slightly) worse generalization on RoboCup (since $\noise$ becomes less expressive). 
See \cref{app:diff_c} for more details, including learning curves (\cref{fig:c1}).

\section{Conclusion}\label{sec:conclusion}\label{sec:future}

We have introduced a novel instantiation of the general NCE principle for training a multivariate point process model.
Our objective has the same optimal parameters as the log-likelihood objective (if the model is well-specified), but needs fewer expensive function evaluations and much less wall-clock time in practice. 
This benefit is demonstrated on several synthetic and real-world datasets. 
Moreover, our method is provably consistent and efficient under mild assumptions. 

\section*{Broader Impact}

Our method is designed to train a multivariate point process for probabilistic modeling of event streams.  By describing this method and releasing code, we hope to facilitate probabilistic modeling of continuous-time sequential data in many domains.  Good probabilistic models make it possible to impute missing events, anticipate possible future events, and react accordingly.  They can also be used in exploratory data analysis.

In addition to making it more feasible and more convenient for domain experts to train complex models with many event types, our method reduces the energy cost necessary to do so.

Examples of event streams with potential social impact include a person's detailed food/exercise/sleep/medical event log, their social media interactions, their interactions with educational exercises or games, or their educational or workplace events (for time management and career planning); a customer's interactions with a particular company or its website or other user interface; a company's sales and purchases; geopolitical events, financial events, human activity modeling, music modeling, and dynamic resource requests.  

We are not aware of any negative broader impacts that might stem from publishing this work.

\section*{Disclosure of Funding Sources}

This work was supported by a Ph.D.\ Fellowship Award to the first author by Bloomberg L.P. and a National Science Foundation Grant No.\@ 1718846 to the last author, as well as two Titan X Pascal GPUs donated by NVIDIA Corporation and compute cycles from the Maryland Advanced Research Computing Center.

\section*{Acknowledgments}
We thank the anonymous NeurIPS reviewers and meta-reviewer as well as Hongteng Xu for helpful comments on this paper. 

\bibliographystyle{icml2020_url}
\vspace{0pt}
\bibliography{nce-pointprocess}

\clearpage
\theendnotes  %
\let\footnote\realfootnote  %
\crefname{footnote}{footnote}{footnotes}
\setcounter{footnote}{\value{endnote}}  
\renewcommand{\thefootnote}{\arabic{footnote}}

\clearpage
\appendix
\appendixpage

\section{Proof Details for MLE}\label{app:mle_proof_details}

In this section, we prove the claim in \cref{sec:mle} that $\argmax_{\param} \mle(\param) = \Param^* \defeq \{ \param^*: \model_{\param^*} = p^* \}$.
For this purpose, we first rearrange $\mle(\param) = \E[\data(\es{x}{[0,T)})]{\log \model_{\param}(x_{[0, T)})}$ as below: 
\begin{subequations}\label{eqn:mlederi}
	\begin{align}
	& \sum_{x_{[0, T)}} \model^{*}(x_{[0, T)}) \log \model_{\param}(x_{[0, T)})\label{eqn:mlederi_1} \\
	=& \int_{t=0}^{T}\sum_{\es{x}{[0,t)}}p^*(\es{x}{[0,t)}) \xunderbrace{\sum_{\es{x}{[t,t+\dt)}} \data(\es{x}{[t,t+\dt)} \mid \es{x}{[0,t)}) \log  \model_{\param}(\es{x}{[t,t+\dt)} \mid \es{x}{[0,t)})  \label{eqn:mlederi_2} }{\text{call it } H_{\param}(t, \es{x}{[0,t)}) }%
	\end{align}
\end{subequations}
The intuition for \cref{eqn:mlederi_2} is that due to the form of the autoregressive model, $\log \model_{\param}(x_{[0, T)})$ in \cref{eqn:mlederi_1} can be broken up into a sum of log (infinitesimal) probabilities of $\es{x}{[t,t+\dt)}$ on the infinitesimal intervals $[t,t+dt)$, each probability being conditioned on the past history $\es{x}{[0,t)}$.  When we take the expectation under $\model^*$, each summand gets weighted by the probability that $\es{x}{[0,t)}$ and $\es{x}{[t,t+\dt)}$ would take on the values in that summand.  
This gives a form \eqref{eqn:mlederi_2} that aggregates the infinitesimal quantities 
$H_{\param}(t, \es{x}{[0,t)})$ over possible times $t \in [0,T)$ and possible histories $\es{x}{[0,t)}$.

\begin{proof}

We first observe that $H_{\param}(t, \es{x}{[0,t)})$ is the negative cross-entropy between the conditional distributions of $\data$ and $\model_{\param}$ at time $t$ (both conditioned on history $\es{x}{[0,t)}$).
Technically, $\es{x}{[t,t+\dt)}$ will have an event of type $k$ with probability $\intend{k}{t}\dt$ under $\data$ ($\inten{k}{t}\dt$ under $\model_{\param}$) or has no event at all with probability $1-\sum_{k=1}^{K}\intend{k}{t}\dt$ under $\data$ ($1-\sum_{k=1}^{K}\inten{k}{t}\dt$ under $\model_{\param}$). 
So the term $H_{\param}(t, \es{x}{[0,t)})$ is actually the negative cross entropy between the following two discrete distributions over $\{\nothing,1,\ldots,K\}$:
\begin{subequations}\label{eqn:dist_mle}
	\begin{align}
	&
	\left[ \left( 1- \sum_{k=1}^{K} \intend{k}{t \mid \esh{x}{t}} \dt\right),\; \intend{1}{t \mid \esh{x}{t}} \dt ,\; \ldots\; ,\; \intend{K}{t \mid \esh{x}{t}} \dt \right] \label{eqn:dist_mle_data} \\ 
	&
	\left[ \left(1- \sum_{k=1}^{K} \inten{k}{t \mid \esh{x}{t}} \dt\right),\; \inten{1}{t \mid \esh{x}{t}} \dt ,\; \ldots\;, \; \inten{K}{t \mid \esh{x}{t}} \dt \right] 
	\end{align}
\end{subequations}
The (infinitesimal) negative cross-entropy between them is always smaller than or equal to the negative entropy of the distribution in \cref{eqn:dist_mle_data}: it will be strictly smaller if these two distributions are distinct, and equal when they are identical. 

It is then obvious that any $\param^* \in \Param^*$ maximizes $\mle(\param)$ because it maximizes the negative cross-entropy for any history $\es{x}{[0,t)}$ at any time $t$. 

To check if any other $\bar{\param} \notin \Param^*$ maximizes $\mle(\param)$ as well, we analyze
\begin{align}\label{eqn:mle_diff}
\mle(\bar{\param}) - \mle(\param^*) 
= \int_{t=0}^{T} \sum_{\esh{x}{t}} \data(\esh{x}{t} ) \xunderbrace{(H_{\bar{\param}}({t}, \esh{{x}}{{t}}) - H_{\param^*}({t}, \esh{{x}}{{t}}))}{\text{denote it as } G_{\bar{\param}}(t, \esh{x}{t}) \dt }
\end{align}
where $\param^*$ can be any member in $\Param^*$. 
Note that we can denote $H_{\bar{\param}} - H_{\param^*}$ as $G_{\bar{\param}}\dt$ because the probabilities in $H$ and thus the entropy changes (if any) are all infinitesimal. 

According to the definition of $\bar{\param}$ and $\param^*$, there must exist a stream $\esh{\bar{x}}{T}$, a time $\bar{t} \in (0, T)$ and a type $\bar{k} \in \{1, \ldots, K\}$ such that $\inten{\bar{k}}{\bar{t} \mid \esh{\bar{x}}{\bar{t}}} \neq \intend{\bar{k}}{\bar{t} \mid \esh{\bar{x}}{\bar{t}}}$. 
Therefore, we have $G_{\bar{\param}}(\bar{t}, \esh{\bar{x}}{\bar{t}}) < 0$ since the distributions in \cref{eqn:dist_mle} are distinct for the given history $\esh{\bar{x}}{\bar{t}}$. 
Does this difference lead to any overall change of the entire objective? 

Actually, according to \cref{lem:nonzero} (that we will prove shortly), the existence of such $\esh{\bar{x}}{T}$, $\bar{t}$ and $\bar{k}$ implies that there exists an interval $(t', t'') \subset [0, T)$ such that, for any $t \in (t', t'')$, there exists a set $\mathcal{X}(t)$ of histories with non-zero measure such that any $\esh{x}{t} \in \mathcal{X}(t)$ satisfies $\inten{\bar{k}}{{t} \mid \esh{{x}}{{t}}} \neq \intend{\bar{k}}{{t} \mid \esh{{x}}{{t}}}$. 
That is to say, the fraction of the integral over $(t', t'')$ is a non-infinitesimal negative number: 
\begin{subequations}
\begin{align}
&\int_{t=t'}^{t''} \sum_{\esh{x}{t}} \data(\esh{x}{t} ) G_{\bar{\param}}(t, \esh{x}{t}) \dt \\
=& 
\xunderbrace{\int_{t=t'}^{t''} \sum_{\esh{x}{t} \in \mathcal{X}(t)} \data(\esh{x}{t} ) G_{\bar{\param}}(t, \esh{x}{t}) \dt}{<0} + \xunderbrace{ \int_{t=t'}^{t''} \sum_{\esh{x}{t} \notin \mathcal{X}(t)} \data(\esh{x}{t} ) G_{\bar{\param}}(t, \esh{x}{t}) \dt}{\leq 0}
\end{align}
\end{subequations}
where the second integral $\leq 0$ because $G_{\param}$ always $\leq 0$. 
For the same reason, we also have $\int_{t=0}^{t'} \sum_{\esh{x}{t}} \data(\esh{x}{t} ) G_{\bar{\param}}(t, \esh{x}{t}) \dt \leq 0$ and $\int_{t=t''}^{T} \sum_{\esh{x}{t}} \data(\esh{x}{t} ) G_{\bar{\param}}(t, \esh{x}{t}) \dt \leq 0$. 
Then the overall difference must be strictly negative, i.e., 
\begin{align}
	 \mle(\bar{\param}) - \mle(\param^*) < 0
\end{align}
Note that this inequality holds for any $\bar{\param} \notin \Param^*$ and any $\param^* \in \Param^*$, meaning that $\param^* \in \Param^*$ is necessary to maximize the objective. 

Now the proof of $\argmax_{\param} \mle(\param) =  \Param^*$ is complete. 

\end{proof}

\begin{restatable}{lemma}{nonzero}\label{lem:nonzero}
	Suppose that we have two intensity functions that meet \cref{asmp:inten_cont}: they have different parameters $\param$ and $\param^*$ and are denoted as $\inten{k}{t \mid \esh{x}{t}}$ and $\intend{k}{t \mid \esh{x}{t}}$ respectively. 
	If there exists a stream $\esh{\bar{x}}{T}$, a time $\bar{t} \in (0, T)$ and a type $\bar{k} \in \{1, \ldots, K\}$ such that $\inten{\bar{k}}{\bar{t} \mid \esh{\bar{x}}{\bar{t}}} \neq \intend{\bar{k}}{\bar{t} \mid \esh{\bar{x}}{\bar{t}}}$, then there exists an open interval $(t', t'') \subset [0, T)$ such that, for any $t \in (t', t'')$, there exists a set $\mathcal{X}$ of histories with non-zero measure such that any $\esh{x}{t} \in \mathcal{X}$ satisfies $\inten{\bar{k}}{{t} \mid \esh{{x}}{{t}}} \neq \intend{\bar{k}}{{t} \mid \esh{{x}}{{t}}}$. 
\end{restatable}

This lemma says: if $\param$ and $\param^*$ are meaningfully different in that they predict different intensities at time $t$ for \emph{some history}, then they actually do so for a \emph{set of histories of non-zero measure}, making this difference visible in the objective functions like $\mle(\param)$ (see above) and $\nce(\param)$ (see \cref{app:nce_proof_details}).
Note that previous work did not encounter this since they only worked on either non-sequential data (e.g., \citet{gutmann-10-nce,gutmann-12-nce}) or discrete-time sequential data (e.g., \citet{ma-18-nce}).

\begin{proof}

We first prove the existence of an interval $(t', t'')$ such that $\inten{\bar{k}}{{t} \mid \esh{\bar{x}}{{t}}} \neq \intend{\bar{k}}{{t} \mid \esh{\bar{x}}{{t}}}$ for the given stream $\esh{\bar{x}}{T}$ and any time $t \in (t', t'')$. 
It turns out to be straightforward under \cref{asmp:inten_cont}: since the intensity functions are continuous between events, we can construct this interval by expanding from the given time $\bar{t}$ until $\inten{\bar{k}}{{t} \mid \esh{\bar{x}}{{t}}} = \intend{\bar{k}}{{t} \mid \esh{\bar{x}}{{t}}}$. 

We use $d$ to denote the maximal difference between the intensities over $(t', t'')$, i.e., $d \defeq \max_{t \in (t', t'')} | \inten{\bar{k}}{{t} \mid \esh{\bar{x}}{{t}}} - \intend{\bar{k}}{{t} \mid \esh{\bar{x}}{{t}}} |$. 
Then, to facilitate the rest of the proof, we shrink the interval $(t', t'')$ such that $| \inten{\bar{k}}{{t} \mid \esh{\bar{x}}{{t}}} - \intend{\bar{k}}{{t} \mid \esh{\bar{x}}{{t}}} | > d/2$ for any time $t \in (t', t'')$. 

Now, for any time $t \in (t', t'')$, we prove the existence of the set described in \cref{lem:nonzero} by constructing it. 

We initialize this set as $\{\esh{\bar{x}}{t}\}$. 
If $\esh{\bar{x}}{t}$ doesn't have any event, then its probability $\model(\esh{\bar{x}}{{t}} ) = \exp(-\int_{s=0}^{{t}} \sum_{k=1}^{K} {\lambda}_{k}(s  \mid \esh{\bar{x}}{s}) \ds )$ is not infinitesimal and this set already has non-zero measure. 

What if $\esh{\bar{x}}{t}$ has $I \geq 1$ events at times $0 < t_1 < \ldots < t_I < t$? 
Intuitively, we can construct many other histories satisfying the intensity inequality by slightly shifting the time of each event: as long as they aren't shifted by too far, the $d/2$ difference between intensities won't vanish (even if it decreases). 
See the formal proof as below. 

In the case of $I \geq 1$, the probability $\model(\esh{\bar{x}}{{t}} )$ is infinitesimal in the order of $(\dt)^{I}$: $\model(\esh{\bar{x}}{{t}} ) = \prod_{i=1}^{I} (\inten{\bar{x}_{t_i}}{t_i \mid \esh{\bar{x}}{t_i} } \dt) \exp(-\int_{s=0}^{{t}} \sum_{k=1}^{K} {\lambda}_{k}(s  \mid \esh{\bar{x}}{s}) \ds )$. 
Therefore, to construct a set with non-zero measure, the number of histories satisfying the inequality has to be in the order of $(\frac{1}{\dt})^I$. 

We define an open interval $(t'_1, t''_1)$ that covers $t_1$ but not any other event time. 
Now we can construct uncountably many---in the order of $\frac{1}{\dt}$---histories $\esh{x}{t}$ by freely shifting the event time $t_1$ inside $(t'_1, t''_1)$. 
Suppose that $t_1$ has been shifted by $\Delta \in \Real$. %
Under \cref{asmp:inten_cont}, there is a continuous function $c(\Delta)$ such that $c(0) = 0$ and 
\begin{align}
	\inten{\bar{k}}{{t} \mid \esh{x}{t}} - \intend{\bar{k}}{{t} \mid \esh{x}{t} } 
	= \inten{\bar{k}}{{t} \mid \esh{\bar{x}}{{t}}} - \intend{\bar{k}}{{t} \mid \esh{\bar{x}}{{t}}} + c(\Delta)
\end{align}
meaning that the intensity difference will change by $c(\Delta)$. 
By triangle inequality, we have
\begin{align}
	\left| \inten{\bar{k}}{{t} \mid \esh{x}{t}} - \intend{\bar{k}}{{t} \mid \esh{x}{t} } \right| 
	\geq \left|  \left| \inten{\bar{k}}{{t} \mid \esh{\bar{x}}{{t}}} - \intend{\bar{k}}{{t} \mid \esh{\bar{x}}{{t}}}  \right| - \left| c(\Delta) \right| \right|
\end{align}
Since $c(\Delta)$ is continuous, as long as we make $\left| \Delta \right|$ small enough, we'll have $\left| c(\Delta) \right| \leq d/2$ and then the following inequality holds: 
\begin{align}\label{eqn:ineq1}
\left| \inten{\bar{k}}{{t} \mid \esh{x}{t}} - \intend{\bar{k}}{{t} \mid \esh{x}{t} } \right| 
\geq \left|  \inten{\bar{k}}{{t} \mid \esh{\bar{x}}{{t}}} - \intend{\bar{k}}{{t} \mid \esh{\bar{x}}{{t}}}  \right| - \left| c(\Delta) \right| > d/2 - d/2 = 0
\end{align}
meaning that the intensities given the new history are still different. 
Therefore, as long as we keep the interval $(t'_1, t''_1)$ small enough, we'll have order-$\frac{1}{\dt}$ many histories and the inequality in \cref{eqn:ineq1} holds given any of them. 

Recall that we need order-$(\frac{1}{\dt})^{I}$ many such histories. 
We can obtain them by simply defining $I$ \emph{disjoint} open intervals $(t'_1, t''_1), \ldots, (t'_I, t''_I)$ such that $t_i \in (t'_i, t''_i)$ and freely shifting each event time $t_i$ inside $(t'_i, t''_i)$. 
Suppose that $t_i$ has been shifted by $\Delta_i \in \Real$, 
Under \cref{asmp:inten_cont}, there is a continuous function $c(\Delta_1, \ldots, \Delta_I)$ such that $c(0, \ldots, 0) = 0$ and
\begin{align}
\inten{\bar{k}}{{t} \mid \esh{x}{t}} - \intend{\bar{k}}{{t} \mid \esh{x}{t} } 
= \inten{\bar{k}}{{t} \mid \esh{\bar{x}}{{t}}} - \intend{\bar{k}}{{t} \mid \esh{\bar{x}}{{t}}} + c(\Delta_1, \ldots, \Delta_I)
\end{align}
Since $c$ is a continuous function, there exist $I$ positive real numbers $\bar{\Delta}_1, \ldots, \bar{\Delta}_I$ such that $\left| c(\Delta_1, \ldots, \Delta_I) \right| \leq d/2$ as long as $\left| \Delta_i \right| \leq \bar{\Delta}_i$ holds for all $i=1, \ldots, I$.  
In this case, by triangle inequality, we still have 
\begin{align}\label{eqn:ineqi}
\left| \inten{\bar{k}}{{t} \mid \esh{x}{t}} - \intend{\bar{k}}{{t} \mid \esh{x}{t} } \right| 
\geq \left|  \inten{\bar{k}}{{t} \mid \esh{\bar{x}}{{t}}} - \intend{\bar{k}}{{t} \mid \esh{\bar{x}}{{t}}}  \right| - \left| \Delta_i \right| > 0
\end{align}

Now we have order-$(\frac{1}{\dt})^{I}$ many histories: each of them has order-$(\dt)^I$ probability and the inequality in \cref{eqn:ineqi} holds given any of them. 
That is to say, the set of these histories has non-zero measure and we have $\inten{\bar{k}}{{t} \mid \esh{x}{t}} \neq \intend{\bar{k}}{{t} \mid \esh{x}{t} }$ given any $\esh{x}{t}$ in this set. 

This completes the proof. 

\end{proof}

\section{NCE Details}\label{app:method_details}\label{app:nce_proof_details}

In this section, we will discuss the theoretical guarantees of our NCE method in detail. 

\subsection{Derivation Details}\label{app:derivation}

In this section, we show how to get the rearranged NCE objective in \cref{sec:theory} from \cref{eqn:nce_inten}. 

First of all, we observe that:
\begin{subequations}\label{eqn:deri1}
\begin{align}
&\E[\esm{x}{[0,T)}{0}\sim \data, \esm{x}{[0,T)}{1:M} \sim \noise]{\sum_{t : \esm{x}{t}{0} \neq\nothing } \log \frac{ \inten{\esm{x}{t}{0}}{t \mid \esmh{x}{t}{0}} }{ \underline{\lambda}_{\esm{x}{t}{0}}(t \mid \esmh{x}{t}{0}) } + \sum_{m=1}^M \sum_{t : \esm{x}{t}{m}\neq\nothing } \log \frac{ \intenq{\esm{x}{t}{m}}{t \mid \esmh{x}{t}{0}} }{ \underline{\lambda}_{\esm{x}{t}{m}}(t \mid \esmh{x}{t}{0}) } } \label{eqn:deri1_1}  \\
=&
\int_{t=0}^{T} \E[\esm{x}{[0,t)}{0} \sim \data]{\sum_{k=1}^{K} \intend{k}{t \mid \esmh{x}{t}{0}} \dt \log \frac{ \inten{\esm{x}{t}{0}}{t \mid \esmh{x}{t}{0}} }{ \underline{\lambda}_{\esm{x}{t}{0}}(t \mid \esmh{x}{t}{0}) } + \sum_{m=1}^M \sum_{k=1}^{K} \intenq{k}{t \mid \esmh{x}{t}{0}} \dt \log \frac{ \intenq{\esm{x}{t}{m}}{t \mid \esmh{x}{t}{0}} }{ \underline{\lambda}_{\esm{x}{t}{m}}(t \mid \esmh{x}{t}{0}) }}  \label{eqn:deri1_2}
\end{align}
\end{subequations}
This rearrangement is similar to that of \crefrange{eqn:mlederi_1}{eqn:mlederi_2}.
The intuition of \cref{eqn:deri1_1} is that we sample $M$ i.i.d.\@ noise streams $\esm{x}{[0, T)}{1}, \ldots, \esm{x}{[0, T)}{M}$ for each possible real data $\esm{x}{[0, T)}{0}$, sum up the log-ratio whenever $\esm{{x}}{t}{0:M}$ has an event, and then take the expectation over all the possible real data $\esm{x}{[0, T)}{0}$. 
The intuition of \cref{eqn:deri1_2} is that we draw noise samples $\esm{x}{t}{1}, \ldots, \esm{x}{t}{M}$ for each real history $\esmh{x}{t}{0}$ at each time $t$, compute the log-ratio if $\esm{{x}}{t}{0:M}$ has an event, take the expectation of the log-ratio over all the possible real histories and then sum over all the possible times. 
Therefore, these two expectations are equal. 

We further rearrange \cref{eqn:deri1} as
\begin{subequations}\label{eqn:deri2}
\begin{align}
&
= \int_{t=0}^{T} \E[\esm{x}{[0,t)}{0} \sim \data]{\sum_{k=1}^{K} \left( \intend{k}{t \mid \esmh{x}{t}{0}} \dt \log \frac{ \inten{k}{t \mid \esmh{x}{t}{0}} }{ \underline{\lambda}_{k}(t \mid \esmh{x}{t}{0}) } + \sum_{m=1}^M \intenq{k}{t \mid \esmh{x}{t}{0}} \dt \log \frac{ \intenq{k}{t \mid \esmh{x}{t}{0}} }{ \underline{\lambda}_{k}(t \mid \esmh{x}{t}{0}) } \right) } \\
=& 
\int_{t=0}^{T} \E[\esm{x}{[0,t)}{0} \sim \data]{\sum_{k=1}^{K} \left( \intend{k}{t \mid \esmh{x}{t}{0}} \dt \log \frac{ \inten{k}{t \mid \esmh{x}{t}{0}} }{ \underline{\lambda}_{k}(t \mid \esmh{x}{t}{0}) } + M \intenq{k}{t \mid \esmh{x}{t}{0}} \dt \log \frac{ \intenq{k}{t \mid \esmh{x}{t}{0}} }{ \underline{\lambda}_{k}(t \mid \esmh{x}{t}{0}) } \right) } \\
=& 
\int_{t=0}^{T} \E[\esm{x}{[0,t)}{0} \sim \data]{\sum_{k=1}^{K} \underline{\lambda}^*_{k}(t \mid \esmh{x}{t}{0}) \dt \left( \frac{\intend{k}{t \mid \esmh{x}{t}{0}}}{\underline{\lambda}^*_{k}(t \mid \esmh{x}{t}{0})}  \log \frac{ \inten{k}{t \mid \esmh{x}{t}{0}} }{ \underline{\lambda}_{k}(t \mid \esmh{x}{t}{0}) } + M \frac{\intenq{k}{t \mid \esmh{x}{t}{0}}}{\underline{\lambda}^*_{k}(t \mid \esmh{x}{t}{0})} \log \frac{ \intenq{k}{t \mid \esmh{x}{t}{0}} }{ \underline{\lambda}_{k}(t \mid \esmh{x}{t}{0}) } \right) } 
\end{align}
\end{subequations}
where $\underline{\lambda}^*_{k}(t \mid \esmh{x}{t}{0}) \defeq \intend{k}{t \mid \esmh{x}{t}{0}} + M \intenq{k}{t \mid \esmh{x}{t}{0}}$ can be thought of as the intensity of type $k$ under the superposition of $\data$ and $M$ copies of $\noise$. 

Now we obtain the final rearranged objective:
\begin{align}\label{eqn:nce_entropy}
\int_{t=0}^{T} \sum_{\esmh{x}{t}{0}} \data(\esmh{x}{t}{0} ) \sum_{k=1}^{K} \underline{\lambda}^*_{k}(t  \mid \esmh{x}{t}{0}) \xunderbrace{\left( \frac{{\lambda}^*_{k}(t\mid \esmh{x}{t}{0})}{\underline{\lambda}^*_{k}(t \mid \esmh{x}{t}{0}) } \log \frac{{\lambda}_{k}(t\mid \esmh{x}{t}{0})}{\underline{\lambda}_{k}(t \mid \esmh{x}{t}{0}) } + M \frac{{\lambda}^q_{k}(t\mid \esmh{x}{t}{0})}{\underline{\lambda}^*_{k}(t \mid \esmh{x}{t}{0}) } \log \frac{{\lambda}^q_{k}(t\mid \esmh{x}{t}{0})}{\underline{\lambda}_{k}(t \mid \esmh{x}{t}{0}) } \right)}{\text{call it } H_{\param}(k, t, \esmh{x}{t}{0}) } \dt
\end{align}

\subsection{Optimality Proof Details}\label{app:optimal}

In this section, we prove \cref{thm:optimality} that we stated in \cref{sec:theory}. Recall the theorem:
\optimality*

We first need to highlight the key insight that $H_{\param}(k, t, \esmh{x}{t}{0})$ in \cref{eqn:nce_entropy} is the negative cross-entropy between the following two discrete distributions over $\{\nothing,1,\ldots,K\}$:
\begin{subequations}\label{eqn:dist}
	\begin{align}
	&[ \frac{{\lambda}^*_{k}(t\mid \esmh{x}{t}{0})}{\underline{\lambda}^*_{k}(t \mid \esmh{x}{t}{0}) }  , \frac{{\lambda}^q_{k}(t\mid \esmh{x}{t}{0})}{\underline{\lambda}^*_{k}(t \mid \esmh{x}{t}{0}) }, \ldots, \frac{{\lambda}^q_{k}(t\mid \esmh{x}{t}{0})}{\underline{\lambda}^*_{k}(t \mid \esmh{x}{t}{0}) }] \label{eqn:dist_data} \\ 
	&[ \frac{{\lambda}_{k}(t\mid \esmh{x}{t}{0})}{\underline{\lambda}_{k}(t \mid \esmh{x}{t}{0}) }  , \xunderbrace{\frac{{\lambda}^q_{k}(t\mid \esmh{x}{t}{0})}{\underline{\lambda}_{k}(t \mid \esmh{x}{t}{0}) }, \ldots, \frac{{\lambda}^q_{k}(t\mid \esmh{x}{t}{0})}{\underline{\lambda}_{k}(t \mid \esmh{x}{t}{0}) } }{\text{length is } M} ] 
	\end{align}
\end{subequations}
This negative cross-entropy is always smaller than or equal to the negative entropy of the distribution in \cref{eqn:dist_data}: it will be strictly smaller if these two distributions are distinct and equal when they are identical. Notice that in contrast to the negative cross-entropy at \cref{eqn:dist_mle}, this negative cross-entropy here is not infinitesimal.  

\begin{proof}

The ``if'' part is straightforward to prove. 
Any $\param$ for which $\model_{\param}=\data$ would make ${\lambda}_{k}(t\mid \esmh{x}{t}{0}) = {\lambda}^*_{k}(t\mid \esmh{x}{t}{0})$, thus maximizing the negative cross-entropy between the two distributions in \cref{eqn:dist}, for any type $k$ and any real history $\esmh{x}{t}{0}$ at any time $t$.
Then the NCE objective in \cref{eqn:nce_entropy} is obviously maximized. 

To check if any other $\bar{\param} \notin \Param^* \defeq \{\param^* : \model_{\param^*}=\data\}$ maximizes $\nce(\param)$ as well, we analyze
\begin{align*}%
\nce(\bar{\param}) - \nce(\param^*) 
= \int_{t=0}^{T} \sum_{\esmh{x}{t}{0}} \data(\esmh{x}{t}{0} ) \sum_{k=1}^{K} \underline{\lambda}^*_{k}(t  \mid \esmh{x}{t}{0})  \xunderbrace{ \left( H_{\bar{\param}}(k, t, \esmh{x}{t}{0}) - H_{\param^*}(k, t, \esmh{x}{t}{0}) \right) }{ \text{denote it as } G_{\bar{\param}} ( k, t, \esmh{x}{t}{0} ) }  \dt
\end{align*}
where $\param^*$ can be any member in $\Param^*$. 
Note that $G_{\bar{\param}}$ is not infinitesimal because the probabilities in $H$ and thus the entropy changes (if any) are not infinitesimal. 

According to the definition of $\bar{\param}$ and $\param^*$, there must exist a stream $\esh{\bar{x}}{T}$, a time $\bar{t} \in (0, T)$ and a type $\bar{k} \in \{1, \ldots, K\}$ such that $\inten{\bar{k}}{\bar{t} \mid \esh{\bar{x}}{\bar{t}}} \neq \intend{\bar{k}}{\bar{t} \mid \esh{\bar{x}}{\bar{t}}}$. 
Therefore, we have $G_{\bar{\param}}(\bar{k}, \bar{t}, \esh{\bar{x}}{\bar{t}}) < 0$ since the distributions in \cref{eqn:dist} are distinct for the given history $\esh{\bar{x}}{\bar{t}}$. 
Does this difference lead to any overall change of the entire objective? 

Actually, according to \cref{lem:nonzero} in \cref{app:mle_proof_details}, the existence of such $\esh{\bar{x}}{T}$, $\bar{t}$ and $\bar{k}$ implies that there exists an interval $(t', t'') \subset [0, T)$ such that, for any $t \in (t', t'')$, there exists a set $\mathcal{X}(t)$ of histories with non-zero measure such that any $\esh{x}{t} \in \mathcal{X}(t)$ satisfies $\inten{\bar{k}}{{t} \mid \esh{{x}}{{t}}} \neq \intend{\bar{k}}{{t} \mid \esh{{x}}{{t}}}$. 
Then, given any of these histories, the entropy difference $G_{\bar{\param}}$ would be $< 0$.
That is to say, the following integral must be a non-infinitesimal negative number: 
\begin{subequations}
\begin{align}
	&\int_{t=0}^{T} \sum_{\esmh{x}{t}{0}} \data(\esmh{x}{t}{0} ) \underline{\lambda}^*_{\bar{k}}(t  \mid \esmh{x}{t}{0}) G_{\bar{\param}}(\bar{k}, t, \esmh{x}{t}{0}) \dt \\
	=& 
	\int_{t=t'}^{t''} \sum_{\esmh{x}{t}{0} \in \mathcal{X}(t)} \data(\esmh{x}{t}{0} ) \underline{\lambda}^*_{\bar{k}}(t  \mid \esmh{x}{t}{0}) G_{\bar{\param}}(\bar{k}, t, \esmh{x}{t}{0}) \dt &&\text{ ($< 0$ )} \\
	+& 
	\int_{t=t'}^{t''} \sum_{\esmh{x}{t}{0} \notin \mathcal{X}(t)} \data(\esmh{x}{t}{0} ) \underline{\lambda}^*_{\bar{k}}(t  \mid \esmh{x}{t}{0}) G_{\bar{\param}}(\bar{k}, t, \esmh{x}{t}{0}) \dt &&\text{ ($\leq 0$ )} \\
	+& 
	\int_{t=0}^{t'} \sum_{\esmh{x}{t}{0}} \data(\esmh{x}{t}{0} ) \underline{\lambda}^*_{\bar{k}}(t  \mid \esmh{x}{t}{0}) G_{\bar{\param}}(\bar{k}, t, \esmh{x}{t}{0}) \dt &&\text{ ($\leq 0$ )} \\
	+& 
	\int_{t=t''}^{T} \sum_{\esmh{x}{t}{0}} \data(\esmh{x}{t}{0} ) \underline{\lambda}^*_{\bar{k}}(t  \mid \esmh{x}{t}{0}) G_{\bar{\param}}(\bar{k}, t, \esmh{x}{t}{0}) \dt &&\text{ ($\leq 0$ )}
\end{align}
\end{subequations}
Therefore, the overall difference must be $< 0$ as well:  
\begin{subequations}
\begin{align}
\mle(\bar{\param}) - \mle(\param^*) 
&= 
\int_{t=0}^{T} \sum_{\esmh{x}{t}{0}} \data(\esmh{x}{t}{0} ) \sum_{k=1}^{K} \underline{\lambda}^*_{k}(t  \mid \esmh{x}{t}{0}) G_{\bar{\param}}(k, t, \esmh{x}{t}{0}) \dt \\
&=
\int_{t=0}^{T} \sum_{\esmh{x}{t}{0}} \data(\esmh{x}{t}{0} ) \underline{\lambda}^*_{\bar{k}}(t  \mid \esmh{x}{t}{0}) G_{\bar{\param}}(\bar{k}, t, \esmh{x}{t}{0}) \dt  &&\text{ ($< 0$ )} \\
&+ 
\int_{t=0}^{T} \sum_{\esmh{x}{t}{0}} \data(\esmh{x}{t}{0} ) \sum_{k\neq\bar{k}} \underline{\lambda}^*_{k}(t  \mid \esmh{x}{t}{0}) G_{\bar{\param}}(k, t, \esmh{x}{t}{0}) \dt &&\text{ ($\leq 0$ )}
\end{align}
\end{subequations}

Note that $\mle(\bar{\param}) - \mle(\param^*) < 0$ holds any $\bar{\param} \notin \Param^*$ and any $\param^* \in \Param^*$, meaning that $\param^* \in \Param^*$ is necessary to maximize the objective. 
Then the proof of the ``only if'' part is complete. 

Now we have proved both the ``if'' and ``only if'' parts so the proof is complete. 

\end{proof}

\subsection{Consistency Proof Details}\label{app:consistency}

To discuss the statistical consistency (in this section) and efficiency (in \cref{app:efficiency}), we first need to spell out the empirical version of the objective
\begin{align}\label{eqn:nce_n}
\nce^{N}(\param)
&=
\frac{1}{N} \sum_{n=1}^{N} \left( \sum_{t : \esm{x}{t,n}{0} \neq\nothing } \log \frac{ \inten{\esm{x}{t,n}{0}}{t \mid \esm{x}{[0,t),n}{0}} }{ \underline{\lambda}_{\esm{x}{t,n}{0}}(t \mid \esm{x}{[0,t),n}{0}) } + \sum_{m=1}^M \sum_{t : \esm{x}{t,n}{m}\neq\nothing } \log \frac{ \intenq{\esm{x}{t,n}{m}}{t \mid \esm{x}{[0,t),n}{0}} }{ \underline{\lambda}_{\esm{x}{t,n}{m}}(t \mid \esm{x}{[0,t),n}{0}) } \right) 
\end{align}
where the subscript $_{n}$ denotes the $n$\th i.i.d.\@ draw of the observed sequence and the $M$ noise samples for this sequence. 
It is obvious that $\lim_{N\rightarrow\infty} \nce^{N}(\param) \rightarrow \nce(\param)$. 

To analyze the consistency, we make the following assumptions: 
\begin{assumption}[Continuity wrt.\ $\param$]\label{asmp:inten_cont_param}	
	For any history $\esh{x}{t}$ and event type $k \in \{1, \ldots, K\}$, $\inten{k}{t \mid \es{x}{[0,t)}}$ is continuous with respect to $\param$.
\end{assumption}
\begin{assumption}[Compactness]\label{asmp:compact}	
	The set of optimal parameters $\Param^*$ is contained in the interior of a compact set $\Param \subset \Real^{|\param|}$. 
\end{assumption}
They are analogous to assumptions 4.2 and 4.3 of \citet{ma-18-nce} respectively. 

Our NCE method turns out to be strongly consistent in the sense that: 
\begin{restatable}[Consistency]{theorem}{consistency}\label{thm:nce_cons}
	Under \cref{asmp:inten_cont_param,asmp:exist,asmp:compact}, 
	for any $\param \in \Param_{\text{\upshape NC}}^N \defeq \argmax_{\param} \nce^N(\param)$ and $M \geq 1$, with probability 1, we have $ \lim_{N\rightarrow\infty} \min_{\param^* \in \Param^*} \| \param - \param^* \| = 0$ where $\|\cdot\|$ is the L$_2$ norm. 
\end{restatable}
The intuition of this theorem is that: since the two functions $\nce^{N}(\param)$ and $\nce(\param)$ will become the same as $N \rightarrow \infty$ and they are continuous with respect to $\param$, then any $\param \in \argmax_{\param} \nce^{N}(\param)$ has to be close to some member of the set $\argmax_{\param} \nce(\param)$. 
The full proof is almost identical to the proof of Theorem 4.2 in \citet{ma-18-nce}. 
But we will still spell it out in our notation for completeness.

\begin{proof}

Under the assumption in \cref{thm:nce_cons}, by classical large sample theory \citep{ferguson-96-course}, we have 
\begin{align}
	\prob{\lim_{N\rightarrow\infty}\sup_{\param \in \Param'} | \nce^{N}(\param) - \nce(\param) | = 0 } = 1 \text{ for any compact set } \Param' \subset \Param 
\end{align}
where $\probsym$ stands for ``probability''.
Since $| \nce^{N}(\param) - \nce(\param) | \geq \nce^{N}(\param) - \nce(\param)$, we have 
\begin{align}\label{eqn:sup_leq_1}
	\prob{\limsup_{N\rightarrow\infty}\sup_{\param \in \Param'} (\nce^{N}(\param) - \nce(\param) ) \leq 0 } = 1
\end{align}
Moreover, for any $\param'^{N} \in \argmax_{\param \in \Param'} \nce^{N}(\param)$, we have 
\begin{align}\label{eqn:ineq_sup_1}
	\sup_{\param \in \Param'} (\nce^{N}(\param) - \nce(\param) ) 
	\geq \nce^{N}(\param'^{N}) - \nce(\param'^N) 
	\geq \sup_{\param \in \Param'} \nce^{N}(\param) - \sup_{\param \in \Param'}\nce(\param) 
\end{align}
Plugging \cref{eqn:ineq_sup_1} into \cref{eqn:sup_leq_1} gives 
\begin{align}\label{eqn:sup_leq_sup}
	\prob{ \limsup_{N\rightarrow\infty} \sup_{\param \in \Param'} \nce^{N}(\param) - \sup_{\param \in \Param'} \nce(\param) \leq 0  } 
	= 
	\prob{\limsup_{N\rightarrow\infty}\sup_{\param \in \Param'} \nce^{N}(\param) \leq \sup_{\param \in \Param'} \nce(\param)  } = 1
\end{align}

For any $\delta > 0$, we define $\Param_{\delta} \defeq \{\param  : \min_{\param^* \in \Param^*} \| \param - \param^* \| > \delta \}$ and have 
\begin{align}\label{eqn:sup_leq_sup_final}
	\prob{\limsup_{N\rightarrow\infty}\sup_{\param \in \Param_{\delta}} \nce^{N}(\param) \leq \sup_{\param \in \Param_{\delta}} \nce(\param) < \sup_{\param \in \Param} \nce(\param)   } = 1
\end{align}

On the other hand, we also have $| \nce^{N}(\param) - \nce(\param) | \geq \nce(\param) - \nce^{N}(\param)$, which gives
\begin{align}\label{eqn:sup_leq_2}
\prob{\limsup_{N\rightarrow\infty}\sup_{\param \in \Param'} (\nce(\param) - \nce^{N}(\param) ) \leq 0 } = 1
\end{align}
For any $\param' \in \argmax_{\param \in \Param'} \nce(\param)$, we have 
\begin{align}\label{eqn:ineq_sup_2}
\sup_{\param \in \Param'} (\nce(\param) - \nce^{N}(\param) ) 
\geq \nce(\param') - \nce^{N}(\param') 
\geq \sup_{\param \in \Param'} \nce(\param) - \sup_{\param \in \Param'}\nce^{N}(\param) 
\end{align}
Plugging \cref{eqn:ineq_sup_2} into \cref{eqn:sup_leq_2} gives 
\begin{align}\label{eqn:inf_geq_sup}
\prob{\sup_{\param \in \Param'} \nce(\param) + \limsup_{N\rightarrow\infty} (- \sup_{\param \in \Param'} \nce^{N}(\param) ) \leq 0  } 
= 
\prob{\liminf_{N\rightarrow\infty}\sup_{\param \in \Param'} \nce^{N}(\param) \geq \sup_{\param \in \Param'} \nce(\param)  } = 1
\end{align}
which, when we let $\Param' = \Param$, gives
\begin{align}\label{eqn:inf_geq_sup_final}
\prob{\liminf_{N\rightarrow\infty}\sup_{\param \in \Param} \nce^{N}(\param) \geq \sup_{\param \in \Param} \nce(\param)  } = 1
\end{align}

Combining \cref{eqn:sup_leq_sup_final} and \cref{eqn:inf_geq_sup_final}, we have that, for any $\param^N \in \Param^N \defeq \argmax_{\param} \nce^N(\param)$ (defined in \cref{thm:nce_cons}), there exists an integer $N'$ such that for any $N \geq N'$
\begin{align}
\prob{ \param^N \notin \Param_{\delta}   } = 1
\end{align}
which holds for any $\delta > 0$ and thus gives 
\begin{align}
\prob{\lim_{N\rightarrow\infty}\min_{\param^* \in \Param^*} \|  \param^N - \param^*  \| = 0  } = 1
\end{align}
which completes the proof of \cref{thm:nce_cons}. 

\end{proof}

\subsection{Efficiency Proof Details}\label{app:efficiency}

To quantify the statistical efficiency of our method, we make the following assumptions: 
\begin{assumption}[Identifiability]\label{asmp:identify}	
	There is only one parameter vector $\param^*$ such that $\model_{\param^*} = \data$. 
\end{assumption}
\begin{assumption}[Differentiability]\label{asmp:inten_diff}	
	For any history $\esh{x}{t}$ and event type $k \in \{1, \ldots, K\}$, $\inten{k}{t \mid \es{x}{[0,t)}}$ is twice continuously differentiable with respect to $\param$.
\end{assumption}
\begin{assumption}[Singularity]\label{asmp:singular}	
	The Fisher information matrix $\vecb{I}_{*}$ under the model $\model_{\param}$ is non-singular. 
\end{assumption}
They are analogous to assumptions 4.4, 4.6 and 4.7 of \citet{ma-18-nce} respectively. 

Before we show the efficiency of our method, we first spell out the definition of $\vecb{I}_{*}$:
\begin{align}\label{eqn:fisher}
\vecb{I}_{*} 
&\defeq \E[ \esh{x}{T} \sim \data ]{ \nabla_{\param} \log \model_{\param^*}(\esh{x}{T}) \nabla_{\param} \log \model_{\param^*}(\esh{x}{T})^\top }
\end{align}
where $\nabla_{\param} \log \model_{\param^*}$ stands for ``the gradient of $\log \model_{\param}$ with respect to $\param$ at $\param=\param^*$.''
This formula can be rearranged as
\begin{subequations}\label{eqn:fisher_rearrange}
	\begin{align}
	&\int_{t=0}^{T} \E[ \esh{x}{t} \sim \data ]{ \E[\es{x}{[t, t+\dt)} \sim \data]{ \nabla_{\param} \log \model_{\param^*}(\es{x}{[t,t+\dt)} \mid \esh{x}{t}) \nabla_{\param} \log \model_{\param^*}(\es{x}{[t,t+\dt)} \mid \esh{x}{t})^\top } } \\
	=&\int_{t=0}^{T} \E[ \esh{x}{t} \sim \data ]{ \E[\es{x}{[t, t+\dt)} \sim \data]{ \frac{\nabla_{\param} \model_{\param^*}(\es{x}{[t,t+\dt)} \mid \esh{x}{t}) }{\model_{\param^*}(\es{x}{[t,t+\dt)} \mid \esh{x}{t})} \frac{\nabla_{\param} \model_{\param^*}(\es{x}{[t,t+\dt)} \mid \esh{x}{t}) }{\model_{\param^*}(\es{x}{[t,t+\dt)} \mid \esh{x}{t})}^\top } } \\
	=& \int_{t=0}^{T} \E[ \esh{x}{t} \sim \data ]{\sum_{\es{x}{[t,t+\dt)}} \frac{\nabla_{\param} \model_{\param^*}(\es{x}{[t,t+\dt)} \mid \esh{x}{t} ) \nabla_{\param} \model_{\param^*}(\es{x}{[t,t+\dt)} \mid \esh{x}{t})^\top}{\model_{\param^*}(\es{x}{[t,t+\dt)} \mid \esh{x}{t} )} }
	\end{align}
\end{subequations}

Technically, $\es{x}{[t,t+\dt)}$ will have an event of type $k$ with probability $\intend{k}{t}\dt$ under $\data$ ($\inten{k}{t}\dt$ under $\model_{\param}$) or has no event at all with probability $1-\sum_{k=1}^{K}\intend{k}{t}\dt$ under $\data$ ($1-\sum_{k=1}^{K}\inten{k}{t}\dt$ under $\model_{\param}$). 
In the former case, we have $\nabla_{\param} \model_{\param^*} \nabla_{\param} \model_{\param^*}^\top / \model_{\param^*} = \nabla_{\param} \intend{k}{t} \nabla_{\param} \intend{k}{t}^\top \dt / \intend{k}{t}$; in the latter case, we have $\nabla_{\param} \model_{\param^*} = -\sum_{k=1}^{K} \nabla_{\param} \intend{k}{t} \dt$ but $\model_{\param^*} \approx 1$, so $\nabla_{\param} \model_{\param^*} \nabla_{\param} \model_{\param^*}^\top / \model_{\param^*} = o(\dt)$ can be ignored. 
Plugging these quantities into \cref{eqn:fisher_rearrange} gives us
\begin{subequations}
\begin{align}
\vecb{I}_{*} 
&= 
\int_{t=0}^{T} \E[ \esh{x}{t} \sim \data ]{ \sum_{k=1}^{K} \frac{ \nabla_{\param} \intend{k}{t \mid \esh{x}{t}} \nabla_{\param} \intend{k}{t \mid \esh{x}{t}}^\top }{ \intend{k}{t \mid \esh{x}{t}} } \dt  }  \\
&= \int_{t=0}^{T} \sum_{\es{x}{[0,t)}} \data(\esh{x}{t}) \sum_{k=1}^{K} \frac{ \nabla_{\param} \intend{k}{t \mid \esh{x}{t}} \nabla_{\param} \intend{k}{t \mid \esh{x}{t}}^\top }{ \intend{k}{t \mid \esh{x}{t}} } \dt 
\end{align}
\end{subequations}
Note that $\nabla_{\param} \intend{k}{t}$ stands for ``the gradient of $\inten{k}{t}$ with respect to $\param$ at $\param=\param^*$.''

Now we proceed to our efficiency theorem. 
We denote the unique optimal parameter vector as $\param^*$ and use $\hat{\param}$ for the estimate given by maximizing $\nce^{N}(\param)$. 
It turns out that our method approaches \emph{Fisher efficiency} as $M$ grows.
\begin{restatable}[Efficiency]{theorem}{efficiency}\label{thm:nce_effi}
	Under \cref{asmp:exist,asmp:compact,asmp:inten_diff,asmp:identify,asmp:singular}, there exists an integer $\bar{M}$ such that for all $M > \bar{M}$
	\begin{align}
	\sqrt{N} ( \hat{\param}  - \param^* ) \rightarrow \Normal(0, \inv{\vecb{I}_{M}})  \text{ as } N \rightarrow \infty
	\end{align}
	for some non-singular matrix $\inv{\vecb{I}_{M}}$. 
	Moreover, there exist a constant $C > 0$ such that for all $M > \bar{M}$
	\begin{align}
	\| \inv{\vecb{I}_{M}} - \inv{\vecb{I}_{*}} \| \leq C / {M}
	\end{align}
	where $\| \vecb{I} \|$ is the spectral norm of matrix $\vecb{I}$. 
\end{restatable}

\begin{proof}

We first prove that $\sqrt{N} (\hat{\param} - \param^*)$ is asymptotically normal. 
By the Mean-Value Theorem, we have 
\begin{align}
	\nabla_{\param} \nce^{N}(\hat{\param}) = \nabla_{\param} \nce^{N}(\param^*) + (\hat{\param} - \param^* ) \int_{u=0}^{1} \nabla_{\param}^2 \nce^{N}(\param^* + u (\hat{\param} - \param^*)) \dt
\end{align}
Since $\hat{\param}$ maximizes $\nce^{N}$, we have 
\begin{align}\label{eqn:diff}
\hat{\param} - \param^* 
= \inv{\left[ - \int_{u=0}^{1} \nabla_{\param}^2 \nce^{N}(\param^* + u (\hat{\param} - \param^*)) \dt\right] } \nabla_{\param} \nce^{N} (\param^*)
\end{align}
By Law of Large Numbers and \cref{thm:nce_cons}, we have 
\begin{align}\label{eqn:mean}
\int_{u=0}^{1} \nabla_{\param}^2 \nce^{N}(\param^* + u (\hat{\param} - \param^*)) \dt
\rightarrow
\xunderbrace{\E[\esm{x}{[0,T)}{0} \sim \data, \esm{x}{[0,T)}{1:M} \sim \noise]{ \nabla_{\param}^2 \objind(\param^*) }}{ \text{short as } \E{ \nabla_{\param}^2 \objind(\param^*)} } \text{ as } N \rightarrow \infty
\end{align}
where $\objind(\param)$ is defined as the objective for a random draw of $\esm{x}{[0,T)}{0:M}$ and thus is just the term inside the expectation of \cref{eqn:nce_inten}: 
\begin{align}
\objind(\param) \defeq
\sum_{t : \esm{x}{t}{0} \neq\nothing } \log \frac{ \inten{\esm{x}{t}{0}}{t \mid \esmh{x}{t}{0}} }{ \underline{\lambda}_{\esm{x}{t}{0}}(t \mid \esmh{x}{t}{0}) } + \sum_{m=1}^M \sum_{t : \esm{x}{t}{m}\neq\nothing } \log \frac{ \intenq{\esm{x}{t}{m}}{t \mid \esmh{x}{t}{0}} }{ \underline{\lambda}_{\esm{x}{t}{m}}(t \mid \esmh{x}{t}{0}) } 
\end{align}
The term $\nabla_{\param}^2 \objind(\param^*)$ stands for ``the Hessian matrix of $\objind(\param)$ with respect to $\param$ at $\param = \param^*$.''
As for $\nabla_{\param} \nce^{N} (\param^*)$, by Central Limit Theorem, we have 
\begin{align}\label{eqn:var}
	\sqrt{N} \nabla_{\param} \nce^{N} (\param^*) \rightarrow \Normal(0, \xunderbrace{\E[\esm{x}{[0,T)}{0} \sim \data, \esm{x}{[0,T)}{1:M} \sim \noise]{ \nabla_{\param} \objind(\param^*)  \nabla_{\param} \objind(\param^*)^\top  }}{\text{short as } \var[\nabla_{\param} \objind(\param^*)] } )
\end{align}
Combining \cref{eqn:diff,eqn:mean,eqn:var}, we obtain the asymptotic normality 
\begin{align}
	\sqrt{N} (\hat{\param} - \param^* ) \rightarrow \Normal(0, \inv{\E{ \nabla_{\param}^2 \objind(\param^*)}} \var[\nabla_{\param} \objind(\param^*)] \inv{\E{ \nabla_{\param}^2 \objind(\param^*)}} )
\end{align}

Now we compute the covariance matrix of the asymptotic normal distribution. 
Following steps similar to \cref{eqn:deri1,eqn:deri2}, we rearrange $\E{\nabla_{\param}^2 \objind(\param^*)}$ to be
\begin{subequations}
\begin{align}
\E{\nabla_{\param}^2 \objind(\param^*)} 
&= \int_{t=0}^{T} \E[\esm{x}{[0,t)}{0} \sim \data]{\sum_{k=1}^{K} \left( \intend{k}{t} \dt \nabla_{\param}^2 \log \frac{ \intend{k}{t} }{ \underline{\lambda}^*_{k}(t) } + M \intenq{k}{t} \dt \nabla_{\param}^2 \log \frac{ \intenq{k}{t} }{ \underline{\lambda}^*_{k}(t) } \right) } \\
&= \int_{t=0}^{T} \E[\esm{x}{[0,t)}{0} \sim \data]{ \sum_{k=1}^{K} (\frac{ 1 }{ \underline{\lambda}^*_{k}(t) } - \frac{ 1 }{ \intend{k}{t} }  ) \nabla_{\param} \intend{k}{t} \nabla_{\param} \intend{k}{t}^\top \dt  } \\
&= \int_{t=0}^{T} \data(\esmh{x}{t}{0}) \sum_{k=1}^{K} (\frac{ 1 }{ \underline{\lambda}^*_{k}(t) } - \frac{ 1 }{ \intend{k}{t} }  ) \nabla_{\param} \intend{k}{t} \nabla_{\param} \intend{k}{t}^\top \dt 
\end{align}
\end{subequations}
where we omit the condition $\esmh{x}{t}{0}$ in the probabilities and intensities for presentation simplicity.
We also omit the tedious arithmetic manipulation that spells $\nabla_{\param}^2 \log (\lambda/\underline{\lambda})$ out.

Following similar steps, we then rearrange $\var[\nabla_{\param} \objind(\param^*)]$ to be 
\begin{subequations}
\begin{align}
&\int_{t=0}^{T} \E[\esm{x}{[0,t)}{0} \sim \data]{\sum_{k=1}^{K} \left( \intend{k}{t} \dt \nabla_{\param} \nabla_{\param}^\top \log \frac{{\lambda}^*_{k}(t)}{\underline{\lambda}^*_{k}(t) } + M \intenq{k}{t } \dt \nabla_{\param} \nabla_{\param}^\top \log \frac{{\lambda}^{q}_{k}(t)}{\underline{\lambda}^*_{k}(t) })  \right)}  \\
=& 
\int_{t=0}^{T}  \E[\esm{x}{[0,t)}{0} \sim \data]{\sum_{k=1}^{K} ( \frac{ 1 }{ \intend{k}{t} } - \frac{ 1 }{ \underline{\lambda}^*_{k}(t) }  ) \nabla_{\param} \intend{k}{t} \nabla_{\param} \intend{k}{t}^\top \dt} \\
=& \E{ - \nabla_{\param}^2 \objind(\param^*)} 
\end{align}
\end{subequations}
where we use $\nabla_{\param}\nabla_{\param}^\top f(\param)$ to denote $(\nabla_{\param} f(\param)) (\nabla_{\param} f(\param)) ^\top$. 
For presentation simplicity, we omit the arithmetic manipulation that spells $\nabla_{\param}\nabla_{\param}^\top \log (\lambda/\underline{\lambda})$ out.

Then we can simplify the asymptotic normality to be
\begin{align}
\sqrt{N} (\hat{\param} - \param^* ) \rightarrow \Normal(0, \inv{\E{ - \nabla_{\param}^2 \objind(\param^*)}} )
\end{align}

We can think of $\vecb{I}_{M} \defeq \E{ - \nabla_{\param}^2 \objind(\param^*)}$ as the ``information matrix'' of our objective $\nce(\param)$. 
And its relation with the Fisher information matrix $\vecb{I}_{*}$ is: 
\begin{align}
\vecb{I}_{M} = \vecb{I}_{*} - \xunderbrace{\int_{t=0}^{T} \sum_{\esm{x}{[0,t)}{0}} \data(\esmh{x}{t}{0}) \sum_{k=1}^{K} \frac{ 1 }{ {\lambda}^*_{k}(t) + M \intenq{k}{t} } \nabla_{\param} \intend{k}{t} \nabla_{\param} \intend{k}{t}^\top \dt}{\text{call it } \Delta \vecb{I} }
\end{align}
Apparently, when $M$ is large enough, $\vecb{I}_{M}$ will be non-singular. 
Precisely, since $\vecb{I}_{*}$ is non-singular, there must exist $\bar{M} > 0$ such that, for any $M > \bar{M}$, $0 < \| \Delta \vecb{I} \| \leq \sigma(\vecb{I}_{*})/2$ where $\sigma(\vecb{I})$ is the \emph{smallest} singular value of matrix $\vecb{I}$ and $\| \vecb{I} \|$ is the \emph{spectral norm}, i.e., the \emph{largest} singular value, of matrix $\vecb{I}$. 
By Weyl's inequality, we have $\sigma(\vecb{I}_{M}) \geq \sigma(\vecb{I}_{*})- \| \Delta \vecb{I} \| \geq \sigma(\vecb{I}_{*}) / 2$, meaning that $\vecb{I}_{M}$ is non-singular. 

Now we can start analyzing $\| \inv{\vecb{I}_{M}} - \inv{\vecb{I}_{*}} \|$. 
By the definition of the spectral norm, we have: 
\begin{align}\label{eqn:bound1}
\| \inv{\vecb{I}_{M} } - \inv{\vecb{I}_{*}} \| 
= \| \inv{\vecb{I}_{*}} ( \vecb{I}_{*} - \vecb{I}_{M} ) \inv{\vecb{I}_{M}}  \| 
\leq \| \inv{\vecb{I}_{*}} \| \| \Delta \vecb{I} \| \| \inv{\vecb{I}_{M}}  \|
\leq \frac{1}{\sigma(\vecb{I}_{*})} \| \Delta \vecb{I} \| \frac{2}{\sigma(\vecb{I}_{*})}
\end{align}
Since the intensity functions are all bounded, continuous and twice continuously differentiable, $\| \nabla_{\param} \intend{k}{t} \nabla_{\param} \intend{k}{t}^\top \|$ will be bounded, meaning that $\| \Delta \vecb{I} \|$ will be bounded as well.
Moreover, the ratio $\intend{k}{t}/ \intenq{k}{t}$ is also bounded. 
We define $r = \sup_{\esmh{x}{t}{0}, k} \frac{\intend{k}{t \mid \esmh{x}{t}{0}}}{\intenq{k}{t \mid \esmh{x}{t}{0}}}$ and have $M \intenq{k}{t} \geq M \intend{k}{t} / r$. 
Then there must exist $B > 0$ such that we have: 
\begin{align}\label{eqn:bound2}
	\| (1 + \frac{M}{r})  \Delta \vecb{I} \| 
	\leq B \| \vecb{I}_{*}  \|
	\Rightarrow 
	\|  \Delta \vecb{I} \| \leq \frac{r B}{ r + M } \| \vecb{I}_{*}  \| < \frac{1}{ M} r B \| \vecb{I}_{*} \|
\end{align}

Combining \cref{eqn:bound1,eqn:bound2}, we have 
\begin{align}
	\| \inv{\vecb{I}_{M} } - \inv{\vecb{I}_{*}} \| \leq \frac{1}{ M} \xunderbrace{ \frac{2}{\sigma(\vecb{I}_{*})^2} r B \| \vecb{I}_{*} \| }{\text{ call it } C}
\end{align}
meaning that there exists $C > 0$ such that, for any $M > \bar{M}$, $\| \inv{\vecb{I}_{M}} - \inv{\vecb{I}_{*}} \| \leq C / {M}$. 

Note that the ratio $r$ reflects the effect of $\intenq{k}{t}$ on the efficiency. 
In the special case of $\noise = \data$, we have $r=1$ and $\Delta \vecb{I} = \frac{1}{M + 1} \vecb{I}_{*}$ and the asymptotic covariance matrix becomes $(1 + \frac{1}{M}) \inv{\vecb{I}_{*}}$. 

This completes our proof. 

\end{proof}

\section{Algorithm Details}\label{app:algo_details}

\subsection{NCE Objective Computation Details}\label{app:nce_obj}

Our main algorithm is presented as \cref{alg:nce}.  It covers the recipe for computing our NCE objective, as well as the algorithm to sample from $\noise$. 

\begin{algorithm}
	\caption{Training Objective Computation for Noise-Contrastive Estimation.}\label{alg:nce}
	\begin{algorithmic}[1]
		\INPUT observed event stream $\es{x}{[0,T)}$ with $I$ events at times $0 = t_0 < t_1 < \ldots t_I < t_{I+1} = T$;\newline
		model $\model_{\param}$; noise distribution $\noise$; number of noise samples $M$
		\OUTPUT training objective $\nce$ evaluated on $\es{x}{[0,T)}$ and the corresponding noise samples
		\Procedure{computeObjective}{$\es{x}{[0,T)}, \model_{\param}, \noise, M$}
		\LineComment{algorithm input $\model_{\param}$ gives info to define intensity function $\inten{k}{t}$}
		\State $\nce \gets 0$ \Comment{initialize the objective}
		\State initialize the neural states $s$ and $s^{\mathrm{q}}$ of $\model_{\param}$ and $\noise$ respectively \Comment{i.e., their LSTM states} 
		\State $i \gets 0$
		\While{$i \leq I$} %
		\State $i \pluseq 1$
		\LineComment{use noise samples in the current interval}
		\For{$(t, k, \lambda^{\mathrm{q}}, \mu )$ {\bfseries in} \textsc{drawNoiseSamples}($t_{i-1}, t_i$)}
			\State compute the model intensity $\inten{k}{t \mid s}$ under $\model_{\param}$
			\State $\nce \pluseq \mu \log \frac{\lambda^{\mathrm{q}}}{\inten{k}{t \mid s} + M \lambda^{\mathrm{q}} } $
			\EndFor
		\IfThen{$i > I$}{{\bfseries break}}
		\LineComment{use the real event at time $t_i$}
		\State $t \gets t_i$, $k \gets \es{x}{t_i}$
		\State compute the model intensity $\inten{k}{t \mid s}$ under $\model_{\param}$
		\State compute the noise intensity $\intenq{k}{t \mid s^{\mathrm{q}}}$ under $\noise$
		\State $\nce \pluseq \log \frac{\inten{k}{t \mid s}}{\inten{k}{t \mid s} + M \intenq{k}{t \mid s^{\mathrm{q}}} }$
		\State update the neural states $s$ and $s^{\mathrm{q}}$ of $\model_{\param}$ and $\noise$ respectively with this real event
		\EndWhile
		\State {\bfseries return} $\nce$ 
		\EndProcedure
		\Procedure{drawNoiseSamples}{$\tbeg, \tend$} \Comment{draw noise samples over interval $(\tbeg, \tend)$}
		\LineComment{has access to $\noise, M$}
		\LineComment{define the \defn{total intensity function} $\intenq{}{t \mid s^{\mathrm{q}}} \defeq \sum_{c=1}^C
			\intenq{c}{t \mid s^{\mathrm{q}}}$}
		\State $\set{Q} \gets \text{empty collection}$ \Comment{collection of noise samples}
		\State $t \gets \tbeg$; find any $\intenbound \geq \sup\;\{\intenq{}{t \mid s^{\mathrm{q}}}: t \in (\tbeg, \tend)\}$ 
		\Repeat
		\State\label{line:expdraw} draw $\Delta \sim \Exp(M \intenbound)$; $t \pluseq \Delta$ \Comment{propose a noise time}
		\If{$t < \tend$}
		\State $\mu \gets {\intenq{}{t \mid s^{\mathrm{q}}}}/{\intenbound}$ \Comment{compute probability to accept the proposed time}
		\If{$\mu < 0.05$} \Comment{stochastically accept $t$ with prob $\mu$ if $\mu < 0.05$}
                \State $u \sim \Uniform(0,1)$; {\bfseries if} $u < \mu$ : $\mu \gets 1$ 
                \EndIf
		\If{$\mu \geq 0.05$} \Comment{otherwise fractionally accept $t$ with weight $\mu$}
		\State draw $c \in \{1,\ldots,C\}$ where probability of $c$ is $\propto\intenq{c}{t\mid s^{\mathrm{q}}}$ \Comment{choose coarse type}
		\State draw $k \in \{1,\ldots,K\}$ where probability of $k$ is $\noise(k \mid c)$ \Comment{choose refinement}
		\State compute the noise intensity $\intenq{k}{t \mid s^{\mathrm{q}}}$ under $\noise$
		\State add $(t, k, \intenq{k}{t\mid s^{\mathrm{q}}}, \mu)$ to $\set{Q}$
		\EndIf
		\EndIf
		\Until{$t \geq \tend$}
		\State \textbf{return} $\set{Q}$
		\EndProcedure
	\end{algorithmic}
\end{algorithm}

\subsection{Training the Noise Distribution $\noise$ by NCE}\label{app:train_q}\label{sec:train_p}

Before we optimize our $\nce(\param)$, we first fit the noise distribution $\noise$ to the training data.  As discussed in endnote~\ref{fn:gan},   
we expect that fitting the data well will give a good training signal to learn $\param$.

In the experiments of this paper, we used MLE to estimate the parameters $\paramq$ of $\noise$, which involves taking approximate integrals as in \citet{mei-17-neuralhawkes}. (After all, we did not yet know whether NCE would work well.)  To avoid the approximate integrals, however, one could instead estimate $\paramq$ using NCE. 
When evaluating this NCE objective during training of $\paramq$, one can take the noise distribution to be $\noise_{\paramq_{\mathrm{old}}}$ where $\paramq_{\mathrm{old}}$ is any snapshot of $\paramq$ from a recent iteration of training (even the current iteration).  The same $\paramq_{\mathrm{old}}$ must be used for both drawing noise events via the thinning algorithm, and for scoring these noise events and their contrasting observed events.

Regardless of whether we use MLE or NCE, it is faster to train $\noise$ than to train $\model$ because $\noise$ only has $C$ event types instead of $K$.

The idea of using as the noise distribution a model previously trained with NCE was also considered in the original NCE paper \citep{gutmann-10-nce}.

\section{Experimental Details and Additional Results}\label{app:exp_details}

\subsection{Dataset Details}\label{app:data_details}\label{app:data_stats}\label{app:email_details}\label{app:collegemsg_details}\label{app:bitcoin_details}\label{app:wikitalk_details}

Besides the datasets we have introduced in \cref{sec:exp}, we also run experiments on the following real-world social interaction datasets: 
\paragraph{CollegeMsg \textnormal{\citep{panzarasa-09-patterns}}.}
This dataset contains anonymized private messages sent on an online social network at an university. 
Each record $(u,v,t)$ means that user $u$ sent a private message to user $v$ at time $t$ and each $u,v$ pair is an event type. 
We consider the top 100 users sorted by the number of messages they sent and received: the total number of possible event types is then $K = 9900$ since self-messaging is not allowed.

\paragraph{WikiTalk \textnormal{\citep{leskovec-10-governance}}.} 
This dataset contains the records of anonymized Wikipedia users editing each other's Talk page. 
Each record $(u,v,t)$ means that user $u$ edited user $v$'s talk page at time $t$ and each $u,v$ pair is an event type. 
We consider the top 100 users sorted by the number of edits they made and received and the total number of possible event types is $K=10000$. 

\Cref{tab:stats_dataset} shows statistics about each dataset that we use in this paper.
\begin{table*}[t]
	\begin{center}
		\begin{small}
			\begin{sc}
				\begin{tabularx}{1.00\textwidth}{l *{1}{S}*{3}{R}*{3}{S}}
					\toprule
					Dataset & \multicolumn{1}{r}{$K$} & \multicolumn{3}{c}{\# of Event Tokens} & \multicolumn{3}{c}{Sequence Length} \\
					\cmidrule(lr){3-8}
					&  & Train & Dev & Test & Min & Mean & Max \\
					\midrule
					Synthetic-1 & $10000$ & $100000$ & $10000$ & $10000$ & $100$ & $100$ & $100$ \\
					Synthetic-2 & $10000$ & $100000$ & $10000$ & $10000$ & $100$ & $100$ & $100$ \\
					EuroEmail & $10000$ & $50000$ & $10000$ & $10000$ & $100$ & $100$ & $100$ \\
					BitcoinOTC & $19800$ & $1000$ & $500$ & $500$ & $100$ & $100$ & $100$ \\
					CollegeMsg & $9900$ & $8000$ & $1000$ & $1000$ & $100$ & $100$ & $100$ \\
					WikiTalk & $10000$ & $100000$ & $20000$ & $20000$ & $100$ & $100$ & $100$ \\
					RoboCup & $528$ & $2195$ & $817$ & $780$ & $780$ & $948$ & $1336$ \\ 
					IPTV & $49000$ & $27355$ & $4409$ & $4838$ & $36602$ & $36602$ & $36602$ \\
					\bottomrule
				\end{tabularx}
			\end{sc}
		\end{small}
	\end{center}
	\caption{Statistics of each dataset. For IPTV, we have a single long sequence of 36602 tokens: we use the first 27355 as training data, the next 4409 as dev data and the remaining 4838 as test data. For other datasets, training, dev and test sequences are separate sequences.}
	\label{tab:stats_dataset}
\end{table*}

\subsection{Training Details}\label{app:training_details}

For each of the chosen models in \cref{sec:exp}, the only hyperparameter to tune is the hidden dimension $D$ of the neural network. 
On each dataset, we searched for $D$ that achieves the best performance on the dev set. 
Our search space is $\{4, 8, 16, 32, 64, 128\}$.

For learning, we used the Adam algorithm \citep{kingma-15} with its default settings. 
For each $\rho$ or $M$, we run training long enough so that the log-likelihood on the held-out data can converge. 

\subsection{More Results on Real-World Social Interaction Datasets}\label{app:real_large}

The learning curves on CollegeMsg and WikiTalk datasets are shown in \cref{fig:social_more}: they look similar to those in \cref{fig:social} and lead to the same conclusions. 

\begin{figure*}[!ht]
	\begin{center}
		
		\begin{subfigure}[t]{0.49\linewidth}
			\renewcommand\thesubfigure{\alph{subfigure}1}
			\begin{center}
				\includegraphics[width=0.49\linewidth]{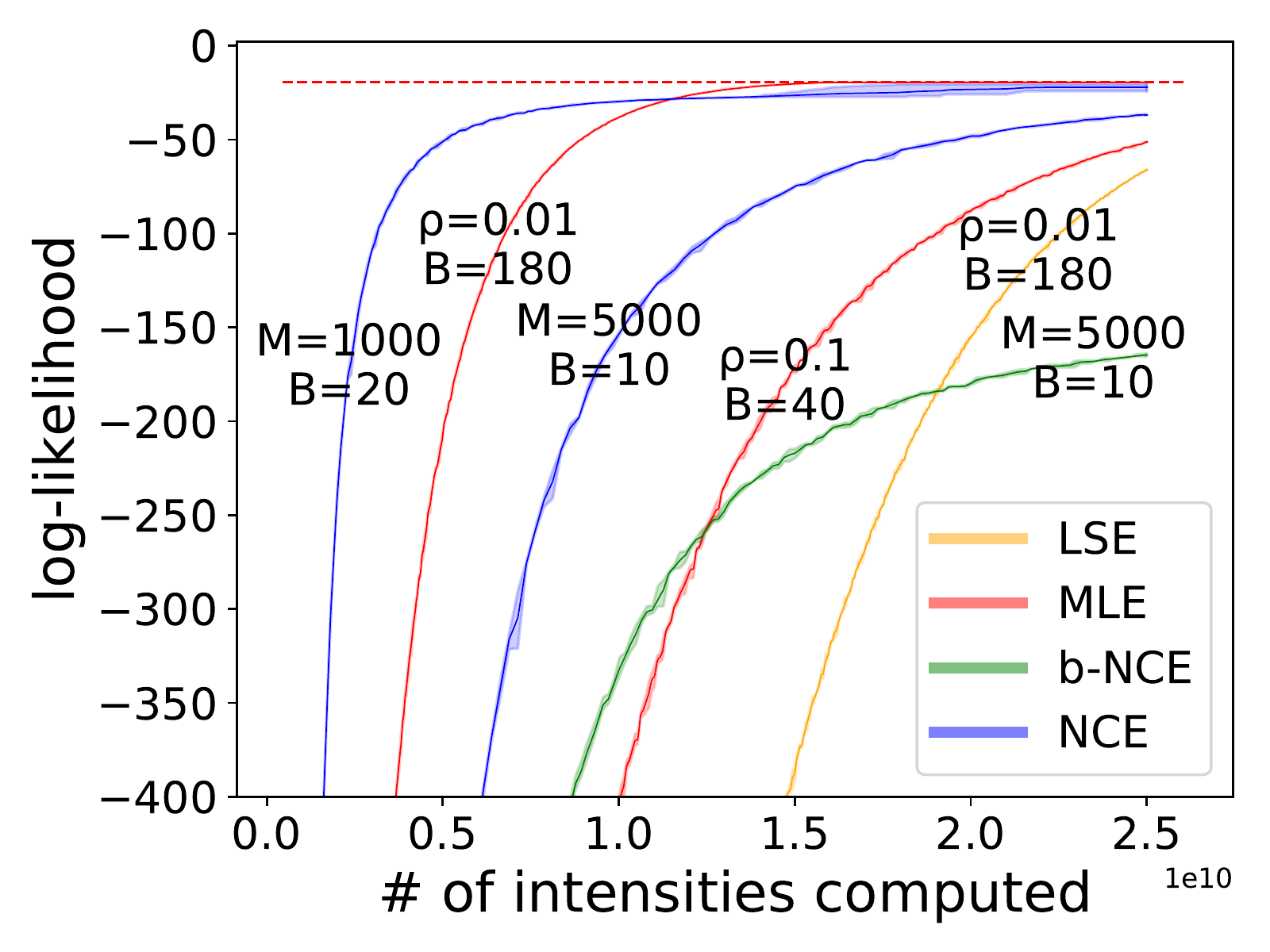}
				\includegraphics[width=0.49\linewidth]{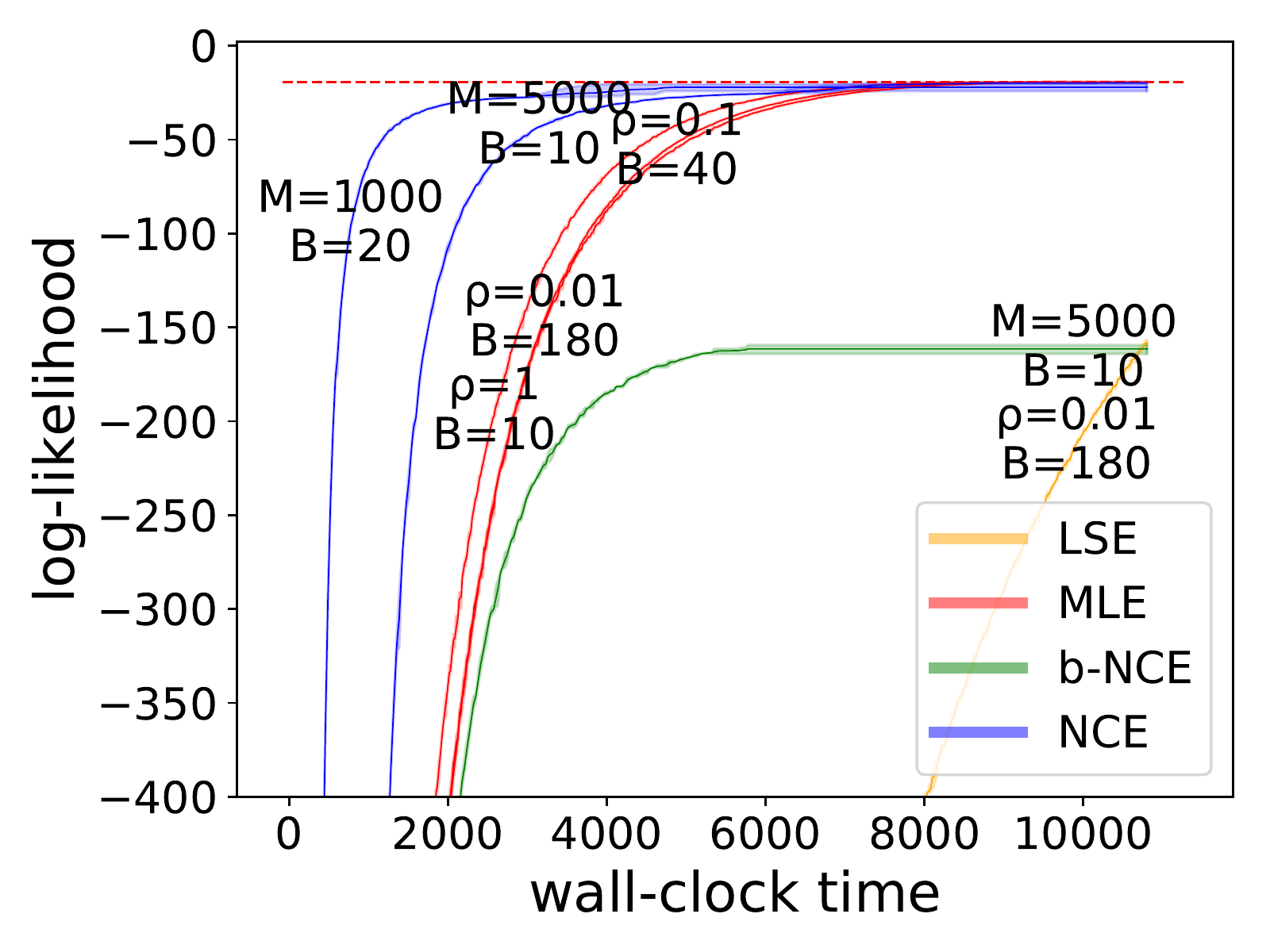}
				\vspace{-16pt}
				\caption{CollegeMsg: neural $\noise$}\label{fig:collegemsg_train_neural}
			\end{center}
		\end{subfigure}
		~
		\begin{subfigure}[t]{0.49\linewidth}
			\addtocounter{subfigure}{-1}
			\renewcommand\thesubfigure{\alph{subfigure}2}
			\begin{center}
				\includegraphics[width=0.49\linewidth]{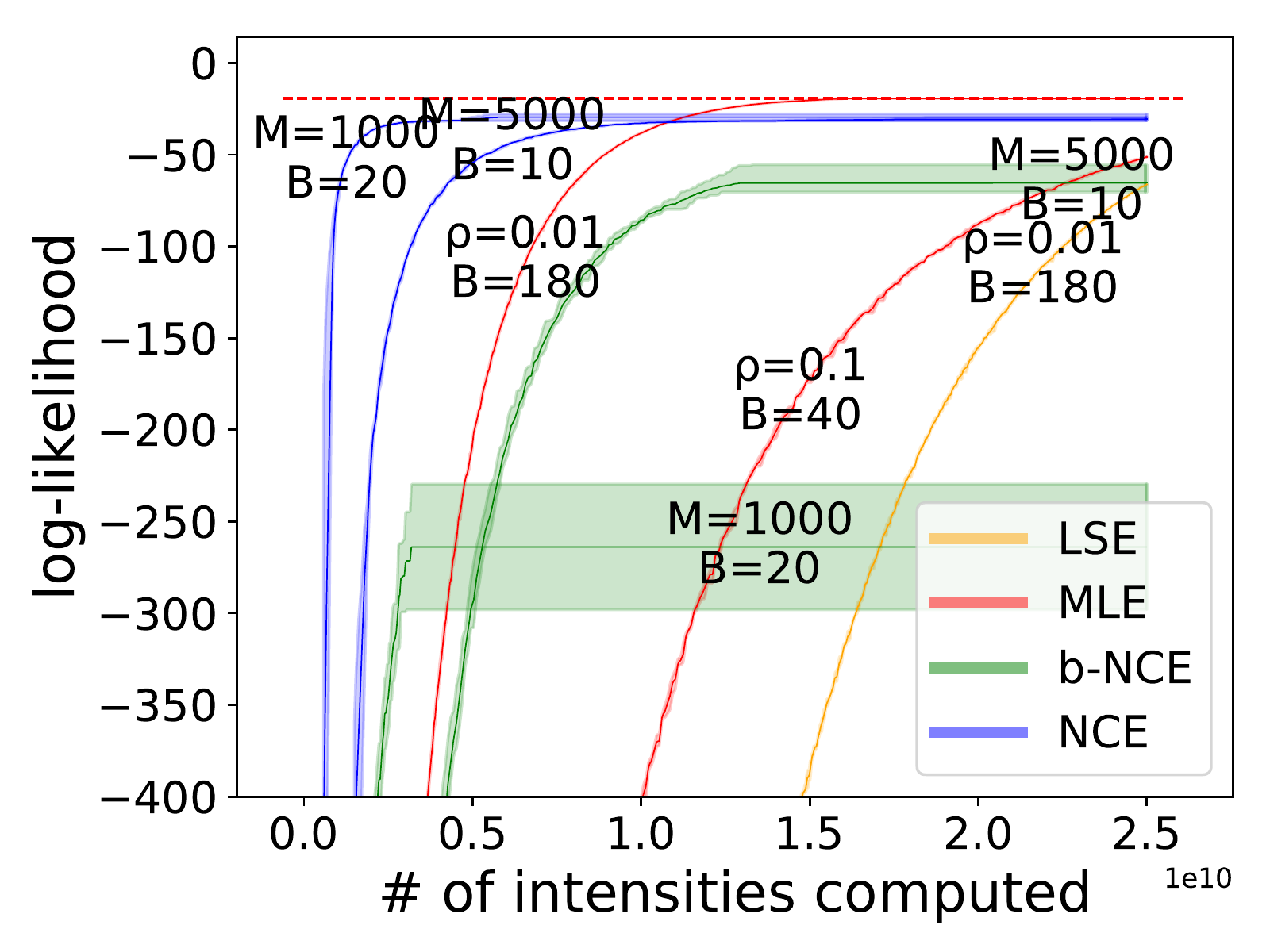}
				\includegraphics[width=0.49\linewidth]{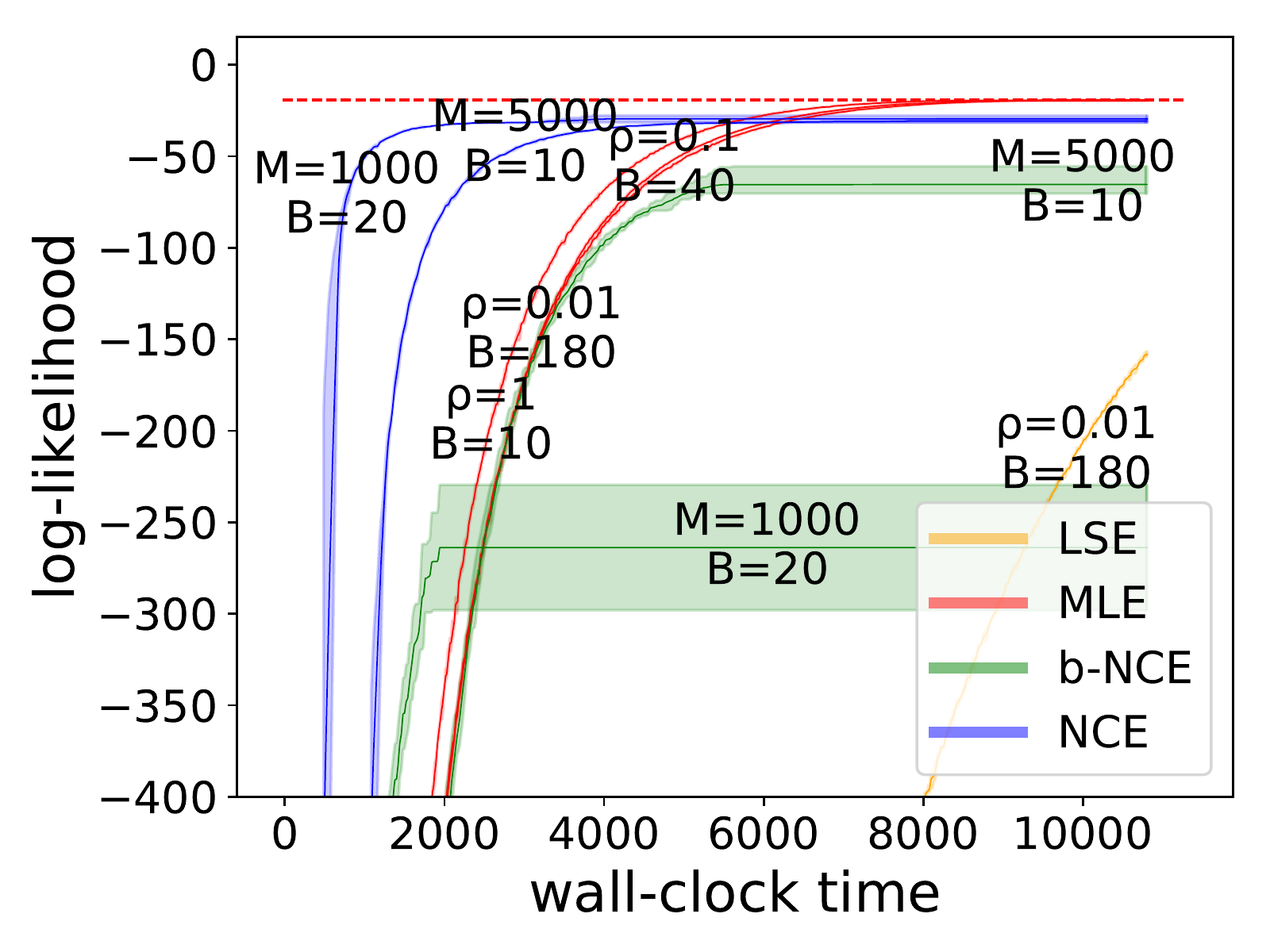}
				\vspace{-16pt}
				\caption{CollegeMsg: Poisson $\noise$}\label{fig:collegemsg_train_poisson}
			\end{center}
		\end{subfigure}
	
		\begin{subfigure}[t]{0.49\linewidth}
			\renewcommand\thesubfigure{\alph{subfigure}1}
			\begin{center}
				\includegraphics[width=0.49\linewidth]{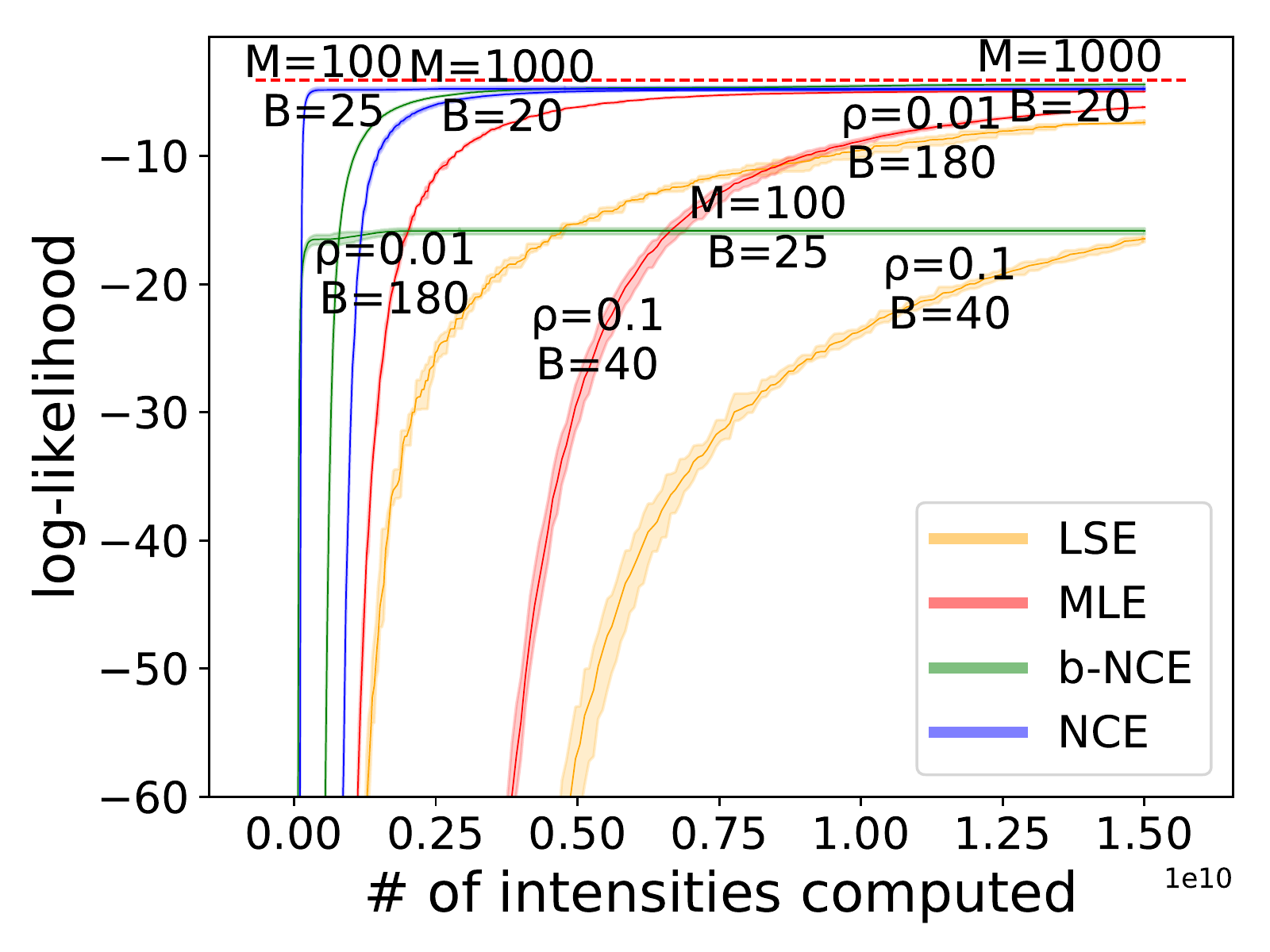}
				\includegraphics[width=0.49\linewidth]{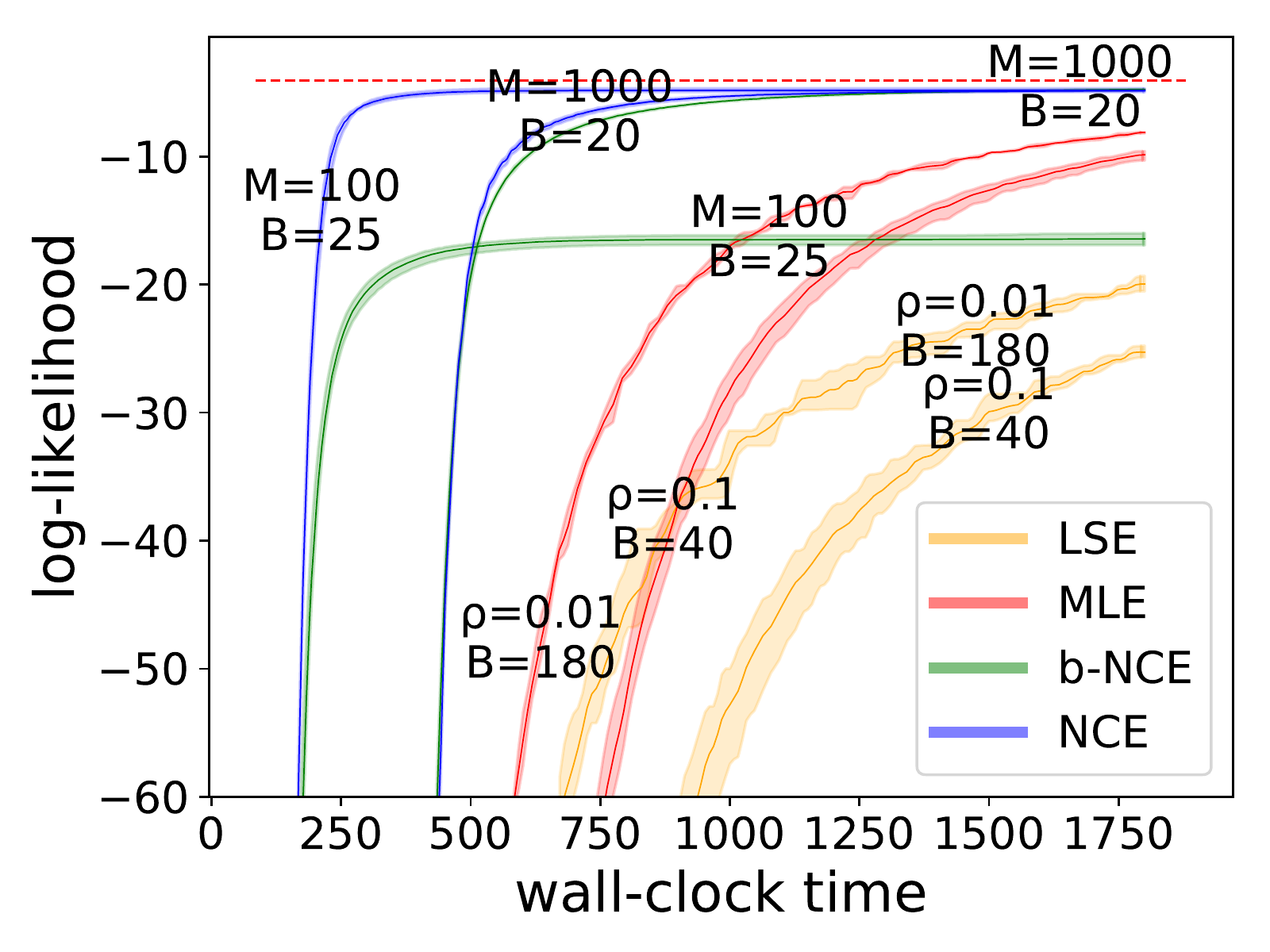}
				\vspace{-16pt}
				\caption{WikiTalk: neural $\noise$}\label{fig:wikitalk_train_neural}
			\end{center}
		\end{subfigure}
		~
		\begin{subfigure}[t]{0.49\linewidth}
			\addtocounter{subfigure}{-1}
			\renewcommand\thesubfigure{\alph{subfigure}2}
			\begin{center}
				\includegraphics[width=0.49\linewidth]{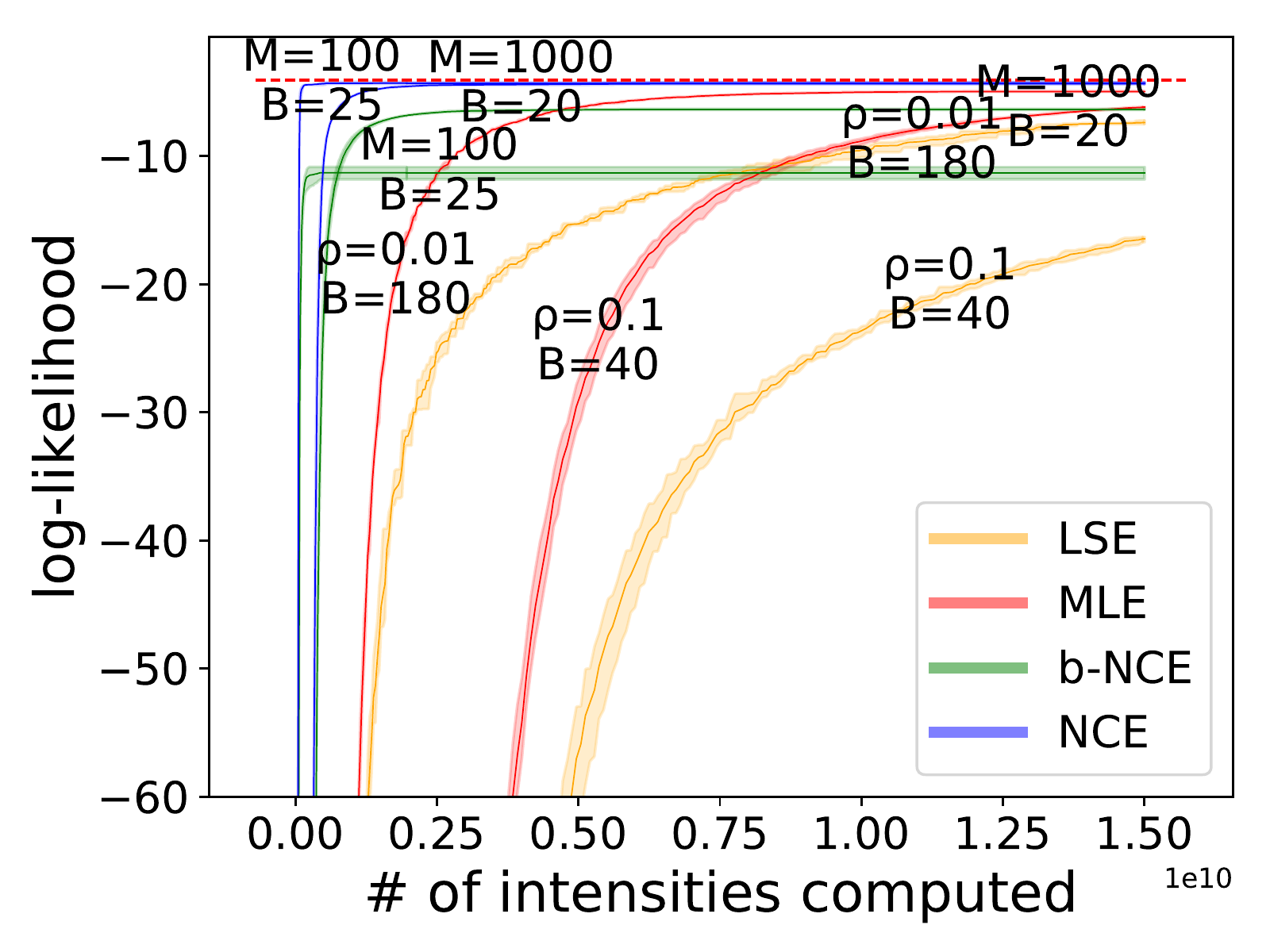}
				\includegraphics[width=0.49\linewidth]{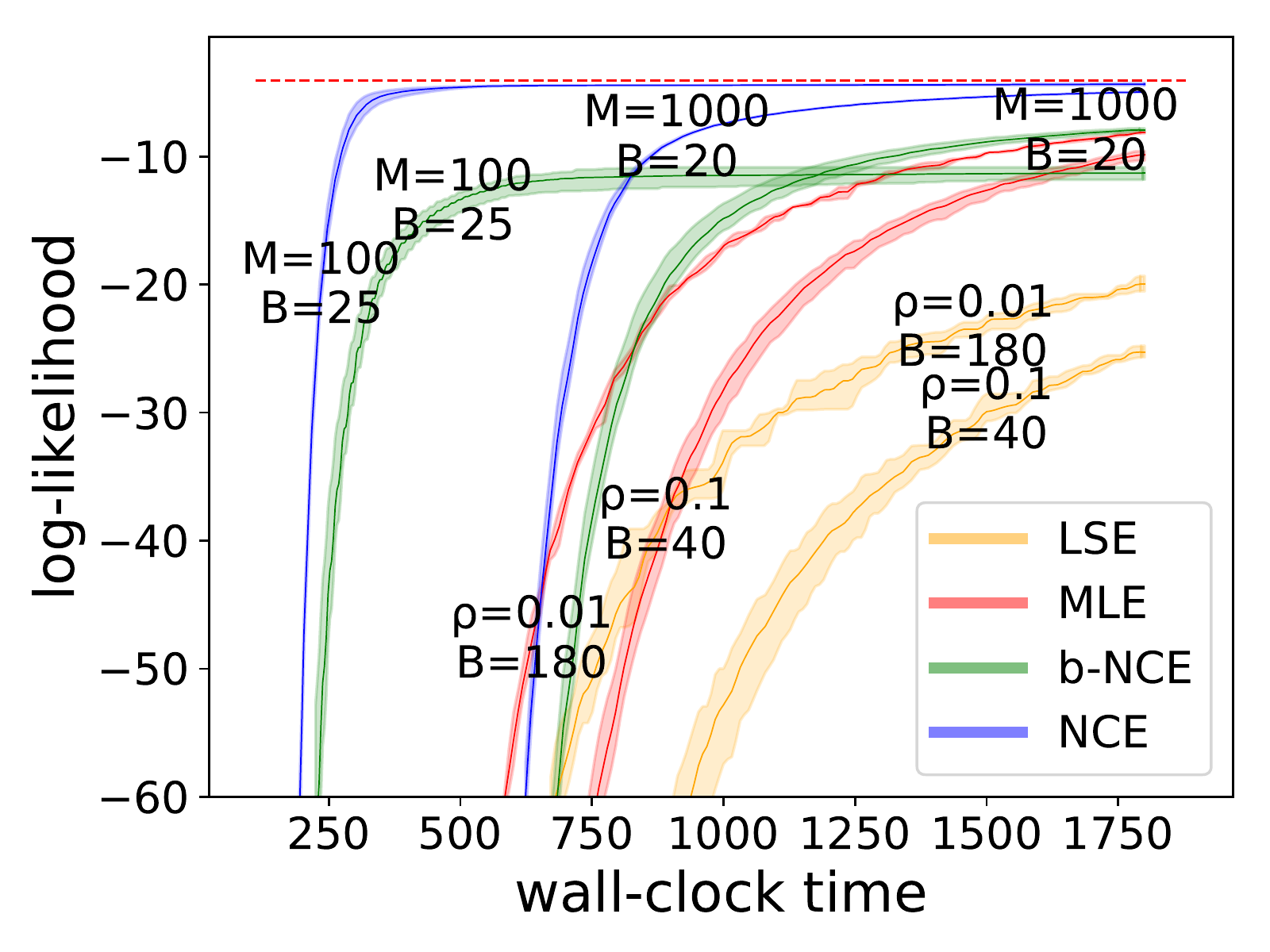}
				\vspace{-16pt}
				\caption{WikiTalk: Poisson $\noise$}\label{fig:wikitalk_train_poisson}
			\end{center}
		\end{subfigure}
		\caption{Learning curves of MLE and NCE on the other real-world social interaction datasets.}
		\label{fig:social_more}
	\end{center}
\end{figure*}

\subsection{Ablation Study I: Always or Never Redraw Noise samples}\label{app:p0p1}

In \cref{fig:nhp_p0p1}, we show the learning curves for the ``always redraw'' and ``never redraw'' strategies on the first synthetic dataset. 
As shown in \cref{fig:nhp_qeqp_p1}, with the ``always redraw'' strategy, NCE (\bluesolid) needs considerably fewer intensity evaluations to reach the highest log-likelihood (\reddash) that MLE (\redsolid) can achieve on the held-out data. 
However, the curve with $M=1000$ increases more slowly than MLE in terms of wall-clock time since it spends too much time on drawing new noise samples. 

As shown in \cref{fig:nhp_qeqp_p0}, with the ``never redraw'' strategy, $M=1000$ overtakes MLE: a single draw of $M=1000$ noise streams is able to give very good training signals and the saved computation can be spent on training $\model_{\param}$ repeatedly on the same samples. 
However, the curve of $M=1$ only achieves $\text{log-likelihood} \approx-200$ and thus falls out of the zoomed-in view. 

\begin{figure*}[!ht]
	\begin{center}
		\begin{subfigure}[b]{0.49\linewidth}
			\begin{center}
				\includegraphics[width=0.48\linewidth]{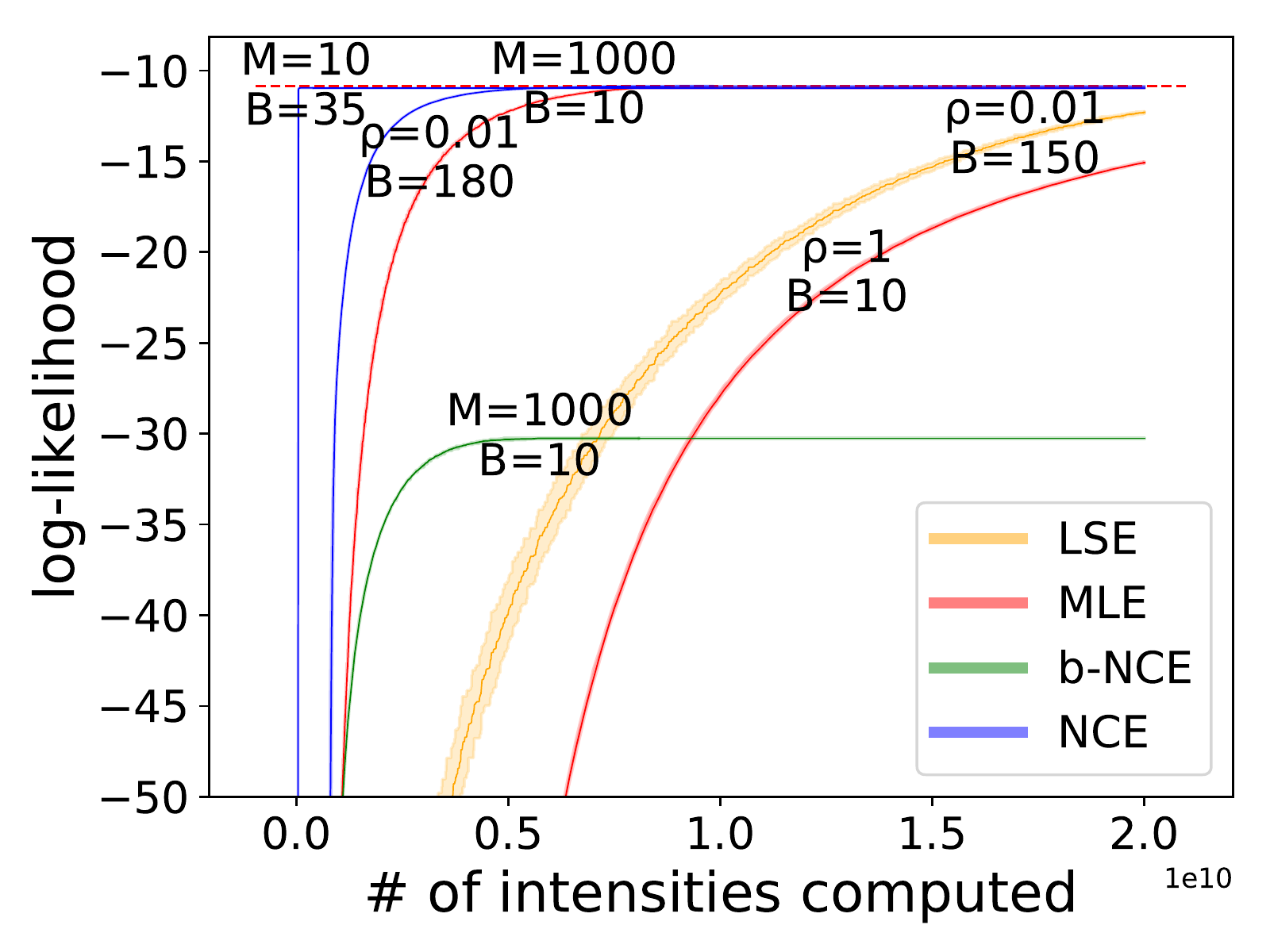}
				\includegraphics[width=0.48\linewidth]{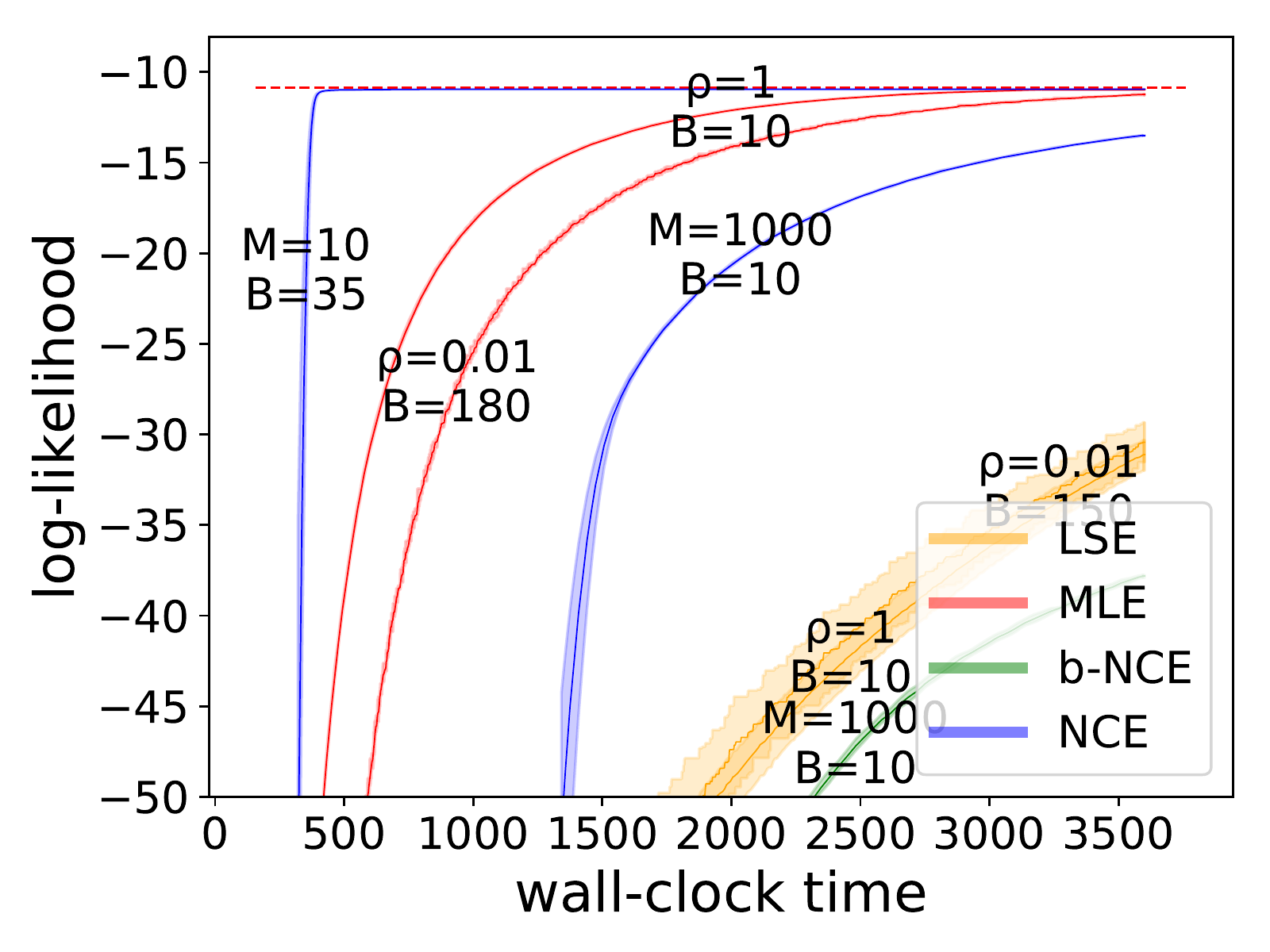}
				\caption{Always redraw new noise samples}\label{fig:nhp_qeqp_p1}
			\end{center}
		\end{subfigure}
		~
		\begin{subfigure}[b]{0.49\linewidth}
			\begin{center}
				\includegraphics[width=0.48\linewidth]{figures/nhp/fullgpu_p0_neural_q_lc_tag=inten_csv=cr_xhigh=20000000000_ylow=-50.pdf}
				\includegraphics[width=0.48\linewidth]{figures/nhp/fullgpu_p0_neural_q_lc_tag=time_csv=cr_xhigh=3600_ylow=-50.pdf}
				\caption{Never redraw new noise samples}\label{fig:nhp_qeqp_p0}
			\end{center}
		\end{subfigure}
		\caption{Ablation Study I. Learning curves of MLE and NCE with $\noise = \data$ and different ``redraw'' strategies. }
		\label{fig:nhp_p0p1}
	\end{center}
\end{figure*}

\begin{figure*}[!ht]
	\begin{center}
		\begin{subfigure}[b]{0.49\linewidth}
			\begin{center}
				\includegraphics[width=0.49\linewidth]{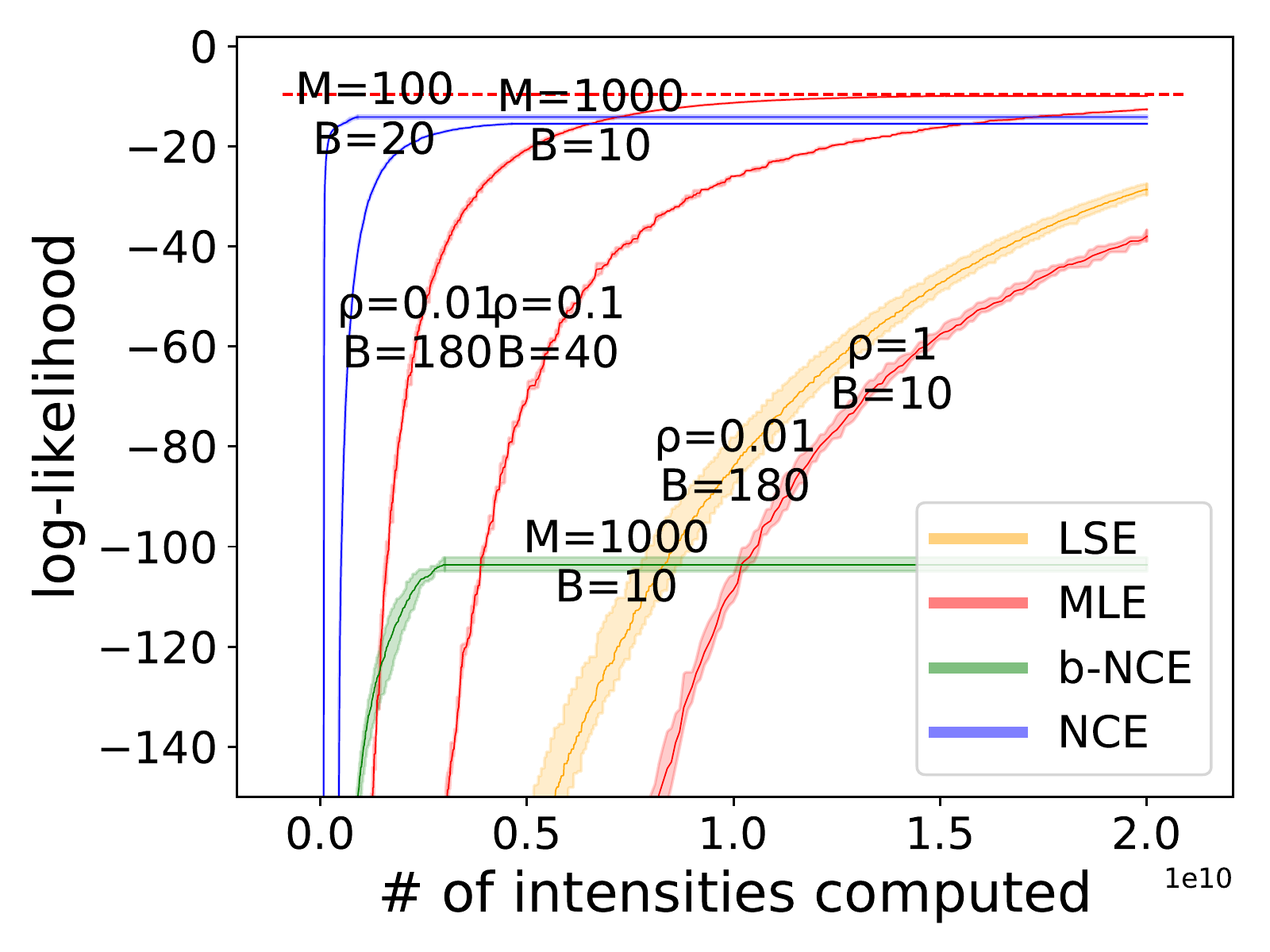}
				\includegraphics[width=0.49\linewidth]{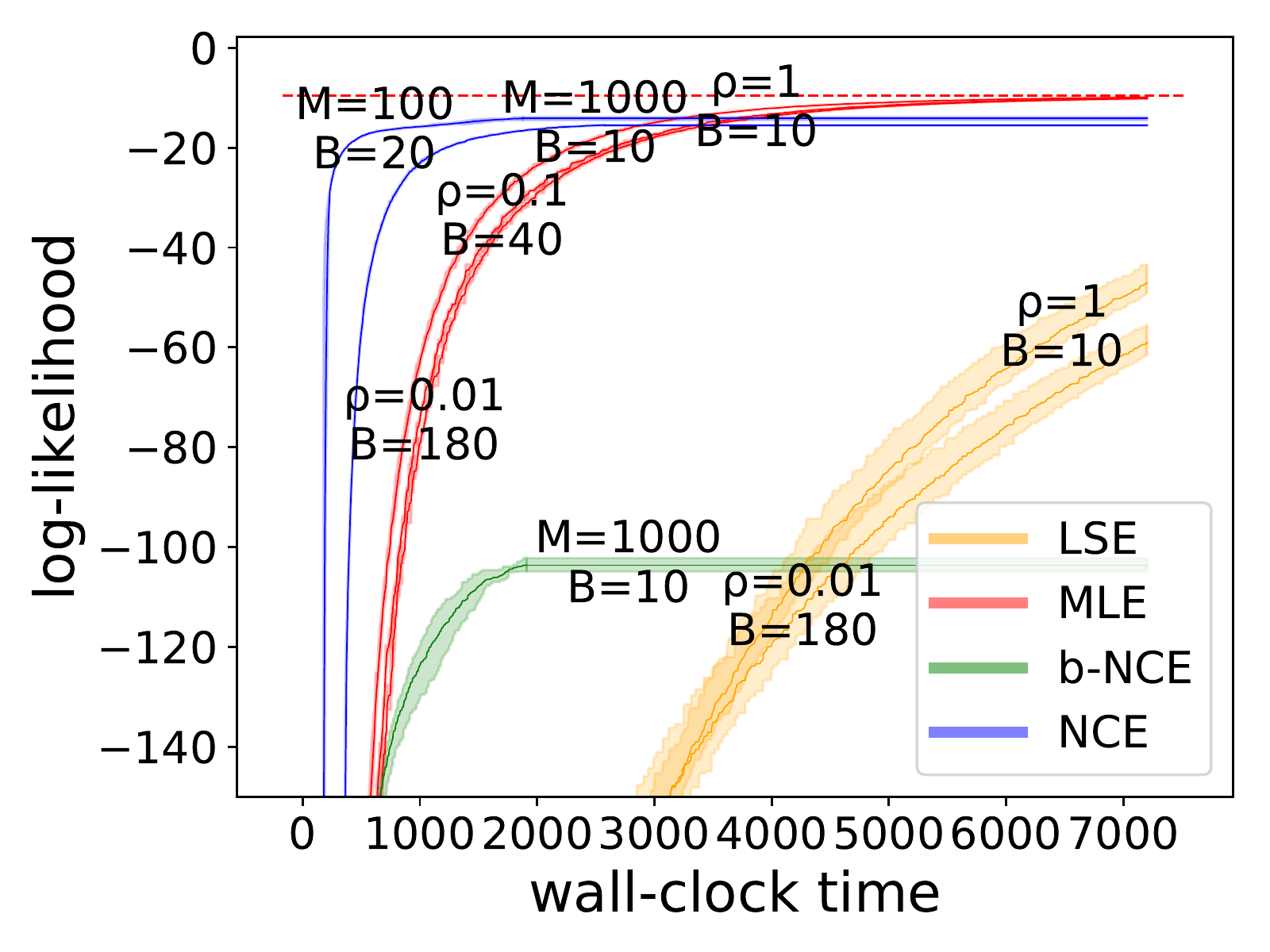}
				\caption{EuroEmail}\label{fig:email_random}
			\end{center}
		\end{subfigure}
		~
		\begin{subfigure}[b]{0.49\linewidth}
			\begin{center}
				\includegraphics[width=0.49\linewidth]{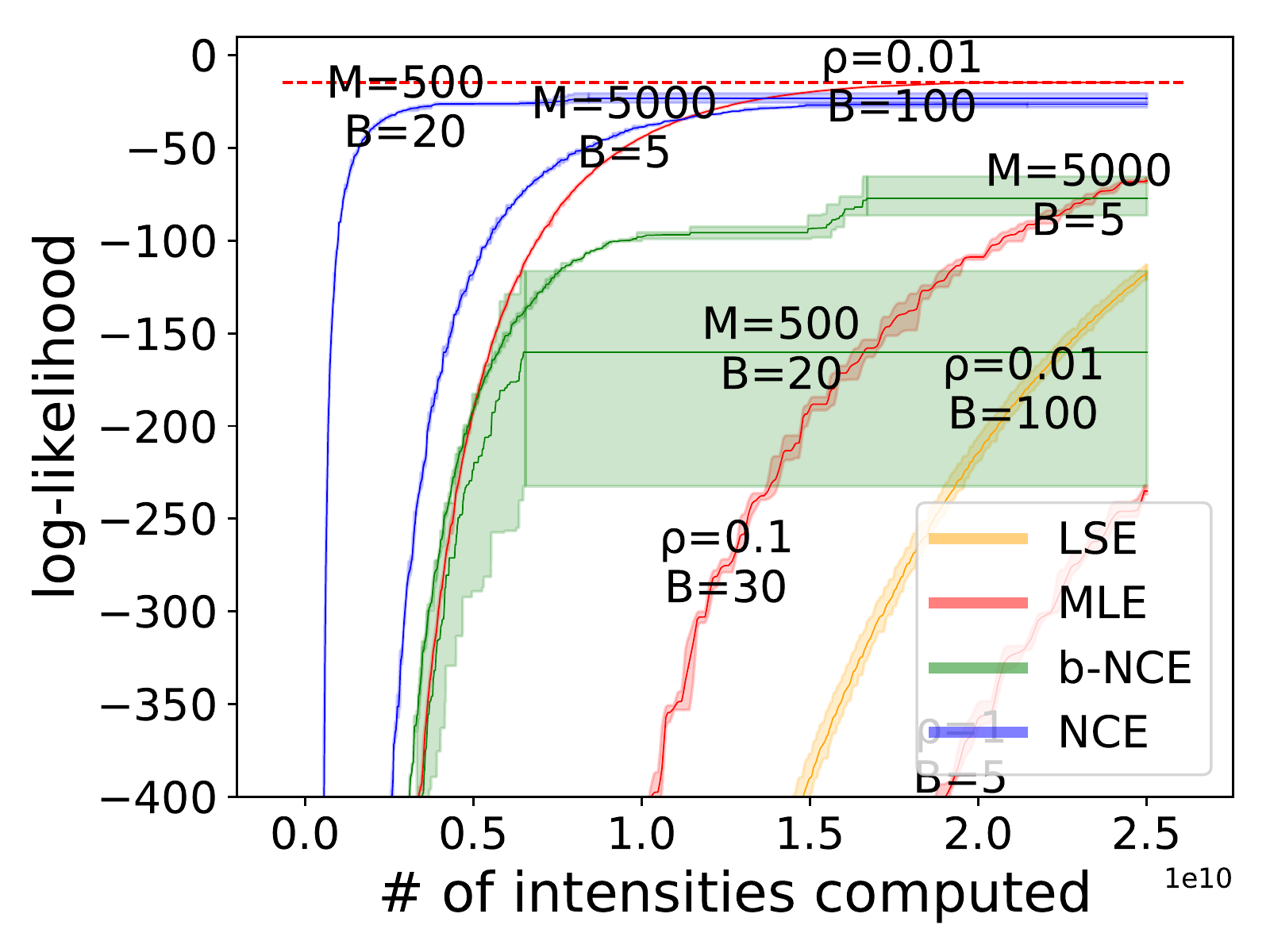}
				\includegraphics[width=0.49\linewidth]{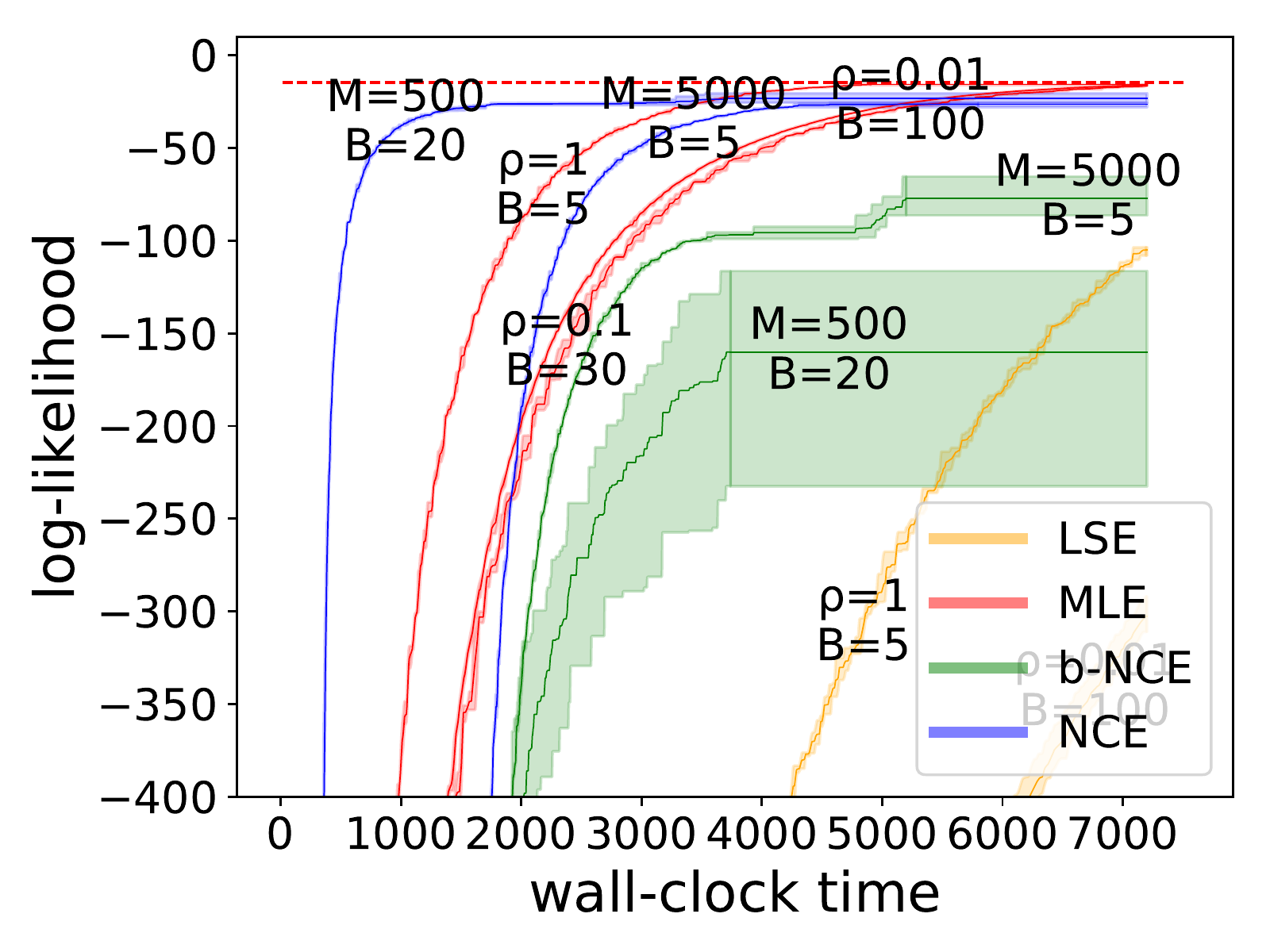}
				\caption{BitcoinOTC}\label{fig:bitcoin_random}
			\end{center}
		\end{subfigure}
		
		\begin{subfigure}[b]{0.49\linewidth}
			\begin{center}
				\includegraphics[width=0.49\linewidth]{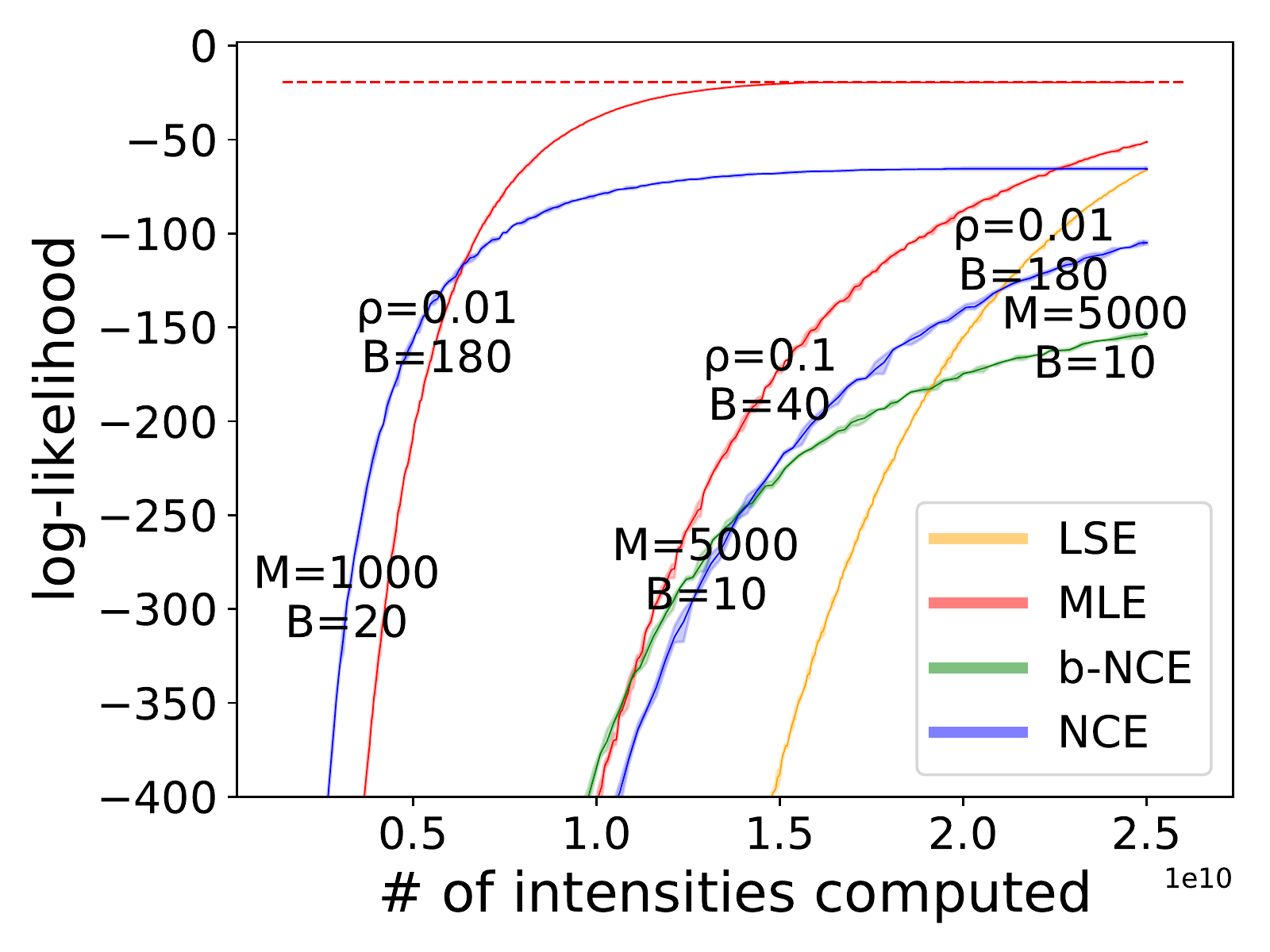}
				\includegraphics[width=0.49\linewidth]{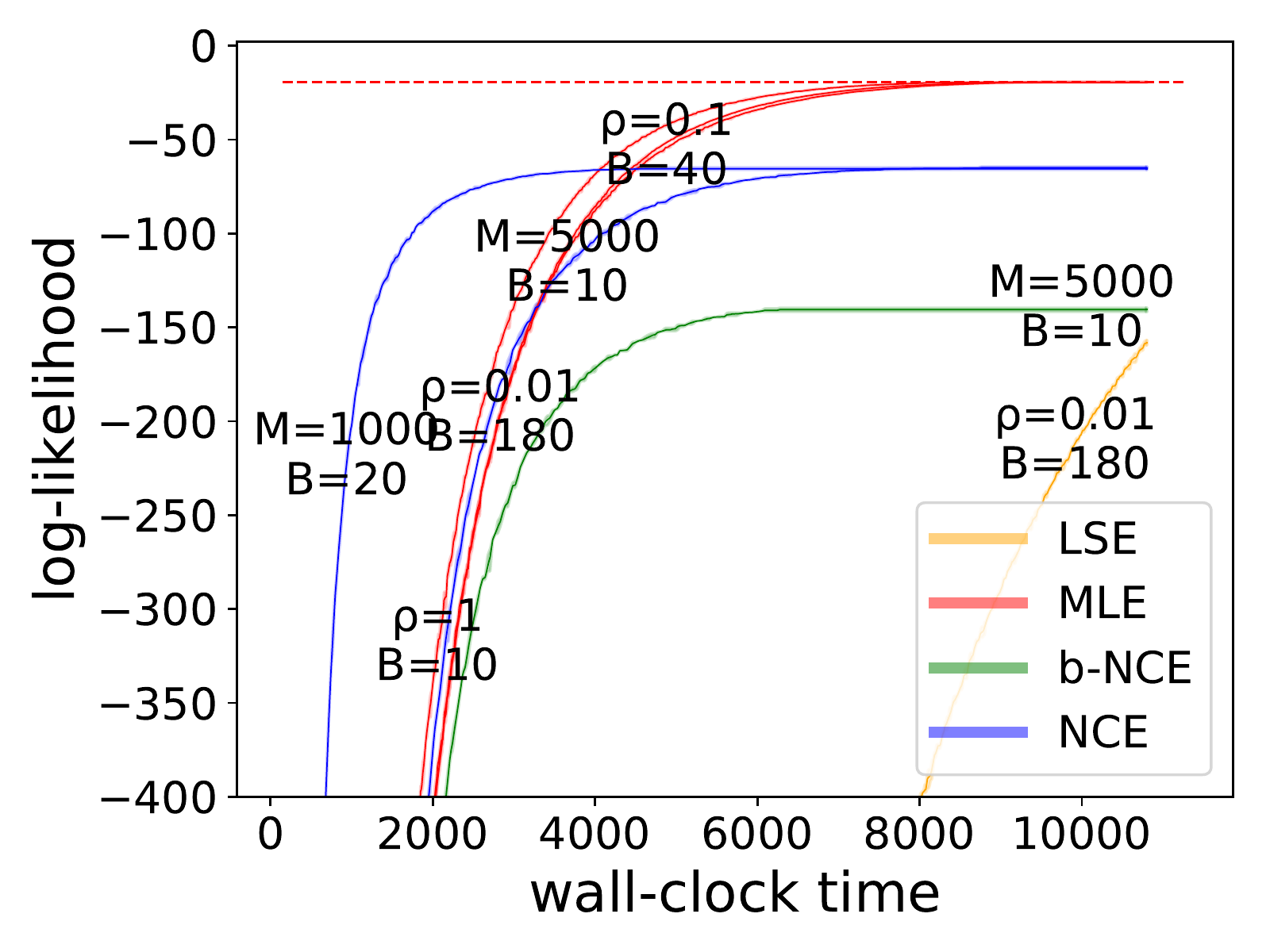}
				\caption{CollegeMsg}\label{fig:collegemsg_random}
			\end{center}
		\end{subfigure}
		~
		\begin{subfigure}[b]{0.49\linewidth}
			\begin{center}
				\includegraphics[width=0.49\linewidth]{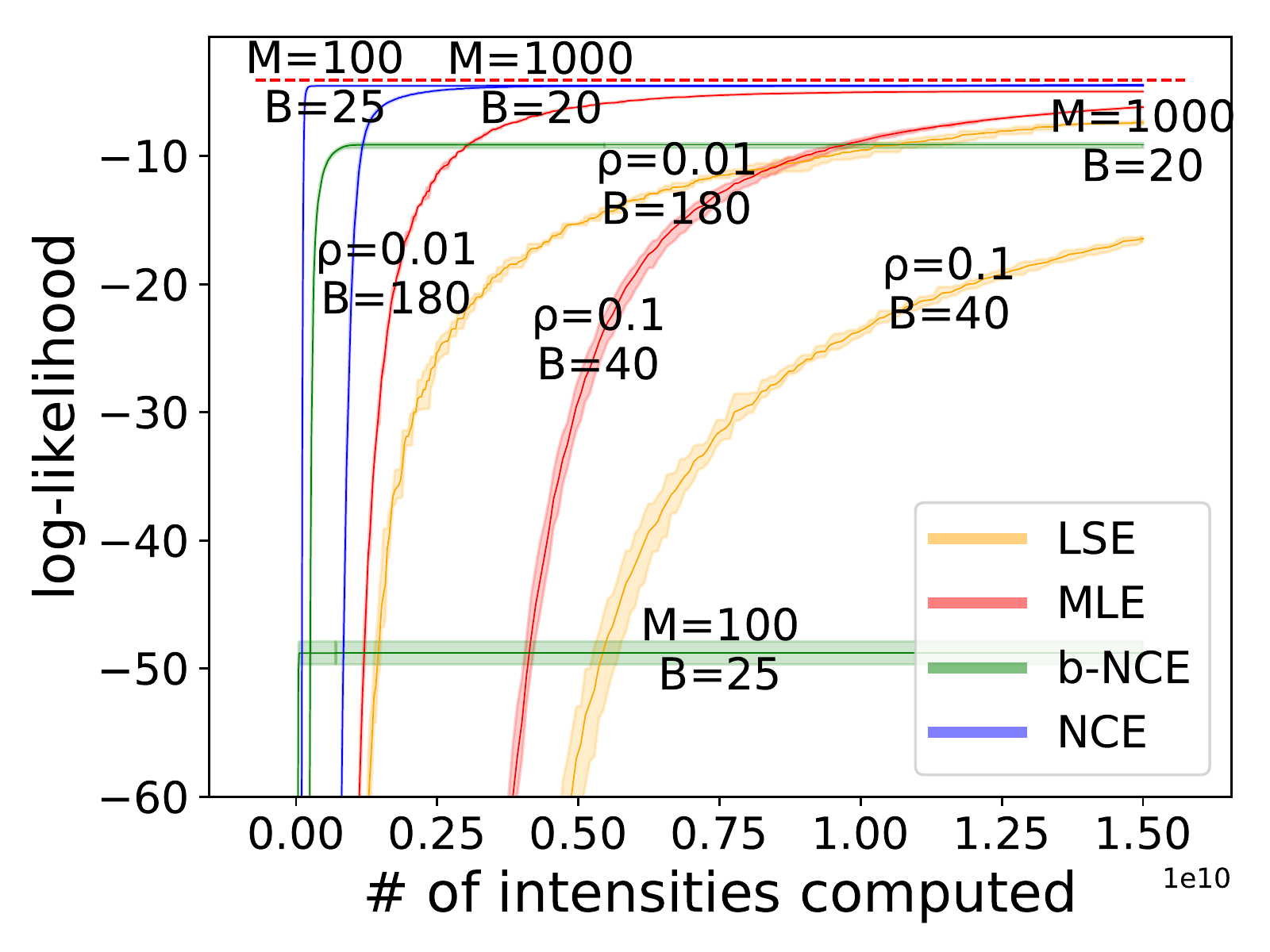}
				\includegraphics[width=0.49\linewidth]{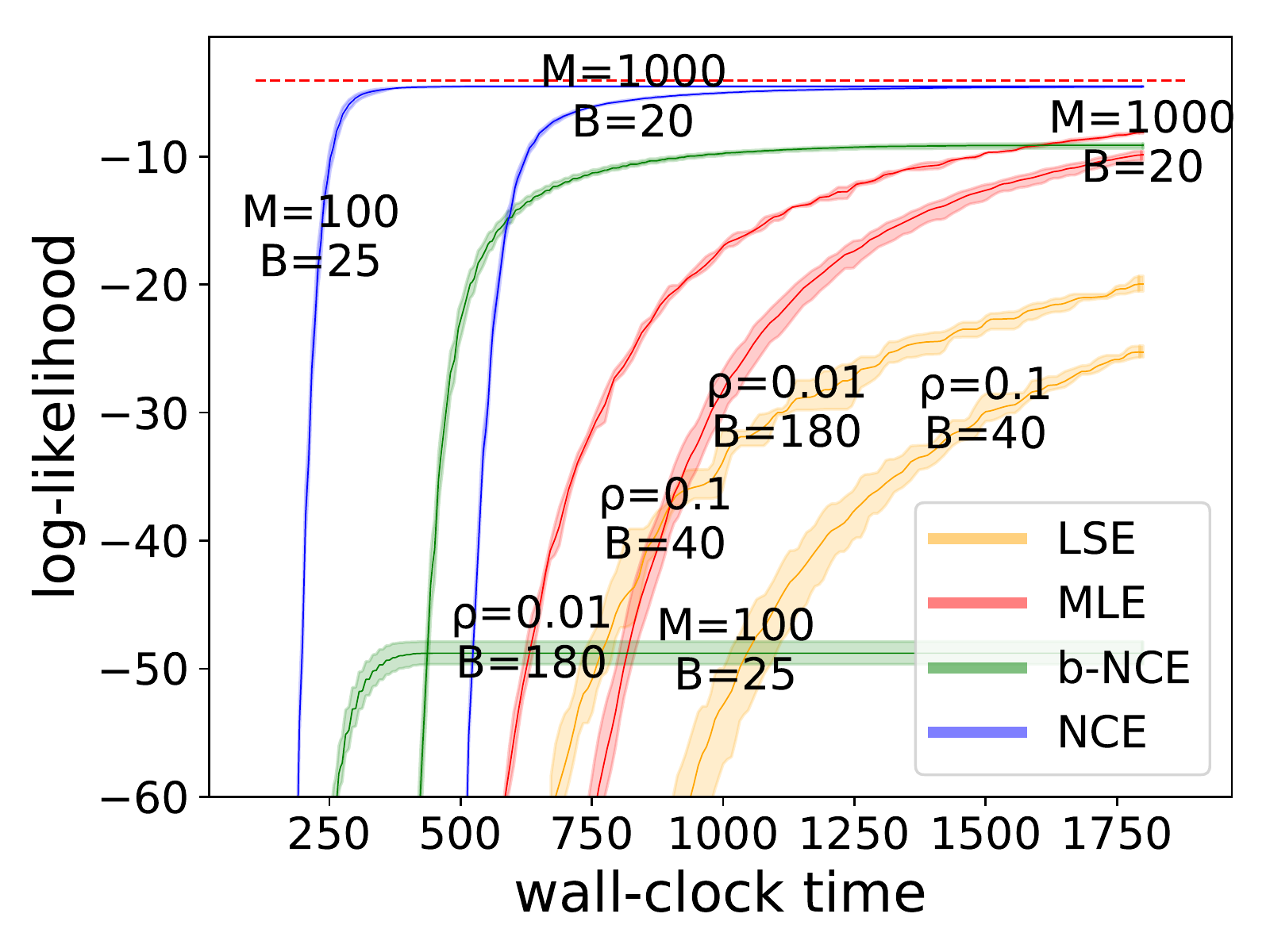}
				\caption{WikiTalk}\label{fig:wikitalk_random}
			\end{center}
		\end{subfigure}
		\caption{Ablation Study II. Learning curves of MLE and NCE with untrained $\noise$ on social interaction datasets.}
		\label{fig:social_random}
	\end{center}
\end{figure*}

\begin{figure*}[!ht]
	\begin{center}
		\begin{subfigure}[b]{0.49\linewidth}
			\begin{center}
				\includegraphics[width=0.49\linewidth]{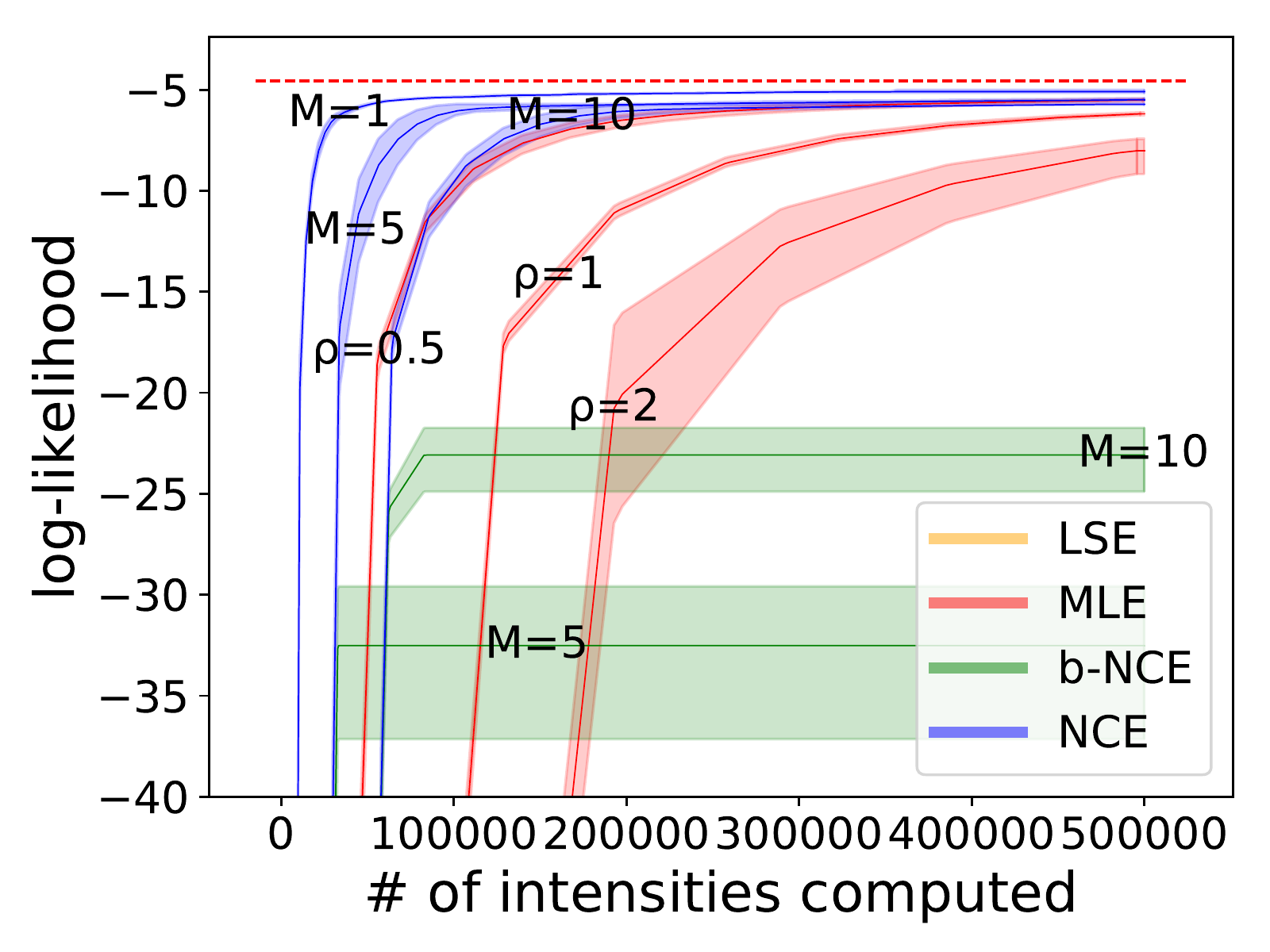}
				\includegraphics[width=0.49\linewidth]{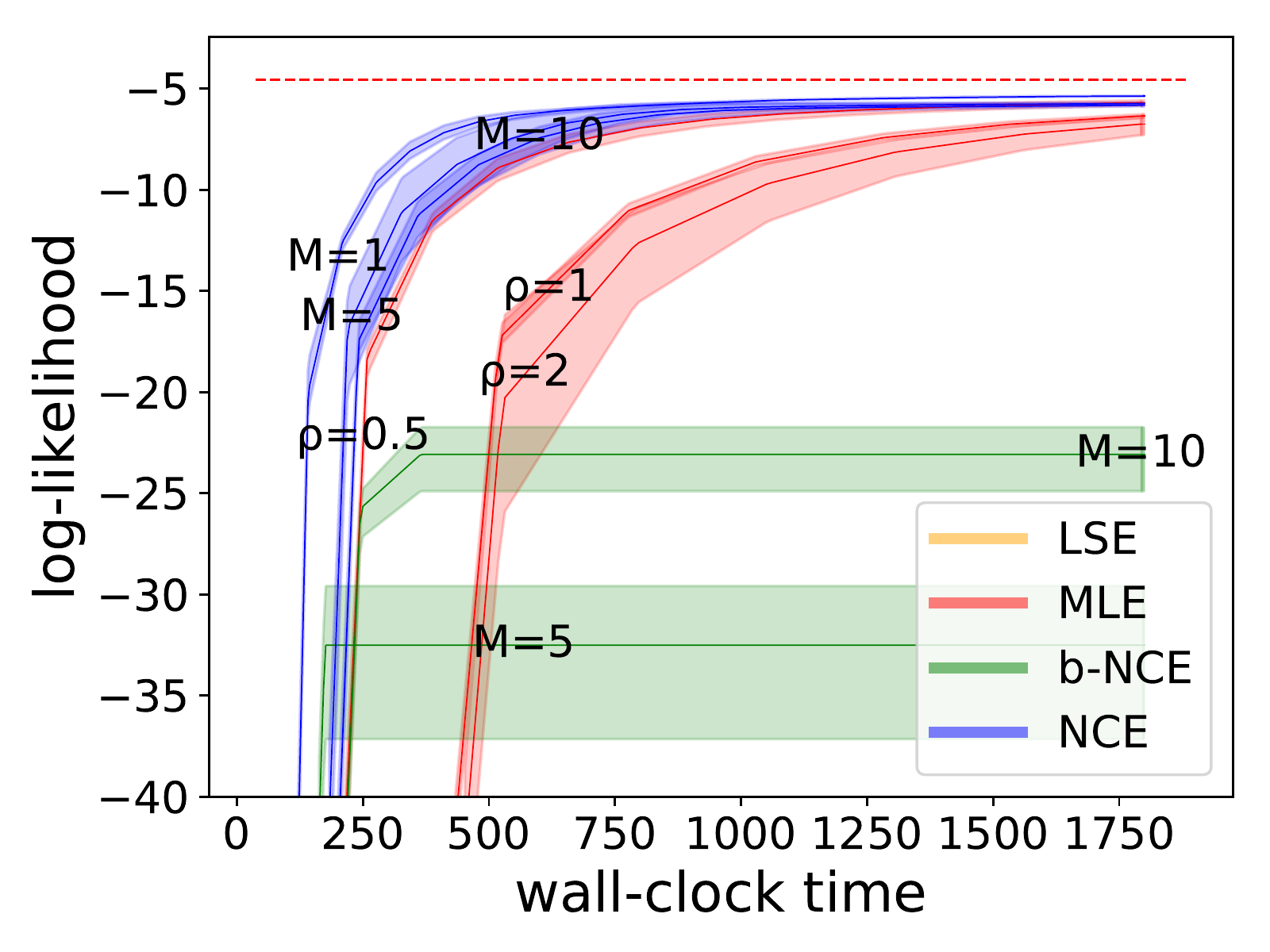}
				\caption{RoboCup}\label{fig:robocup_c1}
			\end{center}
		\end{subfigure}
		~
		\begin{subfigure}[b]{0.49\linewidth}
			\begin{center}
				\includegraphics[width=0.49\linewidth]{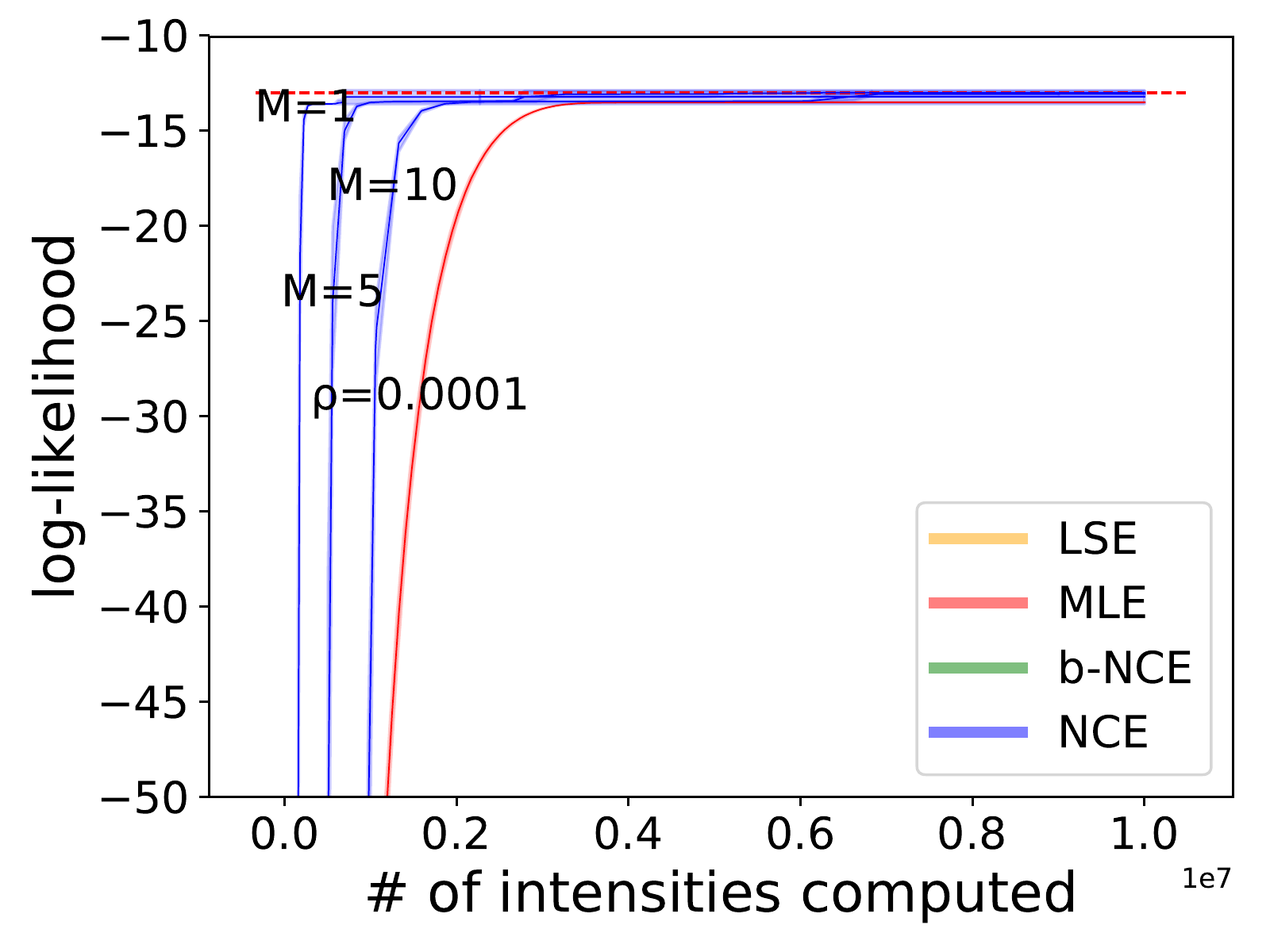}
				\includegraphics[width=0.49\linewidth]{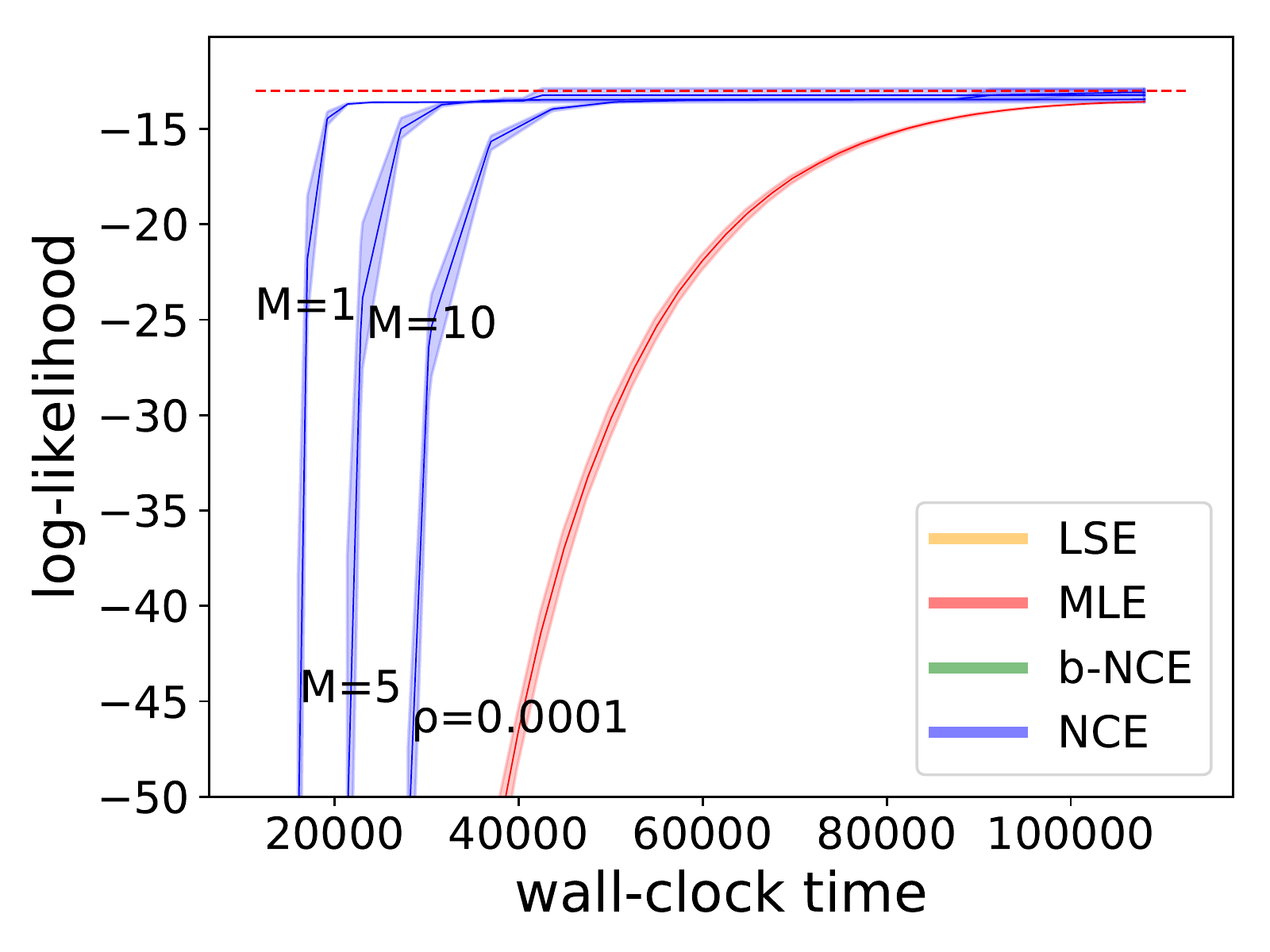}
				\caption{IPTV}\label{fig:iptv_c1}
			\end{center}
		\end{subfigure}
		\caption{Ablation Study III. Learning curves of MLE and NCE using neural $\noise$ with $C=1$.}
		\label{fig:c1}
	\end{center}
\end{figure*}

\newpage
\subsection{Ablation Study II: NCE with Untrained Noise Distribution}\label{app:random_q}

In \cref{fig:social_random}, we show the learning curves of NCE with untrained noise distributions on the real-world social interaction datasets. 
As we can see, NCE in this setting tends to end up with worse generalization (interestingly except on WikiTalk) and suffers slow convergence (on BitcoinOTC and CollegeMsg) and large variance (on BitcoinOTC). 

\subsection{Ablation Study III: Effect of $C$}\label{app:diff_c}

In \cref{fig:c1}, we show learning curves of NCE using the neural $\noise$ with $C=1$.   Taking $C=1$ means that the same number of noise samples can be drawn faster (with fewer intensity evaluations).  However, more training epochs may be needed because the noise looks less like true observations and so NCE's discrimination tasks are less challenging (see endnote~\ref{fn:gan}).

On the RoboCup dataset, $C=1$ exhibits similar learning speed to $C=5$ but has slightly worse generalization. 
On the IPTV dataset, $C=1$ gives a considerable speedup over $C=49$ without harming the final generalization.  The NCE curves for $M = 5$ and $M=10$ shift substantially to the left, since $C=1$ requires \emph{many} fewer intensity evaluations.

\end{document}